\documentclass{article}
\PassOptionsToPackage{numbers, compress}{natbib}
\usepackage[preprint]{neurips_2025}
\usepackage[utf8]{inputenc}
\usepackage[T1]{fontenc}
\usepackage{hyperref}
\usepackage{url}
\usepackage{booktabs}
\usepackage{amsfonts}
\usepackage{amsmath}
\usepackage{amssymb}
\usepackage{amsthm}
\usepackage{nicefrac}
\usepackage{microtype}
\usepackage{xcolor}
\usepackage{graphicx}
\usepackage{array}
\usepackage{enumitem}
\usepackage{algorithm}
\usepackage{algpseudocode}
\usepackage{float}
\usepackage{placeins}
\usepackage{adjustbox}
\usepackage{tikz}
\usetikzlibrary{arrows.meta,positioning,decorations.pathreplacing,calc}
\usepackage{pgfplots}
\pgfplotsset{compat=1.17}

\newtheorem{theorem}{Theorem}
\newtheorem{lemma}{Lemma}
\newtheorem{proposition}{Proposition}
\newtheorem{corollary}{Corollary}
\theoremstyle{definition}
\newtheorem{definition}{Definition}
\newtheorem{assumption}{Assumption}
\theoremstyle{remark}

\definecolor{ours}{HTML}{2a78d6}
\definecolor{base}{HTML}{eb6834}
\definecolor{third}{HTML}{1baf7a}
\definecolor{ink}{HTML}{0b0b0b}
\definecolor{muted}{HTML}{52514e}
\definecolor{gridc}{HTML}{e6e5e0}
\definecolor{deemph}{HTML}{b9b8b1}
\pgfplotsset{
  every axis/.append style={
    axis line style={gridc}, tick style={gridc}, grid style={gridc, line width=0.4pt},
    ymajorgrids, axis x line*=bottom, axis y line*=left,
    label style={font=\small, color=muted}, tick label style={font=\footnotesize, color=muted},
    legend style={font=\footnotesize, draw=none, fill=none, at={(0.02,0.98)}, anchor=north west, cells={anchor=west}},
    title style={font=\small, color=ink},
  }
}
\newcolumntype{L}[1]{>{\raggedright\arraybackslash}p{#1}}

\newcommand{\stcon}{\textsc{stcon}}
\newcommand{\Nw}{N}
\newcommand{\STCON}{\textsc{STCON}_{\mathrm{sf}}}
\newcommand{\V}{\mathcal{V}}
\newcommand{\TV}{T_{\mathcal{V}}}
\newcommand{\cA}{\mathcal{A}}
\newcommand{\ds}{\delta^{*}}
\newcommand{\live}{\mathrm{live}}
\newcommand{\chain}{\mathrm{chain}}
\newcommand{\tok}{\mathrm{tok}}
\newcommand{\need}{\mathrm{need}}
\newcommand{\cost}{\mathrm{cost}}
\newcommand{\eos}{\textsc{eos}}
\newcommand{\dead}{\bot}
\newcommand{\Fp}{\mathbb{F}_p}

\title{Breaking the Space Barrier and its Application to Language Model Inference}
\author{%
  Arip Asadulaev \\
  MBZUAI \\
}

\begin{document}
\maketitle

\begin{abstract}
Language models are more and more often asked for \emph{structured} output: JSON that follows a schema, or a tool call with typed arguments. A small machine, an automaton, enforces the format by forbidding the tokens that would break it. We observe that this machine has a rare property: from any of its states, each token leads along exactly one path. Graphs in which only a few paths join any two points are a classical object of complexity theory, and our theoretical result settles an open question about them: one can decide whether such a graph connects two points \emph{while verifying that it really has few paths}, with very little memory. Precisely, the problem lies in the classes ReachUL, LOGDCFL, C$_=$L and SC$^2$, and needs only $O(\log^2 n/\log\log n)$ space, below the classical $O(\log^2 n)$ of Savitch's theorem. The constructions behind the proofs become an inference engine: text the format forces is written without running the model, the mask is recomputed on the GPU without any table, recursive formats use a small stack, every output stays valid under a token limit, and independent fields are decoded in parallel and verified. On one 16\,GB Apple M2 Pro with Qwen3.5-2B and 4B, against MLX with llguidance, the standard setup for this hardware, schema-constrained extraction finishes $1.2$--$1.3\times$ sooner with the same answers, a grammar costs 3\,MB instead of up to 1.5\,GB, one server holds sixteen grammars where tables run out of memory, and sixteen tool-calling agents finish $2.5\times$ sooner.
\end{abstract}

\section{Introduction}
\label{sec:intro}

A language model writes text one token at a time. When an application needs the output in a fixed format, for example a JSON object with given keys or a call to a tool with typed arguments, the format is compiled into a small machine called an \emph{automaton}. At every step the automaton says which tokens are allowed; all others are forbidden (\emph{masked}). This is \emph{constrained decoding}, and libraries such as llguidance~\cite{llguidance}, Outlines~\cite{willard2023outlines} and XGrammar~\cite{dong2024xgrammar} provide it. On a laptop it is usually a cost: the model runs as before, and the mask is extra work at every step.

Our starting observation is simple. Fix a state of the automaton and a token. The token is a fixed string of bytes and the automaton is deterministic, so reading the token leads to exactly one next state. If we draw all the ways of reading tokens from a state, we get a tree: \emph{there is exactly one path to every point}. Graphs with this property, and more generally graphs in which at most $k$ paths join any two points (\emph{$k$-strongly few} graphs), are a classical topic of complexity theory, the field that studies how much time and memory problems need.

\textbf{The theory.} Apers and Edenhofer~\cite{apers2025} recently asked a precise question about such graphs. Suppose you are handed a graph and two points, and you must answer whether the graph has at most $k$ paths between every pair of points \emph{and} the second point can be reached from the first. How much memory does a classical computer need? The difficulty is that nothing is promised: the graph may contain cycles or astronomically many paths, and the algorithm must notice. We answer the question. The problem can be solved with memory logarithmic in the input size by a machine that never has two ways to reach the same state (the class ReachUL), by a deterministic machine with a stack (LOGDCFL), and by a single algorithm that is fast and uses $O(\log^2 n)$ memory at the same time (SC$^2$); it also lies in C$_=$L, and it needs only $O(\log^2 n/\log\log n)$ memory, less than Savitch's classical $O(\log^2 n)$ for reachability. Section~\ref{sec:theory} explains these results without assuming any background; the proofs are in Appendix~\ref{app:theory}.

\textbf{The practice.} The proofs are built from a few constructive ideas: stop every search at a budget, certify a structure before counting on it, recompute along unique paths instead of storing, keep a small capped stack, and split a problem at points every path must cross. Applied to the graph of one decoding step, these ideas become an engine (Section~\ref{sec:decoding}). Text the format forces is written without running the model, so a constrained output needs \emph{fewer} model runs than a free one; the mask is recomputed on the GPU in microseconds, so no table is stored; forced text is split into tokens exactly as the model's tokenizer would split it; recursive formats run on a small GPU stack; outputs stay valid under any feasible token limit; and independent fields of a JSON object are decoded in parallel and then verified. Table~\ref{tab:headline} summarizes the gains against MLX~\cite{mlx} with llguidance, the standard way to run constrained decoding on Apple hardware.

\begin{table}[t]
\centering
\footnotesize
\caption{\textbf{Headline results} on one Apple M2 Pro (16\,GB), Qwen3.5-2B / 4B in 4-bit precision, greedy decoding. The baseline is llguidance with MLX: in our decode loop with the same kernels for the first two rows, stock mlx-lm with llguidance for the last three.}
\label{tab:headline}
\setlength{\tabcolsep}{4pt}
\begin{tabular}{@{}L{5.6cm}L{3.3cm}L{3.3cm}@{}}
\toprule
Setting & MLX + llguidance & \stcon{} (ours) \\
\midrule
six JSON extractions, decode time (2B / 4B) & 2.41 / 4.59\,s & \textbf{1.89 / 3.89\,s} \\
memory of a 1,106-state tool-call grammar & does not compile (token table: 824\,MB) & \textbf{3.4\,MB} \\
sixteen tool-calling agents, wall time (2B) & 74.6\,s & \textbf{29.3\,s} \\
eight concurrent tool-call requests (2B) & 134 tok/s & \textbf{410 tok/s} \\
eighteen JSON jobs, 8 parallel slots, keys-only prompt (2B) & 6.12\,s & \textbf{3.83\,s} \\
\bottomrule
\end{tabular}
\end{table}

\textbf{Contributions.} (1)~We settle the classical complexity of total few-path reachability, the question of Apers and Edenhofer, with three new constructions, and extend them to arbitrary walk budgets, a compiler for programs with few computation paths, completeness results, a sharp threshold phenomenon and a structure theory (Section~\ref{sec:theory}, Appendix~\ref{app:theory}). (2)~We turn these constructions into six exact mechanisms for constrained decoding, each with a theorem (Section~\ref{sec:decoding}, Appendix~\ref{app:decoding}). (3)~We build them into an engine on MLX and measure speed, answer quality and memory (Section~\ref{sec:exp}, Appendices~\ref{app:speed}--\ref{app:quality}).

\section{Background}
\label{sec:background}

\textbf{Why a laptop counts model runs.} To produce one token, the model runs once over its whole network, and on a laptop that run is dominated by reading all the weights from memory (1.1\,GB for Qwen3.5-2B, 2.26\,GB for the 4B). The arithmetic is nearly free next to it. Two consequences guide everything below: a token that can be written \emph{without} running the model saves a whole run, and a run that processes several tokens at once costs barely more than a run for one (Fig.~\ref{fig:passcost-main}).

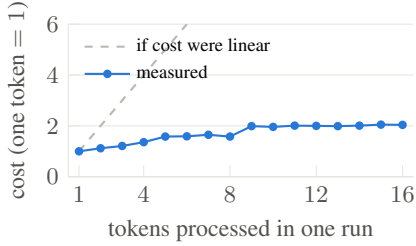
\begin{figure}[t]
\centering
\begin{minipage}[c]{0.44\linewidth}
\centering
\begin{tikzpicture}
\begin{axis}[width=\linewidth, height=3.6cm, xlabel={tokens processed in one run}, ylabel={cost (one token $=1$)}, xmin=0.5, xmax=16.5, ymin=0, ymax=6, xtick={1,4,8,12,16}, legend style={at={(0.02,0.98)}, anchor=north west, font=\scriptsize}]
\addplot[deemph, thick, dashed, domain=1:6] {x};
\addplot[ours, thick, mark=*, mark size=1.2pt] coordinates {(1,1.00)(2,1.12)(3,1.21)(4,1.36)(5,1.58)(6,1.59)(7,1.65)(8,1.58)(9,1.99)(10,1.96)(11,2.01)(12,2.00)(13,1.99)(14,2.01)(15,2.05)(16,2.04)};
\legend{if cost were linear, measured}
\end{axis}
\end{tikzpicture}
\end{minipage}\hfill
\begin{minipage}[c]{0.53\linewidth}
\caption{\textbf{What one model run costs} (Qwen3.5-2B on an M2 Pro). A run over 8 tokens costs $1.6\times$ a run over one, not $8\times$, because the run is dominated by reading the weights. So skipping runs saves time, and tokens added to a run that happens anyway are almost free.}
\label{fig:passcost-main}
\end{minipage}
\end{figure}

\textbf{Constraints.} A JSON schema or a regular expression is compiled into a deterministic automaton over bytes: a few hundred states, and in each state a set of allowed bytes. The vocabulary of Qwen3.5 has 248k tokens, each a string of bytes. At every step the decoder needs the \emph{mask} (which tokens keep the output valid) and, after sampling, the next state. Existing libraries compute the mask on the CPU at every step; on a laptop this forces the GPU to wait for the CPU once per token.

\textbf{Graphs, walks and few paths.} A directed graph has vertices and arrows. A \emph{walk} follows arrows and may revisit vertices; standing still counts as one walk. A graph is \emph{$k$-strongly few} if at most $k$ walks join any two vertices, and a \emph{mangrove} if at most one does (Fig.~\ref{fig:graph}). Such a graph cannot contain a cycle, because going around a cycle again and again gives infinitely many walks. A \emph{forced chain} is a run of vertices with one way in and one way out: whoever enters it must walk to its end.

\begin{figure}[t]
\centering
\begin{tikzpicture}[>=Stealth, every node/.style={circle, draw=muted, fill=white, minimum size=5mm, inner sep=0pt, font=\scriptsize}, edge/.style={->, muted, thick}]
  \node (a) at (0,0) {$s$}; \node (b) at (1.4,0.6) {}; \node (c) at (1.4,-0.6) {}; \node (d) at (2.8,0) {$t$};
  \draw[edge] (a)--(b); \draw[edge] (b)--(d); \draw[edge] (a)--(c); \draw[edge] (c)--(d);
  \node[draw=none, fill=none, shape=rectangle, font=\footnotesize, color=muted] at (1.4,-1.0) {two walks: $\Nw(s,t)=2$};
  \begin{scope}[xshift=5cm]
  \node (a) at (0,0) {$s$}; \node (b) at (1.4,0) {}; \node (c) at (2.8,0) {}; \node (d) at (4.4,0) {$t$};
  \draw[edge, ours, very thick] (a)--(b); \draw[edge, ours, very thick] (b)--(c);
  \draw[edge] (c) to[bend left=40] (d); \draw[edge] (c) to[bend right=40] (d);
  \node[draw=none, fill=none, shape=rectangle, font=\footnotesize, color=muted] at (2.2,-1.0) {a forced chain (blue), then a choice};
  \end{scope}
\end{tikzpicture}
\caption{\textbf{Left:} the smallest graph with two walks between a pair; it is 2-strongly few. \textbf{Right:} a vertex with one way in and one way out is \emph{forced}: whoever enters a forced chain walks it to its end. Theorem~\ref{thm:kernel} says that once forced chains are contracted, at most $2k$ vertices remain between any two vertices; in decoding, a forced chain is output the model is not asked about.}
\label{fig:graph}
\end{figure}
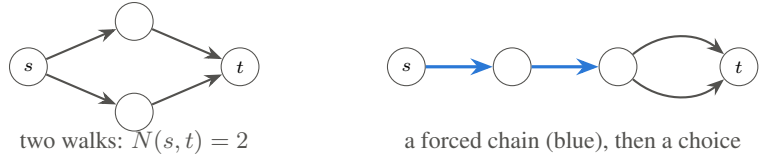

\textbf{Complexity classes in plain words.} Complexity theory measures problems by the resources a machine needs. Here the resource is \emph{memory}, measured in bits as a function of the input size $n$; ``logarithmic space'' means a few counters of $\log n$ bits, enough to remember a handful of vertex names but not the graph. Table~\ref{tab:classes} lists the classes that appear in our results. The classical benchmark is directed reachability: it is in NL, and the best known deterministic algorithm, Savitch's, uses $O(\log^2 n)$ space.

\begin{table}[t]
\centering
\footnotesize
\caption{\textbf{The complexity classes in our results, in plain words.} All machines use logarithmic memory unless stated otherwise.}
\label{tab:classes}
\begin{tabular}{@{}lL{10.6cm}@{}}
\toprule
Class & What the machine may do \\
\midrule
L & compute deterministically \\
NL & guess, and accept if \emph{some} sequence of guesses succeeds (directed reachability is the typical problem) \\
ReachUL & guess, but every state of the machine is reachable by \emph{at most one} sequence of guesses: no ambiguity anywhere \\
LOGDCFL & compute deterministically with an extra stack of unlimited size, in polynomial time \\
SC$^2$ & one deterministic algorithm that is fast (polynomial time) and uses $O(\log^2 n)$ memory simultaneously \\
C$_=$L & decide whether a signed count of computation paths is exactly zero \\
PL & decide whether such a count is positive (much more powerful) \\
\bottomrule
\end{tabular}
\end{table}

\section{The theory in plain words}
\label{sec:theory}

\textbf{The problem.} The input is a graph $G$, two vertices $s,t$ and a number $k$. Answer \emph{yes} exactly when $G$ is $k$-strongly few and $t$ can be reached from $s$. We call this problem $\STCON$. The word \emph{and} is the whole difficulty: the algorithm is not promised a good graph. It must also say \emph{no} on graphs with cycles or with exponentially many walks, and it must do so with the same tiny memory. Think of a ticket inspector who must not only check that a ticket is valid but also detect forgeries, without ever taking notes longer than a few lines.

\begin{theorem}[Main theorem, informal]
\label{thm:informal}
$\STCON$ can be solved by an unambiguous machine with logarithmic memory (ReachUL), by a deterministic machine with a stack (LOGDCFL), by a single fast algorithm with $O(\log^2 n)$ memory (SC$^2$), and by a zero test of a signed count (C$_=$L). It needs only $O(\log^2 n/\log\log n)$ memory, less than Savitch's $O(\log^2 n)$.
\end{theorem}

Theorem~\ref{thm:main} in Appendix~\ref{app:formaltheory} is the formal version. Apers and Edenhofer had shown that a \emph{quantum} computer with logarithmic memory solves the problem and asked whether these classical bounds hold; Theorem~\ref{thm:informal} answers yes. Three ideas prove it, and each reappears in the engine.

\textbf{Idea 1: stop every search at a budget.} If every vertex may be the end of at most $k$ walks from a start, the total budget is $B=nk$. A walk of length $B$ has $B+1$ prefixes, all different walks, which is more than the budget allows. So it suffices to explore walks of length at most $B$: either everything has been seen, or a violation has already been found. With this cap, a depth-first search that keeps only a stack of ``where to continue'' always stops, even on a cyclic graph, and this gives the stack machine (Theorem~\ref{thm:logdcfl}).

\textbf{Idea 2: certify, then count.} Give every walk a short fingerprint, a polynomial hash of its edges, and build the graph of (endpoint, length, fingerprint). If no two walks collide, counting the reachable fingerprints counts the walks. The machine first \emph{certifies} that no collision happened, by a check that itself never becomes ambiguous, and only then counts (Theorem~\ref{thm:reachul}). If the graph is good, some choice of hash has no collision; if the graph is bad, either no hash certifies or the count exposes the violation.

\textbf{Idea 3: one equation for everything.} For each pair of vertices, the product $(N-0)(N-1)\cdots(N-k)$ is zero exactly when the walk count $N$ is at most $k$. Summing the squares of all these products gives a single number that is zero exactly on good inputs (Theorem~\ref{thm:ceql}).

\textbf{What else follows.} The same ideas validate any system of walk budgets, not just a uniform $k$; they turn any program whose number of computation paths is \emph{expected} to be small into one that checks this on every input (Theorem~\ref{thm:compiler}); fixing $k$ gives problems that are as hard as anything in ReachUL (Theorem~\ref{thm:complete}); and writing $k$ in binary instead of unary makes the problem jump to PL-complete (Theorem~\ref{thm:phase}). Few-path graphs also have a clean shape (Appendix~\ref{app:formaltheory}): between two vertices almost everything is a forced chain, at most $2k$ branching vertices remain (Theorem~\ref{thm:kernel}), and the paths split into independent pieces at vertices that every path crosses (Theorem~\ref{thm:block}). The engine uses exactly these two facts.

\section{From the theory to the decoding engine}
\label{sec:decoding}

\textbf{The step graph is a mangrove.} Store the vocabulary as a \emph{trie}: a tree of byte strings in which tokens share their prefixes (0.59 million nodes for Qwen3.5). Fix an automaton state $q$ and pair every trie node with the state the automaton reaches after reading that node's bytes from $q$ (Fig.~\ref{fig:stepgraph}). Each node has one parent, because the trie fixes the bytes and the automaton is deterministic, so this \emph{step graph} is a tree, a mangrove. Every token sits at exactly one node, which carries its next state (Proposition~\ref{prop:mangrove}). Six mechanisms follow (Table~\ref{tab:mech}); Appendix~\ref{app:formaldecoding} states each as a theorem and Appendix~\ref{app:decoding} proves it.

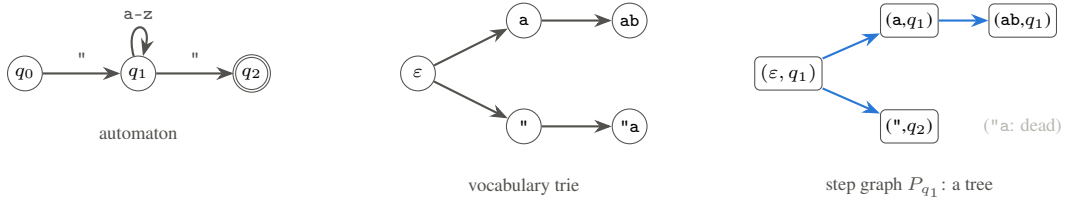
\begin{figure}[t]
\centering
\begin{tikzpicture}[>=Stealth, every node/.style={font=\scriptsize}, st/.style={circle, draw=muted, fill=white, minimum size=4.6mm, inner sep=0pt}, edge/.style={->, muted, thick}]
  \node[st] (q0) at (0,0) {$q_0$}; \node[st] (q1) at (1.5,0) {$q_1$}; \node[st, double] (q2) at (3,0) {$q_2$};
  \draw[edge] (q0)--node[above]{\texttt{"}} (q1); \draw[edge] (q1) to[loop above] node{\texttt{a-z}} (q1); \draw[edge] (q1)--node[above]{\texttt{"}} (q2);
  \node[color=muted] at (1.5,-0.8) {automaton};
  \node[st] (e) at (5.2,0) {$\varepsilon$}; \node[st] (a) at (6.6,0.7) {\texttt{a}}; \node[st] (ab) at (8,0.7) {\texttt{ab}}; \node[st] (qt) at (6.6,-0.7) {\texttt{"}}; \node[st] (qa) at (8,-0.7) {\texttt{"a}};
  \draw[edge] (e)--(a); \draw[edge] (a)--(ab); \draw[edge] (e)--(qt); \draw[edge] (qt)--(qa);
  \node[color=muted] at (6.6,-1.5) {vocabulary trie};
  \tikzset{pr/.style={draw=muted, fill=white, rounded corners=2pt, inner sep=2pt, minimum height=4.4mm}}
  \node[pr] (p0) at (10.1,0) {$(\varepsilon,q_1)$}; \node[pr] (p1) at (11.7,0.7) {(\texttt{a},$q_1$)}; \node[pr] (p2) at (13.2,0.7) {(\texttt{ab},$q_1$)}; \node[pr] (p3) at (11.7,-0.7) {(\texttt{"},$q_2$)};
  \draw[edge, ours] (p0)--(p1); \draw[edge, ours] (p1)--(p2); \draw[edge, ours] (p0)--(p3);
  \node[color=deemph] (dead) at (13.2,-0.7) {(\texttt{"a}: dead)};
  \node[color=muted] at (11.7,-1.5) {step graph $P_{q_1}$: a tree};
\end{tikzpicture}
\caption{\textbf{The step graph.} Left: an automaton for a quoted word. Middle: part of the vocabulary trie. Right: their product from $q_1$. Every trie node appears with the one state the automaton is in when it gets there, so there is exactly one route to every vertex. The token \texttt{"a} is dead from $q_1$ because no letter may follow the closing quote.}
\label{fig:stepgraph}
\end{figure}

\textbf{1. The mask without a table (from ``recompute, do not store'').} A naive engine stores, for every state and token, the next state: $3|Q|\times248\text{k}$ bytes, 77\,MB for a 104-state schema and 1.5\,GB for 1,970 states. Because every token has a single path, the needed row can be recomputed at every step: a GPU kernel runs one thread per token, walks the token's bytes through the automaton and writes the state it reaches. This takes 7--18\,$\mu$s per step against a 7--15\,ms model run, and a grammar costs only its automaton (512 bytes per state) plus 2.9\,MB of vocabulary shared by all grammars (Theorem~\ref{thm:table}). Since the mask lives on the GPU, the GPU never waits for the CPU.

\textbf{2. Forced text is free (from forced chains).} Much of a JSON output is not chosen by the model: after \texttt{"name":"Maria"} the schema allows only \texttt{,"age":}. A state with one allowed byte is a forced chain. After each sampled token the engine feeds the whole forced continuation \emph{in the same model run}, so it costs nothing extra (Fig.~\ref{fig:strip}). The output is exactly what token-by-token decoding would give, and no decoder that does not speculate can use fewer runs (Theorem~\ref{thm:pass}). A key costs no run at all, however long it is, and a field with $r$ possible values costs at most $r-1$ runs (Corollary~\ref{cor:length}).

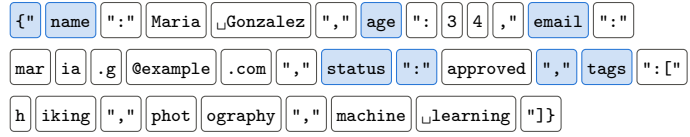
\begin{figure}[t]
\centering
\begin{tikzpicture}[tk/.style={font=\scriptsize\ttfamily, inner sep=2pt, text height=1.6ex, text depth=0.45ex, draw=muted, rounded corners=1.5pt, anchor=west}, s/.style={tk, fill=white}, f/.style={tk, fill=ours!22, draw=ours}]
\node[f] (r0c0) at (0,-0.00) {\{"};
\node[f, right=1.5pt of r0c0] (r0c1) {name};
\node[s, right=1.5pt of r0c1] (r0c2) {":"};
\node[s, right=1.5pt of r0c2] (r0c3) {Maria};
\node[s, right=1.5pt of r0c3] (r0c4) {\textvisiblespace Gonzalez};
\node[s, right=1.5pt of r0c4] (r0c5) {","};
\node[f, right=1.5pt of r0c5] (r0c6) {age};
\node[s, right=1.5pt of r0c6] (r0c7) {":};
\node[s, right=1.5pt of r0c7] (r0c8) {3};
\node[s, right=1.5pt of r0c8] (r0c9) {4};
\node[s, right=1.5pt of r0c9] (r0c10) {,"};
\node[f, right=1.5pt of r0c10] (r0c11) {email};
\node[s, right=1.5pt of r0c11] (r0c12) {":"};
\node[s] (r1c0) at (0,-0.62) {mar};
\node[s, right=1.5pt of r1c0] (r1c1) {ia};
\node[s, right=1.5pt of r1c1] (r1c2) {.g};
\node[s, right=1.5pt of r1c2] (r1c3) {@example};
\node[s, right=1.5pt of r1c3] (r1c4) {.com};
\node[s, right=1.5pt of r1c4] (r1c5) {","};
\node[f, right=1.5pt of r1c5] (r1c6) {status};
\node[f, right=1.5pt of r1c6] (r1c7) {":"};
\node[s, right=1.5pt of r1c7] (r1c8) {approved};
\node[f, right=1.5pt of r1c8] (r1c9) {","};
\node[f, right=1.5pt of r1c9] (r1c10) {tags};
\node[s, right=1.5pt of r1c10] (r1c11) {":["};
\node[s] (r2c0) at (0,-1.24) {h};
\node[s, right=1.5pt of r2c0] (r2c1) {iking};
\node[s, right=1.5pt of r2c1] (r2c2) {","};
\node[s, right=1.5pt of r2c2] (r2c3) {phot};
\node[s, right=1.5pt of r2c3] (r2c4) {ography};
\node[s, right=1.5pt of r2c4] (r2c5) {","};
\node[s, right=1.5pt of r2c5] (r2c6) {machine};
\node[s, right=1.5pt of r2c6] (r2c7) {\textvisiblespace learning};
\node[s, right=1.5pt of r2c7] (r2c8) {"]\}};
\end{tikzpicture}
\caption{\textbf{One constrained output, token by token} (Qwen3.5-2B, the \texttt{person} schema). White: sampled by the model (27 tokens, 27 passes). Blue: forced by the constraint (7 tokens, fed in the passes of the preceding sampled tokens). The key splits (\texttt{name} then \texttt{":"}) are the tokenizer's own (Section~\ref{sec:canon}).}
\label{fig:strip}
\end{figure}

\textbf{3. Forced text in the tokenizer's own split (from the block product).} The forced bytes \texttt{":"} can be split into tokens in several ways, but the model was trained on the tokenizer's split only. Tokenizers first cut text into words, numbers and punctuation runs; these cuts are points every tokenization crosses, so the split factors into independent blocks and only the last block can change when more text arrives. The engine commits every block except possibly the last, and a one-byte test certifies when the last one is safe too (Theorem~\ref{thm:canon}).

\textbf{4. Never cut a value in half (from budgets).} If an output reaches the token limit in the middle of a value, it does not parse. The engine precomputes, for every state, the fewest tokens needed to finish validly (a shortest-path table), and allows a token only if a valid ending still fits. Every output is then valid under every limit that admits a valid output at all, and no less restrictive rule has this guarantee (Theorem~\ref{thm:budget}).

\textbf{5. Recursive formats with a small stack (from the stack machine).} A schema for a tree of nodes refers to itself, so no finite automaton describes it. As in Idea 1, the engine uses an automaton plus a capped stack of return points, kept on the GPU (64 entries). We prove that a token depends only on the stack entries it pops, so a table indexed by the top two entries suffices (Theorem~\ref{thm:pda}).

\begin{figure}[t]
\centering
\begin{tikzpicture}[>=Stealth, every node/.style={font=\scriptsize}, st/.style={circle, draw=muted, fill=white, minimum size=4.6mm, inner sep=0pt}, edge/.style={->, muted, thick}]
  \node[st] (a) at (0,0) {$q_a$}; \node[st] (b) at (1.8,0) {$q_b$}; \node[st] (r) at (3.6,0) {$r$}; \node[st] (c) at (5.4,0) {$q_c$};
  \draw[edge] (a)--node[above]{\texttt{":[}} (b); \draw[edge] (r)--node[above]{\texttt{]\}}} (c);
  \draw[edge, ours, very thick] (b) to[bend left=25] node[above, color=ink]{call: push $r$} (r);
  \node[st] (s0) at (2.7,-1.5) {}; \node[st] (m) at (4.5,-1.5) {}; \node[st, double] (t) at (6.3,-1.5) {};
  \draw[edge, ours, very thick] (b)--(s0); \draw[edge] (s0)--node[above]{\texttt{\{"name":}$\cdots$} (m); \draw[edge] (m)--node[above]{\texttt{\}}} (t);
  \draw[edge, ours, very thick] (t) to[bend right=30] (r);
  \node[color=ink] at (6.9,-0.55) {return: pop, continue at $r$};
  \node[draw=muted, fill=white, inner sep=3pt, align=left, shape=rectangle] at (10.2,-0.7) {state on the GPU:\\$[\,q,\ \text{stack pointer},\ \text{stack}[0..63]\,]$};
\end{tikzpicture}
\caption{\textbf{A recursive schema as an automaton with a stack.} A child node is a ``call'': the engine pushes the place to return to ($r$) and starts a new node; the node's closing bracket pops it. The whole state, stack included, is a small integer vector on the GPU.}
\label{fig:pda-main}
\end{figure}
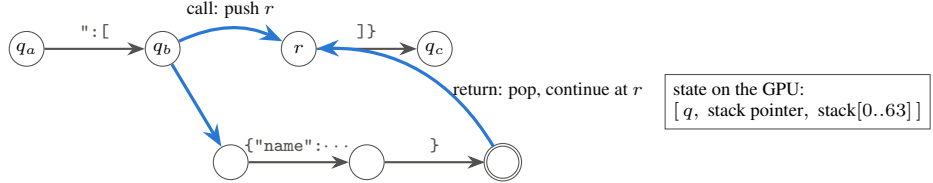

\textbf{6. Independent fields in parallel, then verified (from ``certify, then count'').} The language of a JSON object splits at each key, a point every valid output crosses, so the values can be drafted in parallel as rows of one batch. The model then checks the whole draft in one run, position by position; the first wrong token is corrected and its field continues. The result is token-identical to sequential greedy decoding (Theorem~\ref{thm:cert}).

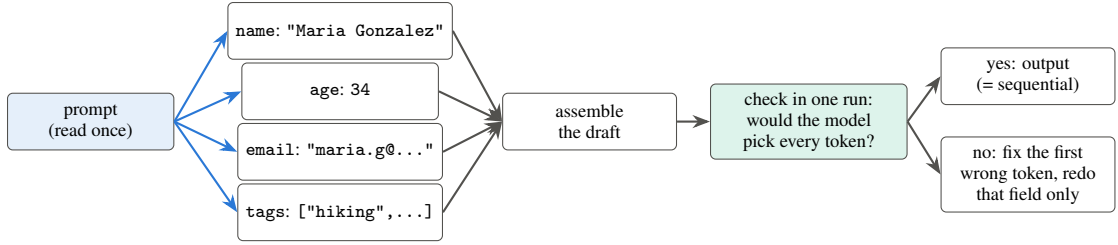
\begin{figure}[t]
\centering
\begin{tikzpicture}[>=Stealth, every node/.style={font=\scriptsize}, box/.style={draw=muted, fill=white, rounded corners=2pt, inner sep=3pt, align=center, minimum height=0.75cm}]
  \node[box, fill=ours!12, minimum width=2.2cm] (p) at (0,0) {prompt\\(read once)};
  \node[box, minimum width=2.6cm] (b1) at (3.3,1.2) {\texttt{name}: \texttt{"Maria Gonzalez"}};
  \node[box, minimum width=2.6cm] (b2) at (3.3,0.4) {\texttt{age}: \texttt{34}};
  \node[box, minimum width=2.6cm] (b3) at (3.3,-0.4) {\texttt{email}: \texttt{"maria.g@..."}};
  \node[box, minimum width=2.6cm] (b4) at (3.3,-1.2) {\texttt{tags}: \texttt{["hiking",...]}};
  \node[box, minimum width=2.3cm] (asm) at (6.6,0) {assemble\\the draft};
  \node[box, fill=third!15, minimum width=2.6cm] (cert) at (9.5,0) {check in one run:\\would the model\\pick every token?};
  \node[box, minimum width=2.3cm] (out) at (12.4,0.6) {yes: output\\(= sequential)};
  \node[box, minimum width=2.3cm] (fix) at (12.4,-0.7) {no: fix the first\\wrong token, redo\\that field only};
  \foreach \b in {b1,b2,b3,b4} { \draw[->, ours, thick] (p.east) -- (\b.west); \draw[->, muted, thick] (\b.east) -- (asm.west); }
  \draw[->, muted, thick] (asm) -- (cert); \draw[->, muted, thick] (cert.east) -- (out.west); \draw[->, muted, thick] (cert.east) -- (fix.west);
\end{tikzpicture}
\caption{\textbf{Certified parallel decoding.} The fields of a JSON object are drafted at the same time, as rows of one batch, then the model checks the whole draft in one run. The result is identical to decoding token by token.}
\label{fig:waves-main}
\end{figure}

\begin{table}[t]
\centering
\footnotesize
\caption{\textbf{From the theory to the engine.} Each mechanism reuses an idea of Section~\ref{sec:theory}; the guarantee is proved in Appendix~\ref{app:decoding}.}
\label{tab:mech}
\setlength{\tabcolsep}{4pt}
\begin{tabular}{@{}L{3.4cm}L{4.0cm}L{5.4cm}@{}}
\toprule
Mechanism & Idea from the theory & Guarantee \\
\midrule
mask without a table & recompute along the unique path & same mask as a full table; memory = automaton \\
forced text in the same run & forced chains & same output; fewest model runs possible without speculation \\
tokenizer's split & paths split at crossing points & committed tokens are the tokenizer's own split \\
budget rule & stop at a budget; shortest paths & valid under every feasible limit; least restrictive \\
GPU stack for recursion & capped stack machine & exact up to nesting depth 64 \\
certified parallel fields & split at crossing points; certify, then use & identical to sequential greedy decoding \\
\bottomrule
\end{tabular}
\end{table}

\section{Experiments}
\label{sec:exp}

\textbf{Setup.} All measurements run on one Apple M2 Pro laptop (19-core GPU, 16\,GB of unified memory) with two current open models, Qwen3.5-2B and Qwen3.5-4B, in 4-bit precision and with greedy decoding. The baseline is \textbf{MLX with llguidance}~\cite{mlx,llguidance}, the standard way to run structured generation on Apple hardware. We use it in two forms. For single requests, llguidance runs inside our MLX decode loop with exactly the same matrix kernels, in its mask-only mode and in its fast-forward mode, so that every difference comes from how the constraint is handled. For serving, batches and agents, we compare the whole engine with stock mlx-lm plus llguidance. The workloads are six JSON extractions with schemas (a person record, a tool call, a product, an event with attendees, a review classification and a contact card), four recursive schemas (trees of nodes), and coding agents that call tools. We report the time to finish a task, the number of model runs, and answer quality; every timing is a median of interleaved rounds in one process.

\noindent\colorbox{ours!8}{\parbox{\dimexpr\linewidth-2\fboxsep\relax}{\small\textbf{Our practical results at a glance.}
\textbf{(1)~Faster:} structured extraction finishes \textbf{$1.28\times$} (2B) and \textbf{$1.18\times$} (4B) sooner than MLX + llguidance, and \textbf{$1.8\times$} sooner than its mask-only mode, with \textbf{31\%} of all tokens written without running the model.
\textbf{(2)~Same answers:} every output valid, \textbf{45/45} fact fields correct, character-identical outputs on the 2B.
\textbf{(3)~Small:} a grammar costs \textbf{3\,MB} instead of up to \textbf{1.5\,GB}; one server keeps \textbf{16} grammars where token tables crash at the 15th.
\textbf{(4)~Guaranteed:} \textbf{90/90} outputs valid under tight token limits (3/90 without our rule); parallel decoding reproduces greedy outputs with \textbf{$1.8\times$} fewer model runs.
\textbf{(5)~In real use:} \textbf{$2.55\times$} faster for sixteen coding agents and \textbf{$3.1\times$} the throughput for concurrent tool calls against stock mlx-lm + llguidance.}}

\subsection{Structured output is faster}

\begin{table}[ht]
\centering
\begin{minipage}[t]{0.5\linewidth}
\centering
\footnotesize
\caption{\textbf{Speed}, six JSON extractions, one request at a time; llguidance runs in the same MLX loop. Decode time for all six outputs and time per model run.}
\label{tab:speed}
\setlength{\tabcolsep}{3pt}
\begin{adjustbox}{max width=\linewidth}
\begin{tabular}{@{}lrrr@{}}
\toprule
 & \multicolumn{2}{c}{decode time (s)} & ms per run \\
\cmidrule(lr){2-3}
engine & 2B & 4B & 2B / 4B \\
\midrule
llguidance, mask only & 3.34 & 6.08 & 9.7 / 18.3 \\
llguidance, fast-forward & 2.41 & 4.59 & 9.8 / 18.8 \\
\stcon{} (ours) & \textbf{1.89} & \textbf{3.89} & \textbf{8.3 / 17.1} \\
\midrule
speedup vs fast-forward & $\mathbf{1.28\times}$ & $\mathbf{1.18\times}$ & \\
speedup vs mask only & $\mathbf{1.77\times}$ & $\mathbf{1.56\times}$ & \\
\bottomrule
\end{tabular}
\end{adjustbox}
\end{minipage}\hfill
\begin{minipage}[t]{0.46\linewidth}
\centering
\footnotesize
\caption{\textbf{Model runs} for the same six outputs (2B): runs per output token and tokens written without a run.}
\label{tab:runs}
\setlength{\tabcolsep}{3pt}
\begin{adjustbox}{max width=\linewidth}
\begin{tabular}{@{}lrr@{}}
\toprule
engine & runs / token & forced \\
\midrule
free text, no constraint & 1.01 & 0\% \\
llguidance, mask only & 1.00 & 0\% \\
llguidance, fast-forward & 0.72 & 28\% \\
\stcon{} (ours) & \textbf{0.69} & \textbf{31\%} \\
\bottomrule
\end{tabular}
\end{adjustbox}
\end{minipage}
\end{table}

Table~\ref{tab:speed} is the main speed result: on the same prompts, with the same model and the same kernels, \stcon{} finishes the six extractions in \textbf{1.89\,s} on the 2B where MLX with llguidance needs 2.41\,s in its best mode and 3.34\,s with masks only; on the 4B, 3.89\,s against 4.59\,s and 6.08\,s. The gain decomposes exactly into the two mechanisms of Section~\ref{sec:decoding}. First, \emph{fewer model runs}: forced text rides along with sampled tokens, so the six outputs need 226 runs instead of 343 (Table~\ref{tab:runs}); a constrained stream needs fewer runs per token (0.69) than unconstrained text (1.01). Second, \emph{cheaper runs}: because the mask is computed on the GPU, the GPU never waits for the CPU, and a run costs 8.3\,ms instead of 9.7\,ms. Multiplying the two, $343\times9.7\,\text{ms}=3.3$\,s becomes $226\times8.3\,\text{ms}=1.9$\,s. llguidance's fast-forward mode also skips forced tokens, but it keeps the CPU-side mask, so each of its runs costs 9.8\,ms and it ends up $1.28\times$ slower.

\subsection{The answers stay the same}

\begin{table}[ht]
\centering
\begin{minipage}[t]{0.46\linewidth}
\centering
\footnotesize
\caption{\textbf{Answer quality}, six JSON extractions (2B / 4B). Gold: 45 fact fields read off the source texts.}
\label{tab:quality}
\setlength{\tabcolsep}{3pt}
\begin{adjustbox}{max width=\linewidth}
\begin{tabular}{@{}lrr@{}}
\toprule
 & llguidance & \stcon{} \\
\midrule
valid outputs & 6/6, 6/6 & \textbf{6/6, 6/6} \\
fact fields correct & 45/45, 45/45 & \textbf{45/45, 45/45} \\
values as in free answer & 49/52, 47/55 & 49/52, 47/55 \\
\bottomrule
\end{tabular}
\end{adjustbox}
\end{minipage}\hfill
\begin{minipage}[t]{0.5\linewidth}
\centering
\footnotesize
\caption{\textbf{How forced text is split into tokens} (2B, same constraint): agreement with the model's free answer, and model runs.}
\label{tab:policy}
\setlength{\tabcolsep}{3pt}
\begin{adjustbox}{max width=\linewidth}
\begin{tabular}{@{}lrr@{}}
\toprule
how forced text is fed & agreement & runs \\
\midrule
llguidance, mask only & 0.95 & 341 \\
llguidance, fast-forward & 0.95 & 245 \\
model chooses the split & 0.93 & 343 \\
whole chain, any split & 0.93 & 246 \\
\stcon{}: tokenizer's own split & \textbf{0.95} & \textbf{224} \\
\bottomrule
\end{tabular}
\end{adjustbox}
\end{minipage}
\end{table}

Speed would be worthless if the answers changed. Table~\ref{tab:quality} checks three things: whether outputs parse and follow the schema, whether the 45 facts that the source texts state (names, ages, prices, dates, cities, phone numbers, enum values and list items) are extracted correctly, and how many values coincide with what the model writes without any constraint. Both engines are perfect on validity and on the facts, and on the 2B the outputs of \stcon{} and llguidance are character for character identical. Table~\ref{tab:policy} shows why the tokenizer rule of Section~\ref{sec:decoding} matters. If forced text is fed in any split, or the model is left to choose the split, agreement drops: in one case the model writes ``Just-Luc Moreau'' where it otherwise writes ``Jean-Luc''. Feeding forced text exactly as the tokenizer would split it restores the agreement and also needs the fewest runs of all, 224.

\subsection{Recursive formats keep their speed}

\begin{table}[ht]
\centering
\begin{minipage}[c]{0.5\linewidth}
\centering
\footnotesize
\caption{\textbf{Recursive schemas} (trees), one request: time per model run (2B / 4B), throughput and valid outputs (2B).}
\label{tab:recursive-main}
\setlength{\tabcolsep}{3pt}
\begin{adjustbox}{max width=\linewidth}
\begin{tabular}{@{}lrrr@{}}
\toprule
engine & ms / run & tok/s & valid \\
\midrule
llguidance, mask only & 9.8 / 18.0 & 94 & 4/4 \\
llguidance, fast-forward & 10.0 / 18.8 & 143 & 4/4 \\
\stcon{} (GPU stack) & \textbf{8.7 / 17.0} & \textbf{170} & \textbf{4/4} \\
\bottomrule
\end{tabular}
\end{adjustbox}
\end{minipage}\hfill
\begin{minipage}[c]{0.46\linewidth}
\small
A schema that refers to itself, such as a tree of nodes with children, cannot be written as a regular expression; the regex library we build on unrolls three levels and then silently drops the recursive field. Our GPU stack handles any depth up to 64, at \textbf{170 tokens/s} against 143 and 94 for llguidance on the 2B (Table~\ref{tab:recursive-main}).
\end{minipage}
\end{table}

Recursion costs almost nothing extra because the stack lives on the GPU together with the automaton state, so the decode loop never waits for the CPU, and because forced text crosses the stack's returns: after the last value of a node, the closing brackets and the next opening key, \texttt{]\},\{"name":"}, are fed as one forced span. Every engine produces valid trees on these tasks; when the model nests without end, our engine stops cleanly at the depth limit.

\subsection{A grammar costs megabytes, not gigabytes}

\begin{table}[ht]
\centering
\begin{minipage}[t]{0.5\linewidth}
\centering
\footnotesize
\caption{\textbf{Constraint memory per grammar} (248k-token vocabulary). Ours: automaton only, plus 2.9\,MB shared by all grammars.}
\label{tab:memgrammar}
\setlength{\tabcolsep}{3pt}
\begin{adjustbox}{max width=\linewidth}
\begin{tabular}{@{}lrrr@{}}
\toprule
grammar (states) & token table & llguidance & \stcon{} \\
\midrule
person schema (104) & 77\,MB & $<$0.5\,MB & \textbf{0.05\,MB} \\
typed tool calls (1,106) & 824\,MB & fails & \textbf{0.57\,MB} \\
96 fields (1,970) & 1.5\,GB & $<$0.5\,MB & \textbf{1.0\,MB} \\
sixteen grammars & 7.2\,GB & 4.7\,MB & \textbf{4.9\,MB} \\
shared, paid once & --- & 153\,MB & \textbf{2.9\,MB} \\
\bottomrule
\end{tabular}
\end{adjustbox}
\end{minipage}\hfill
\begin{minipage}[t]{0.46\linewidth}
\centering
\footnotesize
\caption{\textbf{One server, sixteen grammars} in turn: grammars served, GPU memory above the model, longest wait for a first answer.}
\label{tab:server}
\setlength{\tabcolsep}{3pt}
\begin{adjustbox}{max width=\linewidth}
\begin{tabular}{@{}llrrr@{}}
\toprule
 & mask & served & memory & wait \\
\midrule
2B & tables & 15, crash & 8.2\,GB & 26\,s \\
 & ours & \textbf{16} & \textbf{0.6\,GB} & \textbf{2.8\,s} \\
4B & tables & 14, crash & 6.9\,GB & 26\,s \\
 & ours & \textbf{16} & \textbf{1.1\,GB} & \textbf{6.6\,s} \\
\bottomrule
\end{tabular}
\end{adjustbox}
\end{minipage}
\end{table}

A precomputed mask table grows with the grammar and stays in memory for the life of the server: 77\,MB for a five-field schema, 824\,MB for a tool-call grammar, 1.5\,GB for a 96-field form. Because our mask is recomputed on the GPU at every step in 7--18\,$\mu$s, a grammar costs only its automaton, \textbf{a few hundred kilobytes}, and decoding runs at the table's speed with identical outputs (Table~\ref{tab:memgrammar}). This changes what a laptop server can do. The \emph{typed} tool-call grammar, which fixes the name and the type of every argument of every tool, has 1,106 states; llguidance cannot compile it and its table does not fit in the server, while with our kernel it costs 3.4\,MB and \textbf{46 of 46} tool calls from eight concurrent agents arrive with every argument present and correctly typed, against 43 of 52 with an untyped grammar at the same speed. When a server sees many different grammars, as an agent platform does, tables accumulate until the process crashes; ours serves all sixteen within about 1\,GB (Table~\ref{tab:server}, Fig.~\ref{fig:wild-main}).

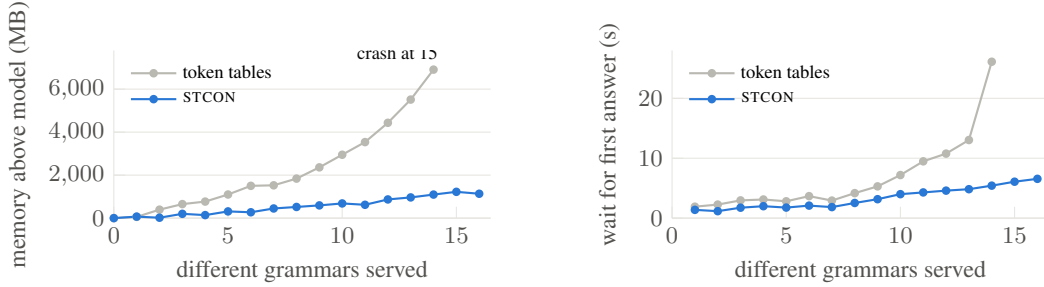
\begin{figure}[ht]
\centering
\begin{tikzpicture}
\begin{axis}[width=0.47\textwidth, height=3.8cm, xlabel={different grammars served}, ylabel={memory above model (MB)}, xmin=0, xmax=16.5, ymin=0, ymax=7800, legend style={at={(0.02,0.98)}, anchor=north west, font=\scriptsize}]
\addplot[deemph, thick, mark=*, mark size=1.2pt] coordinates {(0,0) (1,57) (2,394) (3,652) (4,769) (5,1098) (6,1502) (7,1526) (8,1838) (9,2361) (10,2948) (11,3536) (12,4434) (13,5514) (14,6908)};
\addplot[ours, thick, mark=*, mark size=1.2pt] coordinates {(0,0) (1,70) (2,19) (3,200) (4,140) (5,312) (6,272) (7,449) (8,519) (9,592) (10,683) (11,620) (12,869) (13,963) (14,1095) (15,1223) (16,1134)};
\node[font=\scriptsize, color=ink, anchor=south east] at (axis cs:14.6,6950) {crash at 15};
\legend{token tables, \stcon{}}
\end{axis}
\end{tikzpicture}\hfill
\begin{tikzpicture}
\begin{axis}[width=0.47\textwidth, height=3.8cm, xlabel={different grammars served}, ylabel={wait for first answer (s)}, xmin=0, xmax=16.5, ymin=0, ymax=28, legend style={at={(0.02,0.98)}, anchor=north west, font=\scriptsize}]
\addplot[deemph, thick, mark=*, mark size=1.2pt] coordinates {(1,1.91) (2,2.26) (3,2.96) (4,3.11) (5,2.81) (6,3.67) (7,2.93) (8,4.17) (9,5.30) (10,7.19) (11,9.48) (12,10.77) (13,13.04) (14,26.11)};
\addplot[ours, thick, mark=*, mark size=1.2pt] coordinates {(1,1.38) (2,1.16) (3,1.74) (4,1.99) (5,1.76) (6,2.09) (7,1.84) (8,2.53) (9,3.16) (10,4.00) (11,4.29) (12,4.59) (13,4.83) (14,5.42) (15,6.08) (16,6.56)};
\legend{token tables, \stcon{}}
\end{axis}
\end{tikzpicture}
\caption{\textbf{A server that receives sixteen different grammars} (Qwen3.5-4B). With token tables, memory grows by hundreds of megabytes per grammar, the fourteenth grammar waits 26\,s for its first answer, and the fifteenth crashes the server. With \stcon{}, all sixteen fit in about 1\,GB and first answers stay fast.}
\label{fig:wild-main}
\end{figure}

\subsection{Guarantees: valid under any limit, and exact in parallel}

\begin{table}[ht]
\centering
\begin{minipage}[t]{0.48\linewidth}
\centering
\footnotesize
\caption{\textbf{Guarantees} (2B unless stated). Top: valid outputs under token limits. Bottom: model runs for 18 outputs; certified blocks give the same tokens.}
\label{tab:guarantees}
\setlength{\tabcolsep}{3pt}
\begin{adjustbox}{max width=\linewidth}
\begin{tabular}{@{}lrr@{}}
\toprule
valid under a token limit & no rule & our rule \\
\midrule
JSON, every feasible limit & 3/90 & \textbf{90/90} \\
recursive, 61 limits & 2/61 & \textbf{54/61} \\
\midrule
runs for 18 outputs & sequential & ours \\
\midrule
Qwen3.5-2B & 672 & \textbf{374} (exact) \\
Qwen3.5-4B & 686 & \textbf{484} (exact) \\
\bottomrule
\end{tabular}
\end{adjustbox}
\end{minipage}\hfill
\begin{minipage}[t]{0.48\linewidth}
\centering
\footnotesize
\caption{\textbf{One output under shrinking limits} (event schema, 2B): with our rule the output stays valid and shortens gracefully.}
\label{tab:budgetex-main}
\setlength{\tabcolsep}{2pt}
\begin{adjustbox}{max width=\linewidth}
\begin{tabular}{@{}rL{5.4cm}@{}}
\toprule
limit & output (abridged) \\
\midrule
320 & \texttt{\{"title":"The Applied AI Summit", \ldots, "attendees":[\{3 people\}]\}} \\
48 & \texttt{\{\ldots, "attendees":[\{"name":"Dr.", "role":"speaker"\}]\}} \\
32 & \texttt{\{\ldots, "venue":"", "city":"", \ldots, "attendees":[]\}} \\
40, no rule & \texttt{\{\ldots, "attendees":[\{"name":"} (broken) \\
\bottomrule
\end{tabular}
\end{adjustbox}
\end{minipage}
\end{table}

Every application sets a token limit, and an output cut in the middle of a value does not parse. With the budget rule of Section~\ref{sec:decoding}, \textbf{every} JSON output is valid under every limit that admits a valid output at all, 90 of 90, against 3 of 90 without it, and for recursive schemas 54 of 61 against 2 of 61 (Table~\ref{tab:guarantees}). Table~\ref{tab:budgetex-main} shows how: as the limit tightens, the engine closes the attendee list early, then shortens strings, then empties them, instead of stopping mid-value. Certified parallel decoding, which drafts the fields of an object at the same time and then verifies them, reproduces the sequential greedy outputs token for token on all 36 tests with \textbf{$1.80\times$} (2B) and \textbf{$1.42\times$} (4B) fewer model runs. On this laptop GPU extra rows in a batch are not free, so the wall time does not drop yet; the saving in runs is the hardware-independent number and turns into time on accelerators where extra rows are nearly free.

\subsection{The whole engine in real use}

\begin{table}[ht]
\centering
\footnotesize
\caption{\textbf{The whole engine against stock mlx-lm with llguidance} (2B unless stated). These numbers also include our matrix kernels and batching, which are engineering rather than theory.}
\label{tab:engine}
\setlength{\tabcolsep}{4pt}
\begin{tabular}{@{}L{6.8cm}rrr@{}}
\toprule
workload & mlx-lm + llguidance & \stcon{} & gain \\
\midrule
16 coding agents with tool calls, wall time & 74.6\,s & \textbf{29.3\,s} & $\mathbf{2.55\times}$ \\
8 concurrent tool-call requests through the server & 134 tok/s & \textbf{410 tok/s} & $\mathbf{3.1\times}$ \\
18 JSON jobs, 8 slots, schema in the prompt & 6.12\,s & \textbf{4.66\,s} & $\mathbf{1.31\times}$ \\
18 JSON jobs, 8 slots, keys-only prompt for \stcon{} & 6.12\,s & \textbf{3.83\,s} & $\mathbf{1.60\times}$ \\
8 streams, no constraint, 2B / 4B (stock mlx-lm) & 188 / 87 tok/s & \textbf{576 / 271 tok/s} & $\mathbf{3.1\times}$ \\
40 LiveCodeBench problems, 8 streams, 4B (stock mlx-lm) & 1,181\,s & \textbf{476\,s} & $\mathbf{2.5\times}$ \\
\bottomrule
\end{tabular}
\end{table}

In practice the constraint mechanisms run together with our batched matrix kernels, which read the weights once per step for all requests. Table~\ref{tab:engine} measures the complete systems. Sixteen coding agents, each calling tools such as \texttt{write\_file} and \texttt{run\_python} in the model's native format with every reply constrained to a well-formed call, finish their suite in \textbf{29.3\,s instead of 74.6\,s}, and more of them finish within the time limit (12 against 9 of 16). Through our OpenAI-compatible server, eight concurrent tool-using requests are served at three times the throughput, and every reply parses. A batch of eighteen JSON jobs finishes $1.31\times$ sooner with the same prompts, and $1.6\times$ sooner when the prompt lists only the keys and options, because forced text supplies the schema's syntax without the model reading it. Without any constraint, eight parallel streams run three times faster than stock mlx-lm, and forty recent LiveCodeBench programming problems finish in 476\,s instead of 1,181\,s, with a similar pass rate (12 against 13 of 40). The engine thus brings structured, tool-using workloads that are usually reserved for servers within reach of a 16\,GB laptop.

\FloatBarrier
\section{Related work}
\textbf{Few-path complexity.} Unambiguous and few-path logspace computation was introduced by Buntrock, Jenner, Lange and Rossmanith~\cite{bjlr1991}; Allender and Lange~\cite{allender1998} introduced mangroves and the $O(\log^2 n/\log\log n)$ simulation of ReachUL; Lange~\cite{lange1997}, Reinhardt and Allender~\cite{reinhardt2000}, Garvin et al.~\cite{garvin2014} and Pavan et al.~\cite{pavan2012} developed unambiguous logspace further. Apers and Edenhofer~\cite{apers2025} solved few-path reachability in quantum logspace and asked for its classical complexity, which we settle. \textbf{Constrained decoding.} Outlines~\cite{willard2023outlines} precomputes allowed tokens per state; llguidance~\cite{llguidance} computes masks on the CPU with fast-forward tokens; XGrammar~\cite{dong2024xgrammar}, SynCode~\cite{ugare2024syncode} and grammar-constrained decoding~\cite{geng2023} handle context-free grammars; SGLang~\cite{zheng2024sglang} decodes runs of forced transitions in one step; Koo et al.~\cite{koo2024} compose the tokenizer with the automaton. Non-canonical tokenizations hurt accuracy~\cite{beurerkellner2024domino,geh2024,vieira2024}. \textbf{Fewer model runs.} Speculative decoding~\cite{leviathan2023,chen2023} and lookahead decoding~\cite{fu2024lookahead} verify drafts; our forced spans are drafts accepted with certainty, and certified blocks verify drafts that the grammar makes independent.

\section{Discussion}
\label{sec:discussion}
\textbf{Where the gains come from, and where they transfer.} The speedups of Section~\ref{sec:exp} have two sources that are independent of the particular hardware. Skipping model runs for forced text helps whenever a model run is expensive compared with the work of the constraint, which is the case on every device where decoding is limited by memory bandwidth: laptops, phones and edge accelerators, and servers that decode few requests at a time. Keeping the mask on the GPU helps whenever the CPU and the GPU would otherwise take turns, which is the case in any decode loop that waits for a CPU-side mask before the next step. When a server batches many requests so that each step is limited by arithmetic rather than memory, extra tokens in a run stop being free, and the first gain shrinks (in our batch at 16 slots it disappears), while the memory gains remain.

\textbf{What a practitioner gets.} The engine is a drop-in OpenAI-compatible server: any JSON schema, regular expression or tool-call format can be sent with a request and is enforced exactly, without compiling a table, with outputs that always parse, even when the token limit is reached. Prompts can be shorter, since the format's syntax is supplied by forced text. Many grammars and many agents can share one 16\,GB laptop, because a grammar costs megabytes. All of this keeps the model's own answers: on our extraction tasks the facts extracted are the same as with llguidance, and on the 2B the outputs are identical character for character.

\textbf{Why the theory was useful.} Each mechanism could have been tried as a heuristic; the theory is what turns it into a guarantee. Forced spans are not just ``skip easy tokens'': the pass-optimality theorem says the output is unchanged and no exact decoder without speculation can use fewer runs. The table-free mask is not an approximation: one path per token means the recomputed row equals the table exactly. The budget rule is not a timeout: it is a shortest-path invariant that keeps every output valid and is provably the least restrictive such rule. Parallel decoding is not a guess that fields are independent: the grammar's split points license the draft, and the verification makes the result identical to sequential decoding. In each case the proof also told us what to measure.

\textbf{Outlook.} Three directions follow directly. The table-free GPU walk can be extended to the recursive path, so that recursive grammars also cost only megabytes and can be batched. The same mechanisms apply unchanged to other accelerators and larger models, where the memory savings matter even more. And forced spans and certified blocks combine naturally with sampling-based decoding, since both are defined for any sampling rule. On the theoretical side, the remaining question is whether $\STCON$ belongs to StrongFewL, which would make it a natural complete problem for that class.

\textbf{Limitations.} Our measurements use one machine and one model family, six JSON schemas, four recursive schemas and greedy decoding. The recursive path serves one request at a time, and it accepts grammars whose recursive alternatives can be told apart by their first byte, which covers JSON schemas but not every grammar. Certified parallel decoding saves model runs but not yet wall time on this GPU.

\section{Conclusion}
Total few-path reachability, the question of Apers and Edenhofer, is solved with logarithmic memory by an unambiguous machine, by a deterministic stack machine and by a single zero test, and in less memory than Savitch's bound. The same constructions organize structured decoding: what the format forces costs no model run, the mask is recomputed instead of stored, forced text keeps the tokenizer's split, recursion is a small stack, outputs stay valid under any feasible limit, and independent fields run in parallel and are verified. Few paths, in short, mean fast decoding: the structure that makes the theory work is the structure that makes structured output cheap.

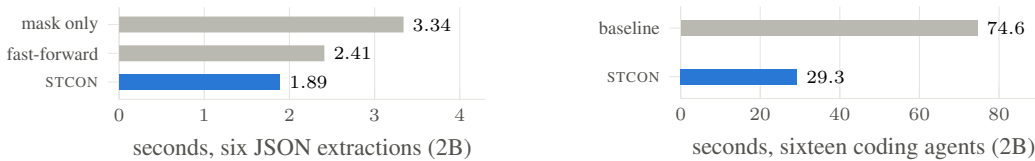
\begin{figure}[!h]
\centering
\begin{tikzpicture}
\begin{axis}[xbar, width=0.46\textwidth, height=2.8cm, bar width=6pt, xmin=0, xmax=4.3, xlabel={seconds, six JSON extractions (2B)},
  symbolic y coords={ours,ff,mask}, ytick={ours,ff,mask}, yticklabels={\stcon{}, fast-forward, mask only},
  nodes near coords, nodes near coords style={font=\scriptsize}, tick label style={font=\scriptsize}, enlarge y limits=0.3, xmajorgrids, ymajorgrids=false]
\addplot[fill=deemph, draw=none, bar shift=0pt] coordinates {(3.34,mask) (2.41,ff)};
\addplot[fill=ours, draw=none, bar shift=0pt] coordinates {(1.89,ours)};
\end{axis}
\end{tikzpicture}\hfill
\begin{tikzpicture}
\begin{axis}[xbar, width=0.46\textwidth, height=2.8cm, bar width=6pt, xmin=0, xmax=92, xlabel={seconds, sixteen coding agents (2B)},
  symbolic y coords={ours,stock}, ytick={ours,stock}, yticklabels={\stcon{}, baseline},
  nodes near coords, nodes near coords style={font=\scriptsize}, tick label style={font=\scriptsize}, enlarge y limits=0.45, xmajorgrids, ymajorgrids=false]
\addplot[fill=deemph, draw=none, bar shift=0pt] coordinates {(74.6,stock)};
\addplot[fill=ours, draw=none, bar shift=0pt] coordinates {(29.3,ours)};
\end{axis}
\end{tikzpicture}
\caption{\textbf{Practical results in two numbers.} Left: one request at a time, same model and kernels, $1.28\times$ faster than llguidance's fast-forward mode and $1.77\times$ faster than its mask-only mode, with the same answers. Right: the whole engine on an agent workload, $2.55\times$ faster than the baseline, stock mlx-lm with llguidance.}
\label{fig:summary}
\end{figure}

\section*{Statement of novelty}
The theoretical results of this paper, the classical placement of total few-path reachability in ReachUL, LOGDCFL, C$_=$L and SC$^2$ with its sub-Savitch space bound, the budget, compiler, completeness and threshold results, and the structure theory of few-path graphs, were obtained independently by the author in August 2026. The remaining time was spent on the practical application: identifying the few-path structure of constrained decoding, deriving the decoding mechanisms and their guarantees from the constructions of the proofs, building the engine on MLX.

\bibliographystyle{plainnat}

\newpage
\appendix
\section*{Appendix}
\addcontentsline{toc}{section}{Appendix}
The appendix is self-contained. Appendix~\ref{app:background} explains the inference terms. Appendix~\ref{app:theory} proves every statement of Section~\ref{sec:theory}, Appendix~\ref{app:decoding} every statement of Section~\ref{sec:decoding}. Appendix~\ref{app:mech} describes the mechanisms and the engine as implemented. Appendices~\ref{app:speed}--\ref{app:quality} hold every measurement, Appendix~\ref{app:neg} the negative results, and Appendix~\ref{app:exact} the tests and the scripts.

\section{Background in plain words, and what goes wrong on a laptop}
\label{app:background}
\textbf{Decoding.} A language model is given a \emph{prompt}, a list of tokens, and produces a score (a logit) for every token of its \emph{vocabulary} (248,320 tokens for the Qwen3.5 models) as the next one. The engine picks one, appends it and asks again. Taking the highest-scoring token is \emph{greedy} decoding; drawing at random from the scores is \emph{sampling}. Processing the prompt is the \emph{prefill}; producing the answer token by token is the \emph{decode}. One decode step is one forward pass, which on a laptop reads all of the model's weights from memory: a \emph{weight stream}. A pass that appends several tokens at once costs little more than a pass that appends one, because the weights are read once either way.

\textbf{Caches.} A transformer keeps, for every token it has seen, a small vector per layer (the \emph{KV cache}), so that it need not re-read the prompt. The Qwen3.5 models are \emph{hybrid}: three of every four layers are recurrent (Gated DeltaNet~\cite{yang2025gdn}) and keep a fixed-size state instead of a per-token cache. On the 4B the KV cache costs 32\,KB per token per stream and the recurrent state 50\,MB per stream.

\textbf{Constraints.} A constraint says which outputs are acceptable: a JSON schema, a regular expression, the grammar of tool calls. It is compiled into an automaton that reads one byte at a time; at each state some bytes are allowed. To enforce it, the engine forbids at every step the tokens whose bytes the automaton would reject (the \emph{mask}) and moves the automaton along the bytes of the chosen token. Three designs exist. Outlines~\cite{willard2023outlines} keeps a table of states $\times$ vocabulary per grammar. XGrammar~\cite{dong2024xgrammar} precomputes, per grammar position, a compressed cache of the tokens whose validity does not depend on the stack and checks the others against a stack on the CPU. llguidance~\cite{llguidance} computes masks lazily on the CPU.

\textbf{Streams and batches.} An answer being decoded, with its prompt, caches and automaton state, is a \emph{stream}. Several streams decoded together form a \emph{batch}; one step serves all of them per weight stream. In an agent system the streams are agents waiting for answers; the engine's \emph{slots} are the largest batch it will run.

\textbf{Where GPU memory goes.} The weights; the prompt cache (recent prefills kept for reuse); the streams' caches, which grow with every token; a transient used while a prompt is prefilled; and the constraints' tables. On a 16\,GB M2 Pro the GPU may wire 10.7\,GB.

\textbf{What goes wrong} as grammars and agents multiply (all measured in Appendices~\ref{app:speed}--\ref{app:memory}):
\begin{itemize}[leftmargin=1.4em,itemsep=1pt,topsep=2pt]
\item \emph{Every grammar costs a table} if masks are tabulated: 77\,MB (104 states) to 1.5\,GB (1,970 states); a typed tool grammar 824\,MB and 1--2 minutes to build, kept for the life of the process.
\item \emph{A server dies as grammars accumulate}: with tables, the fifteenth distinct grammar exhausts the wired limit on the 4B.
\item \emph{Host-side masks cost a step's worth of time} in an MLX decode loop: with masks only, llguidance decodes at 103 tokens/s against 137 free on the 2B, because each step must synchronize with the host (Section~\ref{sec:exp}).
\item \emph{Throughput stalls with more agents} in the stock library (87 tokens/s at 4 and at 8 streams on the 4B), and \emph{long answers do not fit} (stock mlx-lm runs out of memory at six 16k-token streams).
\end{itemize}
The first three are about the control side and are addressed by Sections~\ref{sec:theory}--\ref{sec:decoding}; the last two are about the engine's kernels and memory management (Appendices~\ref{app:speed}--\ref{app:memory}).

\section{The theory: formal statements and proofs}
\label{app:theory}
This appendix first states every theoretical result formally, then proves it.

\subsection*{Formal statements}
\label{app:formaltheory}
\addcontentsline{toc}{subsection}{Formal statements}

\subsection{Walks, few paths, and the total language}
\begin{definition}[Walks]
Let $G=(V,E)$ be a finite directed graph with $n$ vertices. A \emph{walk} of length $i$ from $u$ to $v$ is a sequence of edges whose endpoints concatenate, from $u$ to $v$; vertices may repeat, and the empty walk ($i=0$) is one walk from every vertex to itself. $\Nw_G(u,v)\in\mathbb{N}\cup\{\infty\}$ is the number of walks from $u$ to $v$, and $\Nw^{\le L}_G(u,v)$ the number of those of length at most $L$.
\end{definition}
$\Nw_G(u,v)$ is infinite exactly when a directed cycle is reachable from $u$ and reaches $v$.
\begin{definition}[Few paths]
For $k\ge1$, $G$ is \emph{$k$-strongly few} if $\Nw_G(u,v)\le k$ for every ordered pair of vertices, and \emph{$k$-reach-few from $s$} if $\Nw_G(s,v)\le k$ for every $v$. A 1-strongly-few graph is a \emph{mangrove}~\cite{allender1998}. Since the empty walk is counted, a $k$-strongly-few graph is acyclic.
\end{definition}
The language studied is
\[
\STCON=\bigl\{\langle G,s,t,1^k\rangle : \Nw_G(u,v)\le k\ \text{for all } u,v\in V,\ \text{and}\ \Nw_G(s,t)\ge 1\bigr\},
\]
with $G$ given by its adjacency matrix and $k$ in unary; $M$ is the input length, so $n\le M$ and $nk\le M^2$. More generally, with numerical budgets $\kappa$ of polynomial total size, \textsc{RootBudget} asks for $\Nw_H(r,q)\le\kappa(q)$ for all $q$ and $\Nw_H(r,f)\ge1$ for some $f\in F$, and \textsc{PairBudget} for $\Nw_G(u,v)\le\kappa(u,v)$ for all pairs and $\Nw_G(s,t)\ge1$.

\textbf{The totality obstacle.} The bound on the number of walks is part of what must be \emph{decided}, not a promise handed to the algorithm. A promise algorithm may enumerate paths because fewness guarantees there are few of them. A total algorithm receives an arbitrary graph, possibly with cycles or exponentially many walks, and must reject it within the same space, time and ambiguity bounds it uses on good inputs. Each classical tool fails on bad inputs in its own way: a depth-first traversal of a cyclic graph never ends, a fingerprint hash of exponentially many walks collides and miscounts, and an unambiguous machine that guesses a path through an ambiguous graph ceases to be unambiguous. Fig.~\ref{fig:graph} shows the two shapes that recur below: two walks between a pair, and a \emph{forced chain}.

\subsection{Main results}

\begin{theorem}[Main theorem]
\label{thm:main}
For unary $k$,
\[
\STCON\in\mathrm{ReachUL}=\mathrm{ReachFewL},\qquad \STCON\in\mathrm{LOGDCFL}\cap\mathrm{C}_=\mathrm{L}\cap\mathrm{SC}^2,
\]
\[
\STCON\in\mathrm{DSPACE}\bigl(O(\log^2 M/\log\log M)\bigr).
\]
$\STCON$ is StrongFewL-hard under deterministic logspace many-one reductions.
\end{theorem}

\textbf{What the theorem settles.} Apers and Edenhofer~\cite{apers2025} placed this language in quantum logspace and asked whether the classical bounds known for few-path computation, SC$^2$ and the space bound $O(\log^2M/\log\log M)$, hold for the total language. Theorem~\ref{thm:main} answers yes, and proves more: the total language is in ReachUL itself. A machine with at most one computation path to every configuration decides, on every input, whether the graph has few walks; it rejects cycles and exponentially many walks without ever becoming ambiguous. The result is a deterministic algorithm in space $O(\log^2M/\log\log M)$, below Savitch's $O(\log^2M)$, for a reachability problem that certifies its own few-path condition, and a single deterministic algorithm with polynomial time and $O(\log^2M)$ space. Three constructions prove it (Table~\ref{tab:targets}).

\begin{table}[t]
\centering
\small
\caption{\textbf{The classical placement of total few-path reachability.} Every membership is proved on arbitrary inputs; the proofs are in Appendix~\ref{app:theory}.}
\label{tab:targets}
\footnotesize
\begin{tabular}{@{}lL{2.6cm}L{6.4cm}@{}}
\toprule
Target & Result & Construction \\
\midrule
ReachUL $=$ ReachFewL & membership & certified finite-field lift (Theorem~\ref{thm:reachul}) \\
LOGDCFL & membership & capped deterministic auxiliary pushdown (Theorem~\ref{thm:logdcfl}) \\
SC$^2$ & membership & Cook's simulation of the same pushdown \\
DSPACE$(\frac{\log^2 M}{\log\log M})$ & membership & the ReachUL machine, simulated as in~\cite{allender1998} \\
C$_=$L & membership & one GapL zero test of squared vanishing products (Theorem~\ref{thm:ceql}) \\
NL & membership & collision-averaged fingerprints (Theorem~\ref{thm:phase}) \\
StrongFewL & hardness; membership open & configuration graphs (Theorem~\ref{thm:boundary}) \\
\bottomrule
\end{tabular}
\end{table}

The constructions are not specific to the uniform bound $k$. They validate any polynomially bounded system of walk budgets, and they compile any logspace computation whose ambiguity is bounded only \emph{semantically}.

\begin{theorem}[Total budget validation]
\label{thm:budgets}
If the total numerical budget is polynomially bounded in the input length, then
\[
\textsc{RootBudget},\ \textsc{PairBudget}\ \in\ \mathrm{ReachUL}\cap\mathrm{LOGDCFL}\cap\mathrm{C}_=\mathrm{L}\cap\mathrm{SC}^2 .
\]
Theorem~\ref{thm:main} is the special case $\kappa(u,v)=k$.
\end{theorem}

\begin{theorem}[Fail-closed compiler]
\label{thm:compiler}
Let $M_0$ be a polynomial-time nondeterministic logspace machine, $\pi_x(c)$ the number of its computation paths on input $x$ from the initial configuration to configuration $c$, and $\kappa_x(c)\ge0$ logspace computable with a fixed polynomial bound on $\sum_c\kappa_x(c)$. Then
\[
L_{M_0,\kappa}=\{x:\ \pi_x(c)\le\kappa_x(c)\ \text{for every } c,\ \text{and } M_0 \text{ accepts } x\}
\]
is in $\mathrm{ReachUL}\cap\mathrm{LOGDCFL}\cap\mathrm{C}_=\mathrm{L}\cap\mathrm{SC}^2$.
\end{theorem}
In words: a nondeterministic logspace computation that is \emph{expected} to have few computation paths, with no promise and no restriction on the machine, can be replaced by a reach-unambiguous one that checks the expectation on every input and rejects where it fails.

\begin{theorem}[Completeness]
\label{thm:complete}
For every fixed $k\ge1$, the rooted language $\textsc{RootFew}_k=\{\langle G,s,t\rangle:\Nw_G(s,v)\le k\ \forall v,\ \Nw_G(s,t)\ge1\}$ is ReachUL-complete under logspace reductions, and remains so on DAGs; rooted $k$-few recognition is ReachUL-complete; and the two-cone language $\textsc{FBFew}_{k,1}$, which bounds walks out of $s$ by $k$ and into $t$ by $1$, is ReachUL-complete.
\end{theorem}

\begin{theorem}[A phase transition in the threshold encoding]
\label{thm:phase}
\begin{enumerate}[label=(\roman*),leftmargin=1.6em,itemsep=1pt,topsep=2pt]
\item With polynomially bounded thresholds, $\textsc{RootBudget},\textsc{PairBudget}\in\mathrm{NL}$, and an $\mathrm{FL}^{\mathrm{NL}}$ transducer outputs the capped counts $\min\{\Nw^{\le B}(r,v),K+1\}$. More generally, for every \#L function $f$ and polynomially bounded logspace threshold $\tau$, $\{x:f(x)\ge\tau(x)\}\in\mathrm{NL}$; hence the class of such polynomial-threshold \#L languages equals NL.
\item With the threshold $K$ in binary, $\textsc{Bin-}\STCON$ is PL-complete, even on DAGs on which $s$ reaches $t$ by construction.
\item For a single pair and a fixed cap $k$, $\{\langle G,s,t\rangle: G \text{ a DAG},\ \Nw_G(s,t)\le k\}$ is NL-complete.
\end{enumerate}
\end{theorem}
The complexity of validating few paths thus moves from inside $\mathrm{ReachUL}\cap\mathrm{LOGDCFL}$ to PL-complete exactly when the numerical threshold may become exponential, not when its notation changes (Table~\ref{tab:phase}).

\begin{table}[t]
\centering
\small
\caption{\textbf{Validation tasks and their classification.}}
\label{tab:phase}
\footnotesize
\begin{tabular}{@{}L{5.6cm}L{7.4cm}@{}}
\toprule
Validation task & Classification \\
\midrule
one pair, fixed cap & NL-complete \\
one rooted row, fixed $k$ & ReachUL-complete (also on DAGs) \\
one rooted row or all pairs, polynomial total budget & in $\mathrm{ReachUL}\cap\mathrm{LOGDCFL}\cap\mathrm{C}_=\mathrm{L}\cap\mathrm{SC}^2$ \\
all pairs, unary $k$ ($\STCON$) & in $\mathrm{ReachUL}\cap\mathrm{LOGDCFL}\cap\mathrm{C}_=\mathrm{L}\cap\mathrm{SC}^2$; StrongFewL-hard \\
all pairs, binary threshold & PL-complete \\
\bottomrule
\end{tabular}
\end{table}

\begin{theorem}[The StrongFewL boundary]
\label{thm:boundary}
$\STCON$ is StrongFewL-hard. A proof that $\STCON\in\mathrm{StrongFewL}$ would make it a natural complete problem for StrongFewL.
\end{theorem}
The reach-unambiguous machine of Theorem~\ref{thm:main} controls the paths from its \emph{initial} configuration; StrongFewL asks for polynomially many paths between \emph{every} pair of configurations, including pairs unreachable from the start. That extra invariant is exactly what membership in StrongFewL requires, and it is the precise open question left by this work (Appendix~\ref{app:hardness}).

\subsection{How the proofs go}
\label{sec:proofs}
All three constructions bound a search by a budget, and one elementary lemma says why that is allowed.

\begin{lemma}[Budget truncation]
\label{lem:trunc}
Let $H$ be a finite directed graph with a root $r$, give every vertex a budget $\kappa(v)\in\mathbb{N}$, and let $B=\sum_v\kappa(v)$. Then $\Nw_H(r,v)\le\kappa(v)$ for all $v$ if and only if $\Nw^{\le B}_H(r,v)\le\kappa(v)$ for all $v$.
\end{lemma}
\begin{proof}
Left to right is immediate. Conversely, suppose the bounded counts respect the budgets and some walk from $r$ has length at least $B$. Its prefixes of lengths $0,\dots,B$ are $B+1$ distinct walks from $r$ of length at most $B$; since the budgets sum to $B$, some vertex $v$ ends more than $\kappa(v)$ of them, a contradiction. So every walk from $r$ is shorter than $B$, and $\Nw_H=\Nw^{\le B}_H$.
\end{proof}
Explore walks only up to the total budget: either you have seen every walk, or you have already found a violation. The depth $B$ is tight (a directed cycle on $B$ vertices with unit budgets). All-pairs budgets reduce to one root by a \emph{universal source with per-source copies}: the graph $U(G)$ has a root $\rho$, a vertex $(a,v)$ for every pair, an edge $\rho\to(a,a)$ for every $a$, and an edge $(a,u)\to(a,v)$ for every edge $u\to v$ and every tag $a$. Every walk from $\rho$ to $(a,v)$ enters copy $a$ and stays there, so $\Nw_{U(G)}(\rho,(a,v))=\Nw_G(a,v)$ \emph{for every pair}, and the budget of $(a,v)$ can be set to $\kappa(a,v)$ (Lemma~\ref{lem:universal}).

\textbf{Construction 1: a capped deterministic pushdown (LOGDCFL, SC$^2$).} Buntrock et al.~\cite{bjlr1991} observed that a depth-first search of a reach-few machine, keeping a stack of return points, runs in LOGDCFL; on an arbitrary graph that search may never stop. Capping it by the total budget makes it total.
\begin{theorem}[Capped stack traversal]
\label{thm:logdcfl}
Run, from every source $a$, a depth-first traversal of the tree of walks from $a$ that keeps a stack of return points (a vertex and the index of its next successor) and aborts as soon as more than $nk$ nodes have been entered; reject on any abort; then recount, for every pair $(a,b)$, the walks from $a$ to $b$, rejecting when a count exceeds $k$; accept iff the $(s,t)$ count is positive. This decider accepts exactly $\STCON$, and it is a deterministic auxiliary pushdown automaton with an $O(\log M)$ work tape running in polynomial time; hence $\STCON\in\mathrm{LOGDCFL}\subseteq\mathrm{SC}^2$~\cite{sudborough1978,cook1979}.
\end{theorem}
\begin{proof}[Proof idea]
In a $k$-strongly-few graph every source has at most $nk$ walks, so every traversal completes and, the graph being acyclic, its counts are exact. A source on a directed cycle has a walk of every length, so its traversal enters $nk+1$ nodes and aborts. The order of passes matters: recounts start only after every source has bounded its whole tree, so a bad input cannot force an exponentially long recount (Appendix~\ref{app:logdcfl}).
\end{proof}

\textbf{Construction 2: a certified lift (ReachUL).} The idea is to certify before counting. Hash every walk from the root by a polynomial fingerprint of its edge indices, evaluated at a point $\alpha$ of a prime field, into a layered graph $H_\alpha$ whose states record the endpoint, the length and the hash; walks correspond one-to-one to root paths of $H_\alpha$. If $H_\alpha$ is reach-unambiguous, every walk reaches its own state and the number of reachable states in the fiber of $v$ is exactly $\Nw^{\le B}(r,v)$, \emph{whether or not the input is valid}. Reach-unambiguity of a rooted graph is decided by a total ReachUL recognizer~\cite{lange1997}, and ReachUL is closed under logspace oracle computations~\cite{garvin2014}.
\begin{theorem}[Certified lift]
\label{thm:reachul}
The controller that, for $\alpha=0,1,\dots$, certifies $H_\alpha$ and, for the first certified $\alpha$, compares every fiber count with its budget, decides \textsc{RootBudget} by a reach-unambiguous logspace computation. With the universal source, \textsc{PairBudget} and $\STCON$ are in ReachUL.
\end{theorem}
On a valid input a separating $\alpha$ exists, because at most $(B-1)\sum_v\binom{\kappa(v)}{2}<p$ evaluation points cause a collision. On an invalid input either no lift is certified, or a certified lift exposes an over-budget fiber, by Lemma~\ref{lem:trunc}. Certification is followed by a budget comparison, not replaced by it: in a diamond, a separating lift exists whatever the budget at the merge vertex, and only the exact fiber count of two can reject a budget of one. Appendix~\ref{app:reachul} gives the proof, and Appendix~\ref{app:direct} a second, self-contained certification that uses no oracle.

\textbf{Construction 3: one zero test (C$_=$L).} The bounded counts $f_{uv}=\Nw^{\le nk}(u,v)$ are \#L functions. The products $\prod_{j=0}^{k}(f_{uv}-j)$ vanish exactly when a count is within budget, and the designated pair uses $\prod_{j=1}^{k}(f_{st}-j)$. The sum of their squares is a GapL function that is zero exactly on $\STCON$.
\begin{theorem}[Signed counting]
\label{thm:ceql}
$\STCON\in\mathrm{C}_=\mathrm{L}$, and \textsc{PairBudget}, \textsc{RootBudget} $\in\mathrm{C}_=\mathrm{L}$ for budgets of polynomial total size.
\end{theorem}
C$_=$L also contains ReachUL (ReachUL $\subseteq$ UL $\subseteq$ SPL $\subseteq$ C$_=$L), so the separate value of Theorem~\ref{thm:ceql} is its certificate: one nonnegative integer whose vanishing enforces every budget at once, although C$_=$L is not known to be closed under complement.

\subsection{The shape of few-path graphs}
\label{sec:structure}
Restrict a graph to the vertices and edges on some $s$--$t$ path: the result is a \emph{trimmed two-terminal DAG} $H=(V,E;s,t)$ with $r=\Nw_H(s,t)$ paths and cyclomatic number $c(H)=|E|-|V|+1$. A \emph{branch vertex} has out-degree at least two. The theorems below say what a few-path graph looks like; the decoding mechanisms of Section~\ref{sec:decoding} read several of them off the graph of a decoding step.

\begin{theorem}[Branch kernel]
\label{thm:kernel}
$c(H)+1\le r\le 2^{c(H)}$, and $H$ has at most $c(H)\le r-1$ branch vertices. Suppressing every internal vertex with in- and out-degree one leaves a multigraph with the same $r$ terminal paths, at most $2c(H)+2\le 2r$ vertices and at most $3c(H)+1$ edges. Hence between any two vertices of a $k$-strongly-few graph, the vertices on paths between them form a subdivision of a multigraph with at most $2k$ vertices and $3k-2$ edges.
\end{theorem}
\begin{theorem}[Block product]
\label{thm:block}
Let $B_1,\dots,B_q$ be the biconnected blocks on the route from $s$ to $t$ in the block--cut tree of the underlying undirected graph of $H$, with entry and exit vertices $a_i,b_i$. Then $r=\prod_i\Nw_{B_i}(a_i,b_i)$ and $c(H)=\sum_ic(B_i)$; every $a_i,b_i$ lies on every $s$--$t$ path, and a path is an independent choice of one $a_i$--$b_i$ path per block. If $r\le k$, at most $\lfloor\log_2k\rfloor$ blocks have more than one path.
\end{theorem}
\begin{theorem}[Finite forbidden-subdivision basis]
\label{thm:basis}
For every fixed $k\ge1$ there is a finite family $\mathcal{F}_k$ of acyclic two-terminal multigraphs such that a finite directed graph is $k$-strongly few if and only if it has no directed cycle and contains no directed subdivision of a member of $\mathcal{F}_k$. Every member has between $k+1$ and $2k$ terminal paths, at most $2k$ vertices and at most $3k-1$ edges, and $|\mathcal{F}_k|\le2^{O(k\log k)}$.
\end{theorem}
\begin{theorem}[Bounded fixed point and mangrove normal form]
\label{thm:normal}
\begin{enumerate}[label=(\roman*),leftmargin=1.6em,itemsep=1pt,topsep=2pt]
\item $G$ is $k$-strongly few iff the equation $C=I+AC$ has a solution $C\in\{0,\dots,k\}^{n\times n}$ ($A$ the adjacency matrix); the solution is then unique and equals the path-count matrix $(I-A)^{-1}$.
\item The path-count matrix of every $k$-strongly-few graph, $k\ge2$, factors over the nonnegative integers as $C=M_1M_2\cdots M_r$ with fewer than $4\log_2k$ factors, each the path-count matrix of a mangrove, by a recursion of depth at most $1+\lfloor\log_2\log_2k\rfloor$. For path-convex mangrove decompositions this is optimal: a chain of $r$ diamonds needs at least $r+1$ blocks and depth at least $\lceil\log_2(r+1)\rceil$.
\end{enumerate}
\end{theorem}
\begin{theorem}[Density and width]
\label{thm:density}
A simple $k$-strongly-few graph on $n$ vertices has at most $\min\{\binom n2,\ n^2/4+3(k-1)n/4\}$ edges, and for every fixed $k\ge2$ the maximum is $n^2/4+\Theta_k(n)$. Every reachability interval $I(u,v)=\{x:u\leadsto x\leadsto v\}$ has width at most $\Nw_G(u,v)\le k$ and is the union of at most $k$ directed $u$--$v$ paths.
\end{theorem}
\begin{theorem}[Tractable regimes]
\label{thm:regimes}
Let $\lambda(G)$ be the smaller of the numbers of sources and of sinks, and $b(G)$ the largest cyclomatic number of a biconnected block of the underlying undirected graph. Total strongly-few reachability is decidable in polynomial time and $O((k\lambda(G)+1)\log M)$ space, and in $2^{O(b(G))}M^{O(1)}$ time and $O(b(G)+\log M)$ space. In particular it is in L for fixed $k$ and a fixed number of sources or sinks, in $\mathrm{SC}^d$ when $k\lambda(G)=O(\log^{d-1}M)$, and in L when $b(G)=O(\log M)$, as for chains of diamonds and oriented cacti.
\end{theorem}
All proofs are in Appendix~\ref{app:structure}. Read together, these results say that between two vertices of a few-path graph almost everything is forced, the rest factors into independent blocks, every violation contains one of finitely many small witnesses, and the whole path structure is a short product of mangroves.


\subsection*{Proofs}


\subsection{Machines, classes and the published interfaces}
A nondeterministic logspace machine $M$ has a read-only input tape and a work tape of $O(\log|x|)$ cells. A \emph{configuration} consists of the state, the input head position and the work tape with its head; $G_M(x)$ is the configuration graph on input $x$, with start configuration $C_0$. As usual we assume a clock, so that $M$ halts on every path in polynomial time and $G_M(x)$ is acyclic, and a unique accepting configuration. A computation path halts by entering a configuration without successors; halted paths are not merged into a common rejecting configuration, so a failed guess ends where it is.
\begin{itemize}[leftmargin=1.4em,itemsep=1pt,topsep=2pt]
\item $M$ is \emph{reach-unambiguous} if on every input every configuration has at most one path from $C_0$. ReachUL is the class of languages accepted by such machines; ReachFewL allows polynomially many paths to every configuration, and ReachFewL $=$ ReachUL~\cite{garvin2014}.
\item UL: at most one accepting path on every input. StrongFewL: for every input and every ordered pair of configurations, including pairs unreachable from $C_0$, at most polynomially many paths between them~\cite{bjlr1991}.
\item LOGDCFL: the languages logspace-reducible to deterministic context-free languages; equivalently, those accepted by deterministic auxiliary pushdown automata with a logarithmic work tape in polynomial time~\cite{sudborough1978}. SC$^2$: one deterministic machine within polynomial time and $O(\log^2|x|)$ space simultaneously; LOGDCFL $\subseteq$ SC$^2$~\cite{cook1979}.
\item GapL: differences of two \#L functions; closed under logspace-uniform polynomial sums and products. C$_=$L: $\{x: g(x)=0\}$ for some $g\in$ GapL. PL: $\{x:g(x)>0\}$. SPL: languages whose characteristic function is in GapL~\cite{arz1999}.
\end{itemize}
Two published facts about ReachUL are used. (a)~The total language $\textsc{STCON}_{\mathrm{ru}}=\{\langle J,r,t\rangle:\Nw_J(r,x)\le1\ \forall x,\ \Nw_J(r,t)=1\}$ is in ReachUL (and complete for it)~\cite{lange1997}. (b)~ReachUL is closed under complement and under deterministic logspace oracle computations with logspace-generated queries, $\mathrm{L}^{\mathrm{ReachUL}}=\mathrm{ReachUL}$~\cite{garvin2014}. The composition behind (b) guesses each oracle answer and runs a reach-unambiguous recognizer for the oracle language or its complement; only the branch with the correct answer returns, every call carries its own tags (the answer, the saved controller configuration and a call index), and rejected branches keep their data rather than merge, so the composed machine reaches every configuration by at most one path. Appendix~\ref{app:direct} gives a second proof of the ReachUL membership that uses neither fact.

\subsection{The universal source and total root budgets}
\label{app:budgets}
\begin{lemma}[Universal source with per-source copies]
\label{lem:universal}
Let $U(G)$ have vertices $\{\rho\}\cup\{(a,v):a,v\in V\}$, an edge $\rho\to(a,a)$ for every $a$, and an edge $(a,u)\to(a,v)$ for every edge $u\to v$ of $G$ and every $a$. Then $\Nw_{U(G)}(\rho,(a,v))=\Nw_G(a,v)$ for all $a,v$.
\end{lemma}
\begin{proof}
A walk from $\rho$ to $(a,v)$ begins with the unique edge $\rho\to(a,a)$ and then stays in copy $a$, whose first coordinate never changes. Deleting the first edge is a bijection onto the walks from $a$ to $v$ in $G$; the edge $\rho\to(a,a)$ alone represents the empty walk at $a$.
\end{proof}
So \textsc{PairBudget} with budgets $\kappa(u,v)$ is \textsc{RootBudget} on $U(G)$ with root $\rho$, budget $1$ at $\rho$, budget $\kappa(a,v)$ at $(a,v)$ and target $(s,t)$; the total budget is $1+\sum_{u,v}\kappa(u,v)$, which is $1+n^2k$ for $\STCON$. (A universal source \emph{without} copies would bound only the sums $\sum_a\Nw(a,v)$; the copies are what preserve every pair.)

\begin{theorem}[Root budgets in LOGDCFL]
\label{thm:rootlogdcfl}
If the states, transitions and budgets are logspace enumerable and the total budget $B$ is polynomially bounded, \textsc{RootBudget} $\in$ LOGDCFL.
\end{theorem}
\begin{proof}
Reject if $B=0$. Otherwise traverse the complete, untruncated walk-unfolding from $r$ depth first, with the suspended frames (a state and the index of its next transition) on the pushdown, and reject on entering node $B+1$. If the unfolding has more than $B$ nodes, node $B+1$ is reached after polynomially many entries, even along an infinite branch generated by a cycle. If the traversal completes, the whole unfolding is finite and has at most $B$ nodes. Then rerun the traversal once per state $q$, counting nodes ending at $q$ and rejecting at $\kappa(q)+1$; these recounts are exact because the first pass exhausted the unfolding. Accept iff some accepting state is reached. Every traversal enters at most $B+1$ nodes, the work tape holds a constant number of $O(\log M)$-bit fields, and the procedure is a polynomial-time deterministic auxiliary pushdown automaton. Its invariant is: either the complete unfolding has at most $B$ nodes, or node $B+1$ is reached and the input is rejected.
\end{proof}
With Lemma~\ref{lem:universal}, \textsc{PairBudget} $\in$ LOGDCFL $\subseteq$ SC$^2$. Theorem~\ref{thm:logdcfl} is the uniform special case, written per source.

\subsection{Proof of Theorem~\ref{thm:reachul} (the certified lift)}
\label{app:reachul}
Fix a \textsc{RootBudget} instance $(H,r,\kappa,F)$ with edge occurrences $e_1,\dots,e_m$ and total budget $B\ge1$ (reject $B=0$ or $\kappa(r)=0$ directly). Let $S=\sum_v\binom{\kappa(v)}2$ and let $p$ be the least prime with $p>\max\{m,BS,2\}$, found by trial division; $p$ is polynomial and depends only on the supplied budgets, not on their validity. For $\alpha\in\Fp$ and a walk $P=e_{i_1}\cdots e_{i_d}$ put $h_\alpha(P)=i_1\alpha^{d-1}+\cdots+i_d\bmod p$, computed by $z\mapsto\alpha z+i$.
\begin{definition}[The lift]
$H_\alpha$ has states $(v,d,z)$ with $0\le d\le B$ and $z\in\Fp$, an edge $(u,d,z)\to(v,d+1,\alpha z+i\bmod p)$ for every edge occurrence $e_i=(u,v)$ and $d<B$, and root $(r,0,0)$.
\end{definition}
\begin{lemma}[Separation]
\label{lem:sep}
If all budgets hold, some $\alpha\in\Fp$ separates every two distinct walks of the same length and endpoint.
\end{lemma}
\begin{proof}
For two distinct walks of the same length $d\le B$, the difference of their formal fingerprints is a polynomial of degree at most $B-1$; it is nonzero because some edge-index coefficient differs and indices are below $p$. So each pair collides at most at $B-1$ points. With the budgets holding, at most $\kappa(v)$ walks end at $v$, so there are at most $S$ relevant pairs and at most $(B-1)S<p$ bad points.
\end{proof}
\begin{lemma}[Correspondence and exact fibers]
\label{lem:fibers}
Walks of length $d\le B$ from $r$ to $v$ correspond bijectively to paths from $(r,0,0)$ to $(v,d,h_\alpha(P))$. Whenever $H_\alpha$ is reach-unambiguous from its root, \emph{whether or not the budgets hold},
\[
\Nw^{\le B}_H(r,v)=\bigl|\{(d,z):(r,0,0)\leadsto(v,d,z)\text{ in }H_\alpha\}\bigr|.
\]
\end{lemma}
\begin{proof}
Along a lifted path the residue is forced by the recurrence, so every walk lifts to exactly one path, and projection inverts lifting. Two distinct walks reaching the same lifted state would give two root paths to it, contradicting reach-unambiguity.
\end{proof}
\begin{proof}[Proof of Theorem~\ref{thm:reachul}]
The controller enumerates $\alpha\in\Fp$. For each $\alpha$ it queries $\textsc{STCON}_{\mathrm{ru}}(H_\alpha,(r,0,0),(r,0,0))$, which holds exactly when $H_\alpha$ is reach-unambiguous from its root (the only path from the root to itself is the empty one). For the first certified $\alpha$, it scans every endpoint fiber, querying reachability of each lifted state (an ordinary reachability query once the lift is certified) and keeping a counter saturated at $\kappa(v)+1$; it rejects if a count exceeds its budget, and otherwise accepts iff some fiber of a vertex in $F$ is nonempty. If no $\alpha$ is certified, it rejects.
\emph{Correctness.} On a valid input, Lemma~\ref{lem:sep} gives a certified $\alpha$, every fiber count equals $\Nw^{\le B}(r,v)\le\kappa(v)$ (Lemma~\ref{lem:fibers}), and the reachability condition is read off correctly. On an invalid input, either no lift is certified and the controller rejects, or a certified lift gives exact bounded counts (Lemma~\ref{lem:fibers}), and by Lemma~\ref{lem:trunc} some bounded count exceeds its budget, or the reachability condition fails. (The diamond $r\to a\to v$, $r\to b\to v$ shows why the budget comparison is needed: separating lifts exist whatever the budget at $v$, and only the exact fiber count two rejects a budget of one.)
\emph{Resources.} Each lift has $|Q|(B+1)p$ states and a logspace adjacency predicate; the controller stores $\alpha$, an endpoint, a layer, a residue and a saturated counter, and makes polynomially many queries. It is a deterministic logspace oracle machine with oracle $\textsc{STCON}_{\mathrm{ru}}\in$ ReachUL, so by fact (b) the language is in ReachUL. With Lemma~\ref{lem:universal}, \textsc{PairBudget} and $\STCON$ are in ReachUL.
\end{proof}

\subsection{A second, oracle-free proof of the ReachUL membership}
\label{app:direct}
The certification of the lift can also be done directly, by a reach-unambiguous machine that never guesses through a layer it has not certified. We give this proof for $\STCON$, hashing vertex sequences rather than edge indices.

Let the input be $\langle G,s,t,1^k\rangle$ with $V=\{0,\dots,n-1\}$. Set $L=nk$ and let $p$ be the least prime with $p>\max(n,\,L(nk)^2)$; by Bertrand's postulate $p<2\max(n,L(nk)^2)$, and $p$ is found in logspace by trial division. All numbers below have $O(\log M)$ bits. For a walk $\omega=(v_0,\dots,v_i)$ write $H_\alpha(\omega)=\sum_{j=0}^{i}v_j\alpha^{i-j}\bmod p$ for $\alpha\in\Fp$, computed by Horner's rule $h_0=v_0$, $h_j=\alpha h_{j-1}+v_j$. Write $\Nw_i(a,v)$ for the number of walks of length exactly $i$ from $a$ to $v$.

\begin{definition}[Hashed unfolding]
For $a\in V$ and $\alpha\in\Fp$, $\Gamma_{a,\alpha}$ is the graph on states $(v,i,h)\in V\times\{0,\dots,L\}\times\Fp$ with an edge $(u,i,h)\to(v,i+1,\alpha h+v)$ for every $(u,v)\in E$ and $i<L$, and root $r_a=(a,0,a)$. Layer $i$ is the set of states with second coordinate $i$. $\Gamma_{a,\alpha}$ is \emph{certified through layer $d$} if every state of layers $0,\dots,d$ has at most one path from $r_a$.
\end{definition}
The adjacency of $\Gamma_{a,\alpha}$ is computable in logspace from the input, $a$ and $\alpha$.

\begin{lemma}[Walks are root paths]
\label{lem:b1}
The map $\omega\mapsto\bigl((v_0,0,h_0),\dots,(v_i,i,h_i)\bigr)$ is a bijection from the walks of length $i\le L$ starting at $a$ onto the paths of length $i$ starting at $r_a$; the image of $\omega$ ends at $(v_i,i,H_\alpha(\omega))$. Hence the number of paths from $r_a$ to $(v,i,h)$ equals the number of walks of length $i$ from $a$ to $v$ with $H_\alpha(\omega)=h$.
\end{lemma}
\begin{proof}
Edges of $\Gamma_{a,\alpha}$ project to edges of $G$, and along a path from $r_a$ the third coordinate is determined by the first coordinates through Horner's rule; conversely every walk lifts to exactly one path.
\end{proof}

\begin{lemma}
\label{lem:b2}
$\Gamma_{a,\alpha}$ is certified through layer $d$ if and only if $H_\alpha$ is injective on the walks of length $i$ from $a$ to $v$, for all $i\le d$ and $v\in V$. In that case $\Nw_i(a,v)=|\{h:(v,i,h)\text{ is reachable from } r_a\}|$ for every $i\le d$ and $v$.
\end{lemma}
\begin{proof}
Immediate from Lemma~\ref{lem:b1}.
\end{proof}

\begin{lemma}[A good evaluation point exists]
\label{lem:b3}
If $\sum_v\Nw^{\le L}(a,v)\le nk$, then fewer than $p$ values of $\alpha$ fail to certify $\Gamma_{a,\alpha}$ through layer $L$; in particular some $\alpha\in\Fp$ certifies it.
\end{lemma}
\begin{proof}
Let $\omega\ne\omega'$ be walks of the same length $i\le L$ from $a$ to the same $v$. Their coefficient sequences are their vertex sequences, which differ and have entries below $n<p$, so $P_\omega(X)-P_{\omega'}(X)$, with $P_\omega(X)=\sum_jv_jX^{i-j}$, is a nonzero polynomial over $\Fp$ of degree at most $L$ and vanishes at most at $L$ points. There are at most $(\sum_{i,v}\Nw_i(a,v))^2\le(nk)^2$ such pairs, so at most $L(nk)^2<p$ values of $\alpha$ make some pair collide. By Lemma~\ref{lem:b2} every other value certifies $\Gamma_{a,\alpha}$ through layer $L$.
\end{proof}

\begin{algorithm}[H]
\caption{$\textsc{Tour}(a,\alpha,d)$: Euler tour of the reachable states of layers $\le d$ of $\Gamma_{a,\alpha}$, assumed certified through layer $d$, calling $\textsc{Visit}$ once per state.}
\label{alg:tour}
\small
\begin{algorithmic}[1]
\State $x\gets r_a$;\ \textsc{Visit}$(x)$;\ $\mathit{mode}\gets$ down
\Loop
  \If{$\mathit{mode}=$ down}
    \If{$\mathrm{layer}(x)<d$ \textbf{and} $x$ has an out-neighbour}
      \State $x\gets$ the first out-neighbour of $x$;\ \textsc{Visit}$(x)$;\ \textbf{continue}
    \Else\ $\mathit{mode}\gets$ up
    \EndIf
  \EndIf
  \If{$x=r_a$} \Return \EndIf
  \State $u\gets\textsc{Parent}(x)$ \Comment{nondeterministic, below}
  \If{$u$ has an out-neighbour after $x$}
    \State $x\gets$ that out-neighbour;\ \textsc{Visit}$(x)$;\ $\mathit{mode}\gets$ down
  \Else\ $x\gets u$ \Comment{mode stays up}
  \EndIf
\EndLoop
\Statex
\Statex \textsc{Parent}$(x)$, for $x$ at layer $j\ge1$: guess a state $u$ of layer $j-1$ with an edge $u\to x$; guess $z_0=r_a,z_1,\dots,z_{j-1}$, each $z_{m+1}$ an out-neighbour of $z_m$, storing only $z_m$ and $m$; if $z_{j-1}=u$ return $u$, otherwise halt (this computation path dies).
\end{algorithmic}
\end{algorithm}

\begin{lemma}[Traversal of a certified prefix]
\label{lem:b4}
Suppose $\Gamma=\Gamma_{a,\alpha}$ is certified through layer $d$, and let $T$ be the set of states of layers $\le d$ reachable from $r_a$. Then (i) $T$ with the edges of $\Gamma$ inside it is an out-tree rooted at $r_a$, in which the children of every $x\in T$ of layer $<d$ are exactly its out-neighbours in $\Gamma$; (ii) Algorithm~\ref{alg:tour}, started in a configuration reached by a unique path, visits every state of $T$ exactly once, has exactly one computation path that returns, reaches every configuration on its way by exactly one path from its start, and uses $O(\log M)$ work space in addition to that of \textsc{Visit}.
\end{lemma}
\begin{proof}
(i) Every $x\in T$ has exactly one path from $r_a$; its predecessor on that path is its only in-neighbour in $T$, because an in-neighbour $u\in T$ would give the path $r_a\leadsto u\to x$. Every out-neighbour $y$ of $x\in T$ with $\mathrm{layer}(x)<d$ is reachable, hence in $T$, and its unique path passes through $x$.

(ii) By (i) the tour is the deterministic Euler tour of a tree, so it visits every state of $T$ once, provided \textsc{Parent} returns the parent. In \textsc{Parent}$(x)$ the guessed states $z_m$ follow edges from $r_a$, so each $z_m$ lies in $T$ at layer $m\le d-1$ and its guessed prefix $z_0,\dots,z_m$ is the unique path to it. The guess succeeds iff $u$ is reachable at layer $j-1$ and the guessed path is its path; by (i) the only reachable in-neighbour of $x$ is its parent. So exactly one computation path of \textsc{Parent} succeeds, and it returns the parent. The configurations of the procedure are tour configurations $(x,\mathit{mode},\text{pending child})$ and recovery configurations $(x,\mathit{mode},u,z_m,m)$. The tour configurations occur in a fixed order without repetition, since in an Euler tour each state is entered once and left upward once, and each is reached through the unique successful paths of the preceding recoveries. A recovery configuration arises only from the one tour configuration that invokes \textsc{Parent}$(x)$, by one choice of $u$ and by the guesses $z_1,\dots,z_m$, which $z_m$ determines; hence it too has exactly one path. The work tape holds $x$, $u$, $z_m$, $m$ and the mode, $O(\log(nLp))=O(\log M)$ bits.
\end{proof}

\begin{lemma}[Certifying layer by layer]
\label{lem:b5}
The procedure $\textsc{Certify}(a,\alpha)$, which for $d=0,1,\dots,L-1$ and every state $y$ of layer $d+1$ runs $\textsc{Tour}(a,\alpha,d)$ counting (up to 2) the visited states $x$ of layer $d$ with $x\to y$, and fails as soon as a count reaches 2, returns \textsc{ok} iff $\Gamma_{a,\alpha}$ is certified through layer $L$. It runs \textsc{Tour} only with a $d$ through which $\Gamma_{a,\alpha}$ is certified, has exactly one returning computation path, and reaches every configuration by exactly one path from its start.
\end{lemma}
\begin{proof}
Layer $0$ is $\{r_a\}$, so $\Gamma_{a,\alpha}$ is certified through layer $0$. If it is certified through layer $d$, every reachable state of layer $d$ has exactly one path, so the number of paths to a state $y$ of layer $d+1$ equals the number of its reachable in-neighbours, which the tour counts. All counts are at most one iff the graph is certified through layer $d+1$. \textsc{Tour} is therefore only invoked under its hypothesis. The loop variables $d$, $y$ and the count are part of every configuration of an invocation, so by Lemma~\ref{lem:b4} invocations neither share configurations nor create a second path to any.
\end{proof}

\begin{algorithm}[H]
\caption{The ReachUL decider for $\STCON$.}
\label{alg:reachul}
\small
\begin{algorithmic}[1]
\State compute $n$, $L=nk$ and the prime $p$;\ $R\gets$ false
\For{$a=0,\dots,n-1$}
  \State $\mathit{ok}\gets$ false
  \For{$\alpha=0,\dots,p-1$}
    \If{$\textsc{Certify}(a,\alpha)=\textsc{ok}$}
      \For{$v=0,\dots,n-1$}
        \State $C\gets$ number of states visited by $\textsc{Tour}(a,\alpha,L)$ whose vertex is $v$, counted up to $k+1$
        \If{$C>k$} \textbf{reject} \EndIf
        \If{$a=s$, $v=t$ and $C\ge1$} $R\gets$ true \EndIf
      \EndFor
      \State $\mathit{ok}\gets$ true;\ \textbf{break}
    \EndIf
  \EndFor
  \If{\textbf{not} $\mathit{ok}$} \textbf{reject} \EndIf
\EndFor
\State \textbf{if} $R$ \textbf{then accept else reject}
\end{algorithmic}
\end{algorithm}

\begin{proof}[Oracle-free proof that $\STCON\in\mathrm{ReachUL}$]
Consider Algorithm~\ref{alg:reachul}.

\emph{Correctness.} When an $\alpha$ certifies $\Gamma_{a,\alpha}$ through layer $L$, the count computed for $v$ is $\sum_{i\le L}\Nw_i(a,v)=\Nw^{\le L}(a,v)$ (capped at $k+1$), by Lemma~\ref{lem:b2}. Suppose the input is in $\STCON$. For every $a$, $\Nw(a,v)\le k$ for all $v$, so $\sum_v\Nw^{\le L}(a,v)\le nk$ and Lemma~\ref{lem:b3} provides a certifying $\alpha$; the counts are at most $k$, no rejection occurs, and $R$ becomes true because $\Nw^{\le L}(s,t)\ge1$. Conversely, suppose the algorithm accepts. Then for every $a$ some $\alpha$ certified, and $\Nw^{\le L}(a,v)\le k$ for all $v$; by Lemma~\ref{lem:trunc} with $\kappa\equiv k$ and $B=nk=L$, $\Nw(a,v)\le k$ for all $v$. So $G$ is $k$-strongly few, and $\Nw(s,t)\ge\Nw^{\le L}(s,t)\ge1$. (A rejection because no $\alpha$ certifies is correct by the contrapositive of Lemma~\ref{lem:b3}; a rejection because a count exceeds $k$ is correct because $\Nw\ge\Nw^{\le L}$.)

\emph{Reach-unambiguity.} The outer control is deterministic, and every nondeterministic step happens inside \textsc{Parent}, within invocations of \textsc{Tour} on certified prefixes (Lemma~\ref{lem:b5} for \textsc{Certify}; for the counting tours, \textsc{Certify} has just returned \textsc{ok}). Each invocation has exactly one returning computation path and reaches every configuration by exactly one path (Lemma~\ref{lem:b4}); the outer variables $(a,\alpha,v,C,R,\mathit{ok})$ advance strictly between invocations and are part of every configuration. Hence on every input every configuration is reached by at most one path, and exactly one computation path reaches the end.

\emph{Resources.} The work tape holds a constant number of $O(\log M)$-bit numbers. The number of tours is polynomial, each visits at most $n(L+1)p$ states, and each parent recovery takes $O(L)$ steps, so the surviving path runs in polynomial time.
\end{proof}

The procedure \textsc{Certify} is itself a reach-unambiguous decider for the reach-unambiguity of a layered graph through a given layer, which is what makes this proof self-contained.

\subsection{Proof of Theorem~\ref{thm:ceql} (signed counting)}
\label{app:ceql}
For $x=\langle G,s,t,1^k\rangle$ let $f_{uv}(x)=\Nw^{\le nk}_G(u,v)$. A nondeterministic logspace machine that chooses a length at most $nk$ and then exactly that many edge-occurrence codes, accepting iff the resulting walk goes from $u$ to $v$, has one accepting computation per walk, so $f_{uv}\in\#\mathrm{L}$. For $(u,v)\ne(s,t)$ put $P_{uv}=\prod_{j=0}^{k}(f_{uv}-j)$, and $P_{st}=\prod_{j=1}^{k}(f_{st}-j)$; thus $P_{uv}=0$ iff $0\le f_{uv}\le k$ and $P_{st}=0$ iff $1\le f_{st}\le k$. Let $g=P_{st}^2+\sum_{(u,v)\ne(s,t)}P_{uv}^2$. Logspace-uniform polynomial products and sums preserve GapL, so $g\in$ GapL. Every term is an integer square, so $g(x)=0$ exactly when all truncated counts satisfy the bound and the designated count is positive, and by Lemma~\ref{lem:trunc} (with the per-source budgets of Lemma~\ref{lem:universal}) exactly when $x\in\STCON$. Malformed inputs and $k=0$ get a constant nonzero gap. For pair budgets, use $f_{uv}=\Nw^{\le L}(u,v)$ with $L=\max_u\sum_v\kappa(u,v)$ and products up to $\kappa(u,v)$; for root budgets, $P_v=\prod_{j=0}^{\kappa(v)}(f_v-j)$ and $P_F=\prod_{j=1}^{B}(f_F-j)$ with $f_F=\sum_{v\in F}f_v$. The sum of squares is what makes one zero test enforce every equality at once, although C$_=$L is not known to be closed under complement.

\subsection{Proof of Theorem~\ref{thm:main}, Theorem~\ref{thm:budgets} and the StrongFewL boundary}
\label{app:hardness}
\textbf{Memberships.} ReachUL: Theorem~\ref{thm:reachul}. LOGDCFL: Theorem~\ref{thm:logdcfl} (and Theorem~\ref{thm:rootlogdcfl} for budgets). SC$^2$: Cook's simultaneous simulation of the same polynomial-time deterministic auxiliary pushdown automaton~\cite{cook1979}, a single algorithm with both bounds. C$_=$L: Theorem~\ref{thm:ceql}. DSPACE$(O(\log^2M/\log\log M))$: ReachUL $=\mathrm{RUSPACE}(\log n)\subseteq\mathrm{DSPACE}(\log^2n/\log\log n)$~\cite{allender1998}; this is a pure space bound with no time claim. NL: Theorem~\ref{thm:phase}(i). Also $\STCON\in$ UL $\cap$ coUL, since ReachUL is closed under complement.

\textbf{Hardness.} Let $A$ be accepted by a StrongFewL machine with at most $q(|x|)$ paths between any two configurations, with a unique accepting configuration. Its configuration graph $G_x$ is logspace computable, acyclic (a cycle would give infinitely many paths) and $q(|x|)$-strongly few, possibly after adding a clock or an accepting collector, which multiplies the bound by a fixed polynomial factor; parallel transitions are replaced by private two-edge paths. Hence $x\in A$ iff $\langle G_x,C_0,C_{\mathrm{acc}},1^{q'(|x|)}\rangle\in\STCON$.

\textbf{Why membership in StrongFewL is open.} The reach-unambiguous machines above control the paths from their initial configuration; configurations unreachable from it may still be joined by many paths, and the unique-return oracle composition gives no all-pairs bound inside the composed machine. A proof of $\STCON\in$ StrongFewL needs a total validator whose \emph{entire} configuration graph has polynomially many paths between arbitrary configurations; with the hardness above it would make $\STCON$ StrongFewL-complete.

\subsection{Proof of Theorem~\ref{thm:logdcfl} (the capped stack traversal)}
\label{app:logdcfl}
\begin{algorithm}[H]
\caption{Capped traversal from a source $a$ with a budget of $nk$ nodes. In \emph{target mode} for $b$ it also counts the entered nodes that end at $b$ and aborts at $k+1$.}
\label{alg:traverse}
\small
\begin{algorithmic}[1]
\State $\mathrm{current}\gets a$;\ $\mathrm{next}\gets 1$;\ $\mathrm{depth}\gets 0$;\ $\mathrm{total}\gets 1$;\ $\mathrm{count}_b\gets[a=b]$
\Loop
  \If{$\mathrm{depth}=nk$ \textbf{or} $\mathrm{next}>n$}
    \If{stack empty} \Return (\textsc{ok}, $\mathrm{total}$, $\mathrm{count}_b$) \EndIf
    \State pop $(v,j)$;\ $\mathrm{current}\gets v$;\ $\mathrm{next}\gets j$;\ $\mathrm{depth}\gets\mathrm{depth}-1$;\ \textbf{continue}
  \EndIf
  \State $w\gets\mathrm{next}$;\ $\mathrm{next}\gets\mathrm{next}+1$
  \If{$(\mathrm{current},w)\notin E$} \textbf{continue} \EndIf
  \State push $(\mathrm{current},\mathrm{next})$;\ $\mathrm{current}\gets w$;\ $\mathrm{next}\gets 1$;\ $\mathrm{depth}\gets\mathrm{depth}+1$;\ $\mathrm{total}\gets\mathrm{total}+1$
  \If{$\mathrm{total}=nk+1$} \Return \textsc{overflow} \EndIf
  \If{target mode \textbf{and} $\mathrm{current}=b$}
    \State $\mathrm{count}_b\gets\mathrm{count}_b+1$;\ \textbf{if} $\mathrm{count}_b=k+1$ \Return \textsc{overflow}
  \EndIf
\EndLoop
\end{algorithmic}
\end{algorithm}

For a source $a$ and length $\ell$, the walk-unfolding tree $T^\ell_a$ has one node per walk of length at most $\ell$ from $a$; distinct walks are distinct nodes even when they end at the same vertex, so $|T^\ell_a|=\sum_b\Nw^{\le\ell}(a,b)$. Algorithm~\ref{alg:traverse} is a depth-first traversal of $T^{nk}_a$ whose stack holds, for each node on the current branch, its vertex and the index of its next successor.
\begin{lemma}
\label{lem:b6}
(i) If $|T^{nk}_a|>nk$ then $G$ is not $k$-strongly few. (ii) If $G$ is $k$-strongly few then $|T^{nk}_a|\le nk$ for every $a$. (iii) If $G$ has a directed cycle, a source on it has $|T^{nk}_a|\ge nk+1$. (iv) If $G$ is acyclic, $\Nw^{\le nk}(a,b)=\Nw(a,b)$.
\end{lemma}
\begin{proof}
(i) More than $nk$ nodes among $n$ endpoints put at least $k+1$ distinct walks on some endpoint. (ii) is (i) contrapositive, summed over endpoints. (iii) Following the cycle gives a walk of every length $0,\dots,nk$. (iv) In an acyclic graph every walk is a simple path of length at most $n-1\le nk$.
\end{proof}
\begin{proof}[Proof of Theorem~\ref{thm:logdcfl}]
If the input is in $\STCON$, every global traversal completes by Lemma~\ref{lem:b6}(ii), every recount is exact by (iv) and at most $k$, and the $(s,t)$ count is positive. Conversely, if the decider accepts, no source overflowed, so $G$ is acyclic by (iii); all recounts are then exact and at most $k$, so $G$ is $k$-strongly few, and $\Nw(s,t)\ge1$. The recounts start only after every global traversal has bounded its tree by $nk$ nodes, so an invalid input cannot force an exponentially long recount. Resources: the work tape holds a constant number of $O(\log M)$-bit counters; a stack frame $(v,j)$ has $2\lceil\log_2(n+2)\rceil$ bits; a traversal enters at most $nk+1$ nodes and tests at most $n$ successors per node, so with an $O(M)$ input scan per adjacency test it takes $O(n^2kM+nk\log M)$ steps, and there are $n+n^2$ traversals. The decider is therefore a deterministic auxiliary pushdown automaton with logarithmic work tape running in polynomial time, and such machines accept exactly LOGDCFL~\cite{sudborough1978}.
\end{proof}
Without the budget, this depth-first search is the one Buntrock et al.~\cite{bjlr1991} use for reach-few machines, whose configuration graphs are promised to have few paths; the budget and the order of the passes are what make it total.

\subsection{Proof of Theorem~\ref{thm:complete} (completeness)}
\label{app:complete}
\emph{Membership} of all three languages follows from Theorem~\ref{thm:budgets} (for $\textsc{FBFew}$, apply the root-budget theorem to $s$ in $G$ and to $t$ in the reversed graph, since $\Nw_{G^{\mathrm{rev}}}(t,v)=\Nw_G(v,t)$).
\emph{$\textsc{RootFew}_k$.} For $k=1$ the language is $\textsc{STCON}_{\mathrm{ru}}$, which is ReachUL-complete~\cite{lange1997}. For $k>1$, add a new root $r$ and $k$ internally disjoint two-edge paths $r\to a_i\to s$; then $\Nw_H(r,v)=k\,\Nw_G(s,v)$ for every old vertex, so a bound of one becomes a bound of $k$ and reachability is preserved.
\emph{DAGs.} Given a rooted threshold-one instance on $n$ vertices, build layers $0,\dots,2n-1$ reproducing every edge between consecutive layers, and add for every vertex $v$ a fresh sink $c_v$ with an edge from every layered copy of $v$; use $(s,0)$ as root and $c_t$ as target. On a reach-unambiguous input every layered state and every sink has at most one rooted path. If the reachable part is acyclic but ambiguous, two simple rooted paths of length at most $n-1$ reach one sink; if a cycle is reachable, a simple prefix to a cycle vertex and the same prefix followed by one tour give two walks of length at most $2n-1$ to one sink. The $k$-fold prefix gadget then gives DAG hardness for every fixed $k$.
\emph{Recognition and two cones.} Rooted reach-unambiguity recognition is ReachUL-complete~\cite{dinesh2026}, and the $k$-fold prefix reduces it to rooted $k$-few recognition. For $\textsc{FBFew}_{k,1}$, given $\langle G,s\rangle$ add fresh vertices $r,z$ and the edges $r\to s$, $r\to z$, and use $r$ as source and $z$ as target: $\Nw_H(r,v)=\Nw_G(s,v)$ for old vertices, every backward count into $z$ is at most one, and $r\leadsto z$ holds, so $\langle G,s\rangle$ is rooted $k$-few iff $\langle H,r,z\rangle\in\textsc{FBFew}_{k,1}$.

\subsection{Proof of Theorem~\ref{thm:compiler} (the fail-closed compiler)}
Generate the clocked configuration graph of $M_0$ on $x$ in logspace, take the initial configuration as root and the accepting configurations as $F$, and attach the budgets $\kappa_x$. The graph has polynomial size, a computation path is exactly a rooted path, and the clock makes the graph finite even on rejected inputs. Theorems~\ref{thm:rootlogdcfl}, \ref{thm:reachul} and~\ref{thm:ceql} apply. The compiler validates a bound on computation \emph{prefixes} from the initial configuration; it does not produce a machine with few paths between arbitrary configurations, which is the StrongFewL question above.

\subsection{Proof of Theorem~\ref{thm:phase} (the threshold phase transition)}
\label{app:phase}
\textbf{(i) Polynomial thresholds: collision-averaged fingerprints.} Fix a root, budgets of total mass $B$ and maximum $K\ge1$, edge occurrences $e_1,\dots,e_m$, $q_0=(B-1)\binom{K+1}2+1$, a prime $p>\max\{m,q_0\}$ and a set $A\subseteq\Fp$ of $q_0$ points. For a walk $P$ of length $d\le B$ let $\phi_\alpha(P)=(d,h_\alpha(P))$. For an endpoint $v$ let $W_v$ be the rooted walks of length at most $B$ ending at $v$, $c_v=|W_v|$, and $D_v=\sum_{\alpha\in A}|\phi_\alpha(W_v)|$, the total number of distinct images.
\begin{lemma}[Collision-averaged decoding]
For every $v$ and $t\le K$: $c_v\le t$ iff $D_v\le q_0t$. Hence $\min\{c_v,K+1\}=\min\{\lceil D_v/q_0\rceil,K+1\}$.
\end{lemma}
\begin{proof}
If $c_v\le t$, each point yields at most $c_v$ images, so $D_v\le q_0t$. Conversely let $c_v\ge t+1$ and pick $t+1$ distinct walks. Walks of different lengths never collide; two of the same length collide at at most $B-1$ points. At a fixed $\alpha$ with collision classes of sizes $s_j$, the number of lost images is $\sum_j(s_j-1)\le\sum_j\binom{s_j}2$, the number of colliding pairs. Summing over $A$, the selected walks have at least $q_0(t+1)-(B-1)\binom{t+1}2>q_0t$ images, and the full family has at least as many.
\end{proof}
Taking the disjoint union of the lifts $H_\alpha$, $\alpha\in A$, under a super-root, $D_v$ is the number of reachable states in the fiber of $v$. A logspace machine with an NL oracle counts them, and $\mathrm{L}^{\mathrm{NL}}=\mathrm{NL}$ because NL $=$ coNL~\cite{immerman1988,szelepcsenyi1988}. So \textsc{RootBudget} $\in$ NL, the capped counts are in $\mathrm{FL}^{\mathrm{NL}}$, and with Lemma~\ref{lem:universal} \textsc{PairBudget} $\in$ NL. Applied to the accepting computations of an NL machine (with their length stored), the same decoding shows $\{x:f(x)\ge\tau(x)\}\in$ NL for $f\in\#$L and a polynomially bounded logspace threshold $\tau$; conversely every $L\in$ NL is $\{x:f(x)\ge1\}$ for its accepting-path count.

\textbf{(ii) Binary thresholds are PL-complete.} \emph{Membership:} reject cyclic graphs (NL $\subseteq$ PL); on a DAG each $\Nw_G(u,v)$ is a \#L function, and a logspace machine has exactly $K+1$ accepting branches by guessing a number in $[0,K]$ bit by bit against $K$, so $(K+1)-\Nw_G(u,v)\in$ GapL and $\Nw_G(u,v)\le K$ is a PL predicate. PL is closed under complement, union, intersection and NC$^1$ reductions~\cite{beigel1997}, so the conjunction over all pairs and the reachability test stays in PL. \emph{Hardness:} take a probabilistic logspace machine with exactly $p$ binary choices that accepts on $A(x)$ of its $2^p$ paths, $x\in L\iff A(x)>2^{p-1}$. Between common endpoints $a,b$ put a chain of $p+1$ diamonds ($C=2^{p+1}$ paths) in parallel with the layered configuration graph ($A(x)$ paths) and set $K=C+2^{p-1}$. Every pair other than $(a,b)$ has at most $C$ paths, $\Nw(a,b)=C+A(x)$, and $a$ reaches $b$ on every input, so the instance is accepted iff $A(x)\le2^{p-1}$; this reduces the complement of $L$, and PL is closed under complement.

\textbf{(iii) One pair, fixed cap.} Reject non-DAGs with the NL cycle test and NL $=$ coNL. On a DAG the complement guesses $k+1$ paths in parallel, storing their current vertices and pairwise-difference flags, $O(\log n)$ space for fixed $k$, so the language is in coNL $=$ NL. Hardness: reduce DAG non-reachability by adding a new source with $k+1$ internally disjoint prefixes into the old source; the new count is $0$ if the target was unreachable and at least $k+1$ otherwise.

\subsection{Proofs of the structure theorems}
\label{app:structure}
\begin{proof}[Proof of Theorem~\ref{thm:kernel}]
\emph{Lower bound.} The unit $s$--$t$ flows of $H$ form a polytope; flow conservation at the $|V|-2$ internal vertices together with the unit-value constraint has rank $|V|-1$, so the polytope has dimension $|E|-|V|+1=c(H)$. Its vertices are the indicator vectors of the $s$--$t$ paths (the constraint matrix is totally unimodular), and a $c$-dimensional polytope has at least $c+1$ vertices. \emph{Upper bound.} Fix a path $P_0$. The map $P\mapsto\chi_P\oplus\chi_{P_0}$ sends paths to even subgraphs, elements of the cycle space over GF(2), which has dimension $c(H)$; the map is injective. \emph{Kernel.} In a trimmed $H$ every vertex $x\ne t$ has out-degree at least one, and summing over all $x\ne t$, the source included, gives $\sum_{x\ne t}(d^+(x)-1)=|E|-(|V|-1)=c(H)$. So at most $c(H)$ vertices have out-degree at least two, and likewise for in-degree. Every vertex that survives the contraction is a terminal or one of these, and contracting a vertex of in- and out-degree one merges two edges and preserves the paths and $c$. This gives at most $2c+2\le2r$ vertices and, since $c$ is preserved, at most $c+(2c+2)-1=3c+1$ edges.
\end{proof}
\begin{proof}[Proof of Theorem~\ref{thm:block}]
A simple $s$--$t$ path follows the unique route of blocks in the block--cut tree. It cannot enter an off-route block and return, since it would repeat the cut vertex through which it entered. Inside route block $B_i$ it traverses one $a_i$--$b_i$ path, and arbitrary choices concatenate into distinct $s$--$t$ paths. The cyclomatic numbers add because the blocks share only cut vertices. If $r\le k$, each block with more than one path contributes a factor of at least two.
\end{proof}

\begin{proof}[Proof of Theorem~\ref{thm:basis}]
\emph{Minimal obstructions.} Let $H=(V,E;s,t)$ be acyclic, trimmed, edge-minimal subject to $\Nw_H(s,t)>k$, with its mandatory initial and final chains contracted, so $d^+(s)\ge2$ and $d^-(t)\ge2$; write $r=\Nw_H(s,t)$ and $c=c(H)$. For an out-edge $e$ of $s$, let $r_e$ be the number of terminal paths through $e$. Deleting $e$ leaves at most $k$ paths, so $r_e\ge r-k$; the classes of the at least two out-edges partition the paths, so $r\ge2(r-k)$, i.e.\ $r\le2k$. For $c$: in the unit-flow polytope (dimension $c$), the face $\{x_e=0\}$ of an irredundant inequality is a facet of dimension $c-1$ whose vertices are the paths avoiding $e$, at most $k$ by minimality; a $(c-1)$-dimensional polytope has at least $c$ vertices, so $c\le k$. The source has out-degree at least two and the sink in-degree at least two, and the excess-degree identities of Theorem~\ref{thm:kernel} give at most $c$ vertices of each kind, so after suppressing unary vertices the core has at most $2c\le2k$ vertices and $(2c-1)+c\le3k-1$ edges.
\emph{The basis.} If $G$ is acyclic and violates the bound at some pair, restrict to the vertices and edges on paths of that pair, take an edge-minimal violating subgraph, contract its mandatory terminal chains and suppress unary vertices: the result is a core as above, and the subgraph of $G$ is a directed subdivision of it. Conversely, subdividing edges preserves the number of terminal paths, so a subdivision of a core with more than $k$ paths is a violation. Directed cycles are excluded separately. Only finitely many multigraphs have at most $2k$ vertices and $3k-1$ edges, which gives the counting bound.
\end{proof}

\begin{proof}[Proof of Theorem~\ref{thm:normal}]
(i) If $G$ is $k$-strongly few it is acyclic, $A^n=0$, and $C=I+A+\dots+A^{n-1}=(I-A)^{-1}$ solves $C=I+AC$ with entries the path counts. Conversely let $C\ge0$ solve it and put $\mathbf{r}=C\mathbf 1$; then $\mathbf r=\mathbf 1+A\mathbf r$, so $r_u=1+\sum_{u\to w}r_w>r_v$ for every edge $u\to v$, and a directed cycle would give a strict chain returning to its start. So $G$ is acyclic, $I-A$ is invertible, and $C=(I-A)^{-1}$ is the path-count matrix with entries at most $k$.
(ii) For $1\le\tau<k$ let $\mu(v)=\max_a\Nw(a,v)$, $L_\tau=\{v:\mu(v)\le\tau\}$ and $H_\tau$ the rest. Since $\mu(v)\ge\mu(u)\Nw(u,v)$ along reachability, no path goes from $H_\tau$ to $L_\tau$, both parts are path-convex, $G[L_\tau]$ is $\tau$-strongly few and $G[H_\tau]$ is $\lfloor k/(\tau+1)\rfloor$-strongly few. With $\tau=\lfloor\sqrt k\rfloor$ both parts are $\lfloor\sqrt k\rfloor$-strongly few. Ordering $L$ before $H$, the path-count matrix factors as
$C=\begin{pmatrix}C_L&0\\0&I\end{pmatrix}\begin{pmatrix}I&B\\0&I\end{pmatrix}\begin{pmatrix}I&0\\0&C_H\end{pmatrix}$, where the middle factor counts the identity paths and single bridge edges, a mangrove. Recursing on both blocks, the number of mangrove leaves $\ell(k)$ satisfies $\ell(1)=1$ and $\ell(k)\le2\ell(\lfloor\sqrt k\rfloor)$, so $\ell(k)\le2\log_2k$, and with one bridge factor per internal node the number of factors is $2\ell-1<4\log_2k$; repeated square roots reach one in at most $1+\lfloor\log_2\log_2k\rfloor$ levels. \emph{Optimality.} In a chain of $r$ diamonds with merge vertices $z_0,\dots,z_r$, two consecutive merge vertices cannot share a path-convex mangrove block (it would contain both branches of a diamond), and two non-consecutive ones would force all intermediate merge vertices into the block; so the $z_i$ lie in $r+1$ distinct blocks, and a binary recursive decomposition has depth at least $\lceil\log_2(r+1)\rceil$. With $k=2^r$ the bounds $O(\log k)$ and $O(\log\log k)$ are therefore tight.
\end{proof}

\begin{proof}[Proof of Theorem~\ref{thm:density}]
\emph{Density.} Since $G$ is acyclic every triangle of the underlying graph is transitively oriented $x\to y\to z$, $x\to z$; assign it to its long edge. An edge with $r_e$ triangles has $1+r_e$ paths between its endpoints, so $r_e\le k-1$ and the number of triangles is $T\le(k-1)m$. Summing $|N(x)\cap N(y)|\ge d_x+d_y-n$ over edges gives $3T\ge\sum_xd_x^2-mn\ge4m^2/n-mn$ by Cauchy--Schwarz, so $3(k-1)m\ge4m^2/n-mn$, which rearranges to the bound. For the lower bound, take $|B|=k-1$, split the remaining $n-k+1$ vertices evenly into $A$ and $C$, and add all edges $A\to B$, $B\to C$, $A\to C$: each pair $(a,c)$ has exactly $k$ paths and every other pair at most one, giving $\lfloor(n-k+1)^2/4\rfloor+(k-1)(n-k+1)$ edges.
\emph{Width.} Let $X$ be an antichain in $I(u,v)$. For each $x\in X$ concatenate a path $u\leadsto x$ and a path $x\leadsto v$; two elements of $X$ on the same $u$--$v$ path would be comparable, so $|X|\le\Nw(u,v)$. Every vertex and edge of the interval extends to a $u$--$v$ path, so the at most $k$ such paths cover it.
\end{proof}

\subsection{Proof of Theorem~\ref{thm:regimes} (tractable regimes)}
\textbf{Terminal width.} For a root $a$ of a DAG, the walk-unfolding tree $T_a$ has $L_a=1+\sum_{v:d^+(v)>0}\Nw(a,v)(d^+(v)-1)$ leaves (a rooted tree has one more leaf than the sum of its excess out-degrees), and every leaf ends at a sink, so $L_a\le k\tau^+(G)$ on a valid input, $\tau^+$ being the number of sinks. A depth-first traversal that stores a frame only at nodes with at least two children keeps at most $L_a-1$ frames, since a tree with $L$ leaves has at most $L-1$ branching nodes. The decider counts sources and sinks (rejecting a nonempty graph with neither), chooses the orientation with fewer sinks (reversing edges transposes the path-count matrix and swaps sources and sinks), and runs the source-global and recount passes of Theorem~\ref{thm:logdcfl} with choice frames only, keeping the node cap $nk+1$, a length cap $n$ (a walk of length $n$ certifies a cycle) and a stack cap $k\lambda(G)-1$. On a valid input no cap triggers and the recounts are exact; on an invalid input any cap may reject, and if none does, the absence of a length-$n$ walk proves acyclicity and acceptance is sound. Each frame has $O(\log M)$ bits, giving $O((k\lambda(G)+1)\log M)$ space in polynomial time, hence L for fixed $k$ and $\lambda$, and SC$^d$ when $k\lambda(G)=O(\log^{d-1}M)$.

\textbf{Block-cyclomatic width.} In a connected component with cyclomatic number $c$, enumerate the $2^c$ vectors of the cycle space relative to a spanning tree, generating each selected edge set on demand; requiring in/out-degrees $(0,1)$ at $s$, $(1,0)$ at $t$, $(0,0)$ or $(1,1)$ elsewhere and undirected connectivity of the selection characterizes the directed simple $s$--$t$ paths, one vector each, and directed cycles are detected the same way, in $2^{O(c)}M^{O(1)}$ time and $O(c+\log M)$ space. Every simple cycle lies in one biconnected block, so cycles are detected blockwise. On a DAG, for each pair follow the unique route of blocks in the block--cut tree (navigated by undirected connectivity tests after deleting a candidate articulation vertex, using Reingold's algorithm~\cite{reingold2008}), count the paths between the entry and exit vertices of each route block by cycle-space enumeration with at most $b$ bits, and multiply with a counter saturated at $k+1$ (Theorem~\ref{thm:block}). This gives $2^{O(b)}M^{O(1)}$ time and $O(b+\log M)$ space; a chain of diamonds has unbounded global cycle rank but block-cyclomatic width one.

\section{The decoding mechanisms: formal statements and proofs}
\label{app:decoding}
This appendix first states the decoding results formally, then proves them.

\subsection*{Formal statements}
\label{app:formaldecoding}

\subsection{Cost model, constraints and the step graph}
\textbf{Cost.} Let $\tau(S)$ be the latency of one forward pass that appends $S$ tokens to a cached prefix and returns logits at every new position. On a laptop, single-stream decoding is bound by memory bandwidth: a pass reads all weights (1.1\,GB for Qwen3.5-2B, 2.26\,GB for the 4B at 4 bits), and the arithmetic for a few tokens is nearly free. On an M2 Pro, $\tau(4)=1.36\,\tau(1)$, $\tau(8)=1.58\,\tau(1)$ and $\tau(16)=2.04\,\tau(1)$ on the 2B (Fig.~\ref{fig:passcost} in Appendix~\ref{app:speed}). Hence the cost of decoding is its number of passes. A model maps a token prefix $u$ to logits $\ell(u)$; by causality, one pass over $v_1\cdots v_S$ after $u$ yields the same cache and logits as $S$ single-token passes, up to floating-point reassociation.

\textbf{Constraints.} Let $\Sigma$ be the byte alphabet.
\begin{definition}[Control machine]
A \emph{control machine} is $\cA=(C,\delta,c_0,F)$ with configurations $C$, a partial step $\delta:C\times\Sigma\rightharpoonup C$, start $c_0$ and accepting set $F$. $\ds$ extends $\delta$ to strings, $L(\cA)=\{x:\ds(c_0,x)\in F\}$, and $\live(c)=\{b:\delta(c,b)\text{ defined}\}$. $\cA$ is \emph{trim} if an accepting configuration is reachable from every configuration reachable from $c_0$.
\end{definition}
A DFA is the case of a finite $C=Q$; a deterministic pushdown is the case $C=Q\times R^*$ (Section~\ref{sec:pda}). The vocabulary is a set $\V\subset\Sigma^+$ of byte strings plus a token \eos. Constrained decoding at configuration $c$ samples from the model's logits restricted to $A(c)=\{t\in\V:\ds(c,t)\text{ defined}\}\cup\{\eos: c\in F\}$ (the \emph{mask}), then moves to $\ds(c,t)$; if $\cA$ is trim, every emitted prefix extends to a string of $L(\cA)$. The \emph{vocabulary trie} $\TV$ is the tree of all token prefixes (0.59M nodes for the 248k-token Qwen3.5 vocabulary).

\begin{definition}[Step graph]
For $c\in C$, $P_c$ has vertices $(w,\ds(c,w))$ for the trie nodes $w$ at which $\ds(c,w)$ is defined, edges $(w,c')\to(wb,\delta(c',b))$, and root $(\varepsilon,c)$.
\end{definition}

\begin{proposition}
\label{prop:mangrove}
For every deterministic control machine and every $c$, $P_c$ is an out-tree; hence a mangrove. The disjoint union of $P_c$ over any set of configurations is a forest, and each token $t\in A(c)\setminus\{\eos\}$ corresponds to exactly one vertex of $P_c$, which carries $\ds(c,t)$.
\end{proposition}
\begin{proof}
Every edge lengthens $w$; a vertex $(wb,c'')$ has the single parent $(w,\ds(c,w))$, because the trie fixes $w$ and determinism fixes the configuration.
\end{proof}

The proposition is elementary. Its use is that every fact about a decoding step can be computed along a tree, once per vertex, without merging information from two parents. Neither a nondeterministic automaton nor the lattice of all tokenizations of a string has this property (Fig.~\ref{fig:stepgraph}).

\subsection{The token table, with and without storing it}
\label{sec:table}
\begin{theorem}[Two ways to evaluate the step graph]
\label{thm:table}
Let $\cA$ be a DFA with states $Q$, and $T[t,q]=\ds(q,t)\in Q\cup\{\dead\}$.
\begin{enumerate}[label=(\roman*),leftmargin=1.6em,itemsep=1pt,topsep=2pt]
\item \emph{One pass.} Walking the trie level by level with all $|Q|$ states as the columns of one array computes $T$ in $\mathrm{depth}(\TV)$ gathers with total work at most $|\TV|\cdot|Q|$, within a factor $|\TV|/|\V|$ (2.4 for Qwen3.5) of the $|\V|\cdot|Q|$ entries any algorithm must write, and never more than evaluating every token from every state separately.
\item \emph{No table.} For any state $q$, the row $T[\cdot,q]$ is computed by $|\V|$ independent walks from $q$, one per token, with $\sum_{t}|t|$ lookups in total and depth $\max_t|t|$, using only the automaton ($256|Q|$ entries) and the packed token bytes.
\end{enumerate}
\end{theorem}
\begin{proof}[Proof idea]
(i) Row $w$ at depth $d$ holds the vertices of $\{P_q\}_{q}$ above trie node $w$; by Proposition~\ref{prop:mangrove} each has one parent, in the row of $w$'s parent, so one gather computes a level. (ii) The walk of token $t$ from $q$ is the root-to-$t$ path of $P_q$, which depends on nothing else (Appendix~\ref{app:table}).
\end{proof}

\textbf{What the engine does.} Either way, the mask is a row and the transition is an entry, so both stay on the GPU: the mask at $q$ is $[T[\cdot,q]\ne\dead]$ and the next state is $T[y,q]$, indexed by device values. The host never enters the decode step, so a constrained step pipelines like a free one. With (i) the table costs $3|Q||\V|$ bytes, 77\,MB for a 104-state schema and 1.5\,GB for 1,970 states, and is built in 0.2--0.6\,s for the schemas we use. With (ii) a Metal kernel runs one thread per token each step (7--18\,$\mu$s against a 7--15\,ms step) and the grammar costs its automaton, 512 bytes per state, plus 2.9\,MB of packed vocabulary shared by all grammars. We call (ii) the \emph{walk kernel}. It trades a stored table for recomputation along the single path of each token, the same trade the certified lift makes when it counts reachable states instead of storing walks.

\subsection{Forced output costs no pass}
\label{sec:forced}
\begin{definition}[Forced configurations and chains]
A configuration $c$ is \emph{forced} if $c\notin F$ and $|\live(c)|=1$, and a \emph{decision point} otherwise. $\chain(c)$ is the byte string read by following the unique live byte from $c$ while the configuration is forced.
\end{definition}
Forcing is a property of the automaton, not of the step graph: it says that only one \emph{byte} may follow, whatever tokens exist. If $\cA$ is trim, $\chain(c)$ is finite and ends at a decision point; otherwise every run from $c$ would follow it forever and never accept.

\begin{definition}[Segmentation policy; reference and forced-span decoders]
A \emph{segmentation policy} $\varphi$ assigns to each configuration $c$ a token string $\varphi(c)\in\V^*$ whose bytes are a prefix of $\chain(c)$, and to each sampling configuration a nonempty set $A_\varphi(c)\subseteq A(c)$ of admissible tokens; both depend on $\cA$ and the tokenizer only, never on the model. Both decoders below feed $\varphi(c_0)$ with the prompt and sample first at $c_1=\ds(c_0,\varphi(c_0))$. At a sampling configuration $c$, the \emph{reference decoder} $R_\varphi$ samples $t$ from the logits restricted to $A_\varphi(c)$, sets $c'=\ds(c,t)$, feeds the tokens of $\varphi(c')$ one pass each, and continues at $\ds(c',\varphi(c'))$. The \emph{forced-span decoder} $F_\varphi$ is identical except that $t$ and $\varphi(c')$ are fed in one pass.
\end{definition}

\begin{theorem}[Pass optimality]
\label{thm:pass}
Let $\cA$ be trim and $\varphi$ a segmentation policy.
\begin{enumerate}[label=(\roman*),leftmargin=1.6em,itemsep=1pt,topsep=2pt]
\item \emph{Exactness.} For every causal model, sampling rule and seed, $F_\varphi$ and $R_\varphi$ emit the same tokens and, in exact arithmetic, see the same logits at every sampling point.
\item \emph{Cost.} On an output with $N$ sampled tokens, $F_\varphi$ makes $N$ passes and $R_\varphi$ one pass per emitted token.
\item \emph{Lower bound.} If $|A_\varphi(c)|\ge2$ at every sampling configuration, every non-speculative decoder (one that feeds the model only tokens it has already emitted) whose output distribution equals $R_\varphi$'s for every model under greedy or softmax sampling makes at least $N$ passes on every output with $N$ sampled tokens.
\end{enumerate}
\end{theorem}

\begin{corollary}[Length independence]
\label{cor:length}
Let $\varphi$ commit the whole chain. Let $c,c^\dagger$ be configurations such that every accepting run from $c$ passes through $c^\dagger$, the runs from $c$ to $c^\dagger$ form a trimmed two-terminal DAG $H$ with $r$ paths, and no configuration of $H$ other than $c^\dagger$ accepts. Then $F_\varphi$ samples at most $c(H)\le r-1$ tokens that start in $H$ before $c^\dagger$. A key or any other string the constraint determines costs no pass, and a field with $r$ admissible values costs at most $r-1$ passes, however long they are.
\end{corollary}
\begin{proof}
Every sampled token starts at a decision point; a decision point of $H$ other than $c^\dagger$ has two live bytes, each on an accepting run through $c^\dagger$ since $\cA$ is trim, so it is a branch vertex of $H$. Distinct sampled tokens start at distinct vertices, and Theorem~\ref{thm:kernel} bounds the branch vertices by $c(H)\le r-1$.
\end{proof}

The DAG hypothesis matters: an automaton with free-text fields has cycles and is not few-path, and Theorem~\ref{thm:kernel} applies only to its acyclic parts between two cut configurations (keys, enum values, booleans, closing brackets).

\textbf{What the engine does.} After each sampled token, the engine looks up the chain of the new configuration, splits it into tokens (Section~\ref{sec:canon}) and feeds the sampled token and the forced tokens in one pass of $m+1$ positions. Once the model has written the closing quote of \texttt{"name":"Maria"}, the schema allows only \texttt{,"age":}, seven bytes and two tokens, fed with the quote. On our six JSON schemas about 30\% of output tokens are forced (Fig.~\ref{fig:strip}). Speculative decoding~\cite{leviathan2023,chen2023} goes below $N$ passes only by verification; forced spans are drafts accepted with probability one.

\subsection{Splitting forced text the way the tokenizer would}
\label{sec:canon}
A byte-level BPE tokenizer is a pre-tokenizer $\pi$, which cuts a string into blocks (words, numbers, punctuation runs), followed by a per-block encoder: $\tok(s)=\mathrm{bpe}(s_1)\cdots\mathrm{bpe}(s_m)$ for $\pi(s)=(s_1,\dots,s_m)$. Models are trained on $\tok(s)$. Two things can go wrong with forced bytes. A non-canonical split puts the model in a context it has not seen. And committing the canonical split of forced bytes $w$ too early can contradict the canonical split of $wx$ once the next bytes $x$ arrive, because $x$ may merge into the last block of $w$. Among the segmentations the tokenizer can produce, block boundaries are cut vertices, so only the last block is at risk (Theorem~\ref{thm:block}); in the lattice of \emph{all} tokenizations they are not, since a token such as \texttt{"M} may cross a boundary. The question is exactly when the last block is safe.

\begin{assumption}[Locality of the pre-tokenizer]
\label{ass:local}
For a string $s$ let $m(s)$ be the start of the last block of $\pi(s)$ or of the maximal trailing whitespace run of $s$, whichever is smaller. For every $s$ and character $a$, every block of $\pi(s)$ that ends at or before $m(s)$ is a block of $\pi(sa)$.
\end{assumption}
Regex pre-tokenizers look ahead one character at whitespace (\texttt{\textbackslash s+(?!\textbackslash S)}), which is why the whitespace clause is needed. We tested Assumption~\ref{ass:local} on the Qwen3.5 pre-tokenizer with 720{,}000 random (string, character) pairs and 400{,}000 multi-character extensions over JSON punctuation, letters, digits, whitespace, contractions and non-ASCII characters, and found no violation.

\begin{theorem}[Canonical commitment]
\label{thm:canon}
Let $\pi$ satisfy Assumption~\ref{ass:local} and let $o$ be the output so far, in configuration $c$.
\begin{enumerate}[label=(\roman*),leftmargin=1.6em,itemsep=1pt,topsep=2pt]
\item For every string $x$, every block of $\pi(o)$ that ends at or before $m(o)$ is a block of $\pi(ox)$; their canonical tokens can be committed now.
\item If $o$ does not end in whitespace and $\pi(o)$ is a prefix of $\pi(ob)$ for every $b\in\live(c)$, then $\pi(o)$ is a prefix of $\pi(ox)$ for every $x$ with $\ds(c,x)$ defined, so $\tok(o)$ is a prefix of $\tok(ox)$.
\end{enumerate}
\end{theorem}

Part (i) says what may always be committed; part (ii) is a test on single bytes that certifies the last block, and it quantifies only over the bytes the constraint allows. \textbf{What the engine does} (the \emph{canonical policy}): after a sampled token and its chain, commit the blocks licensed by Theorem~\ref{thm:canon} and hold back the rest; if the held-back block is one vocabulary token of length $\ell$, the next sampled token must span at least $\ell$ bytes (\emph{boundary cover}), so the model chooses how the block continues in one token, as the tokenizer would. The policy only chooses which tokens ride in the next pass, so by Theorem~\ref{thm:pass} it costs nothing.

\subsection{Valid outputs under a token budget}
\label{sec:budget}
Every deployed decoder has a token limit, and a constrained output that reaches it is cut mid-value. Fix a policy $\varphi$ (without healing) and count every emitted token, sampled or forced, against a budget $R$; an output that exhausts the budget ends there and is valid iff it ends in $F$. For a sampling configuration $c$ define
\[
\need(c)=\begin{cases}0 & c\in F,\\ \min_{t\in A_\varphi(c)\setminus\{\eos\}}\cost(c,t) & \text{otherwise,}\end{cases}
\qquad
\cost(c,t)=1+|\varphi(c')|+\need(c''),
\]
with $c'=\ds(c,t)$ and $c''=\ds(c',\varphi(c'))$. With $r$ tokens left, the \emph{budget mask} is $M_r(c)=\{t\in A_\varphi(c):\cost(c,t)\le r\}\cup\{\eos:c\in F\}$.

\begin{theorem}[Least restrictive valid mask]
\label{thm:budget}
Let $\cA$ be a trim DFA. (i) $\need$ is a shortest-path distance, computed by Bellman--Ford in at most $|Q|$ vectorized rounds. (ii) If $\need(c_1)\le R$, where $R$ is the budget left at the first sampling configuration $c_1$, every output decoded under $M_r$, by any model and sampling rule, is in $L(\cA)$, whether it ends by \eos{} or by the budget. (iii) $t\in M_r(c)$ iff some continuation after $t$ reaches $F$ within the remaining budget; any mask that admits a token outside $M_r(c)$ admits a model whose output is invalid.
\end{theorem}
The mask is inactive until the last $\max_{c,t}\cost(c,t)$ tokens (17--38 on our schemas), so outputs with room are unchanged; if $\need(c_1)>R$, no valid output fits, which the decoder can report before it starts.

\subsection{Recursion: a deterministic pushdown with a capped stack}
\label{sec:pda}
A schema that refers to itself, such as a tree of nodes with children, has no regular expression. Definitions on a cycle of the schema's reference graph become nonterminals; each is a DFA over bytes and call symbols; a call edge $q\xrightarrow{\mathrm{call}_j}r$ pushes the return state $r$ and continues at the callee's start, and a callee that accepts with no successor pops and continues at $r$; a callee that accepts with successors pops lazily, when the next byte is not its own. Two LL(1) conditions, checked at compile time, make at most one move applicable, so the machine is deterministic; grammars that violate them are refused. Call a pop \emph{real} if it removes a symbol that was on the stack before the token began, and \emph{exposed} if another move of the token's run follows it.

\begin{theorem}[Stack locality]
\label{thm:pda}
For a token $t$ and a state $q$, the run of $t$ from $(q,\sigma)$ depends on $\sigma$ only through the symbols revealed by its exposed real pops: there is a decision tree of depth $e\le|t|$, branching only on those symbols, whose leaves give the final state (or ``continue at the popped state''), the number of real pops and the pushed string. If $e=0$, one table entry serves every stack.
\end{theorem}

\textbf{What the engine does.} The token table splits into a context-free table $\V\times Q$ and a small context table indexed by the top $e$ stack symbols, both computed by the level pass of Theorem~\ref{thm:table}(i) with rows that carry a symbolic stack. On Qwen3.5's vocabulary and our recursive schemas only 7--12 tokens at 4--11 states are context-dependent, and $e\le2$ suffices except for the token \texttt{\}\}\}}, which is masked; tokens that would push more than three symbols are masked as well. Masking such tokens removes no byte string from the language, because every single byte is a token. The stack is a 64-entry array on the GPU, and masking and advancing are gathers on (state, top two stack symbols). The depth cap changes the language: the engine enforces $L(\cA)\cap\{\text{nesting depth}\le64\}$, and a configuration left with no admissible token stops with an explicit error rather than emit an invalid one. Forced chains cross calls and returns: after a node's last value, \texttt{]\},\{"name":"} is one span. The proof of Theorem~\ref{thm:pda} and the sound byte-level budget for pushdowns are in Appendix~\ref{app:pda}.

\subsection{Independent parts in parallel, certified}
\label{sec:blocks}
\begin{lemma}[Concatenation at cut states]
\label{lem:cut}
Let $\cA$ be a trim DFA and $c_1,\dots,c_m$ states such that every accepting run from $c_0$ visits $c_1,\dots,c_m$ in this order, each exactly once. Then every $w\in L(\cA)$ factors uniquely as $w=x_0x_1\cdots x_m$ with $\ds(c_i,x_i)=c_{i+1}$ (and $c_{m+1}\in F$), and every such concatenation is in $L(\cA)$.
\end{lemma}
For an object schema whose keys are all required, every key is a forced chain, so the state after each key is a cut state, and the language is a product of one value language per field: the automaton analogue of Theorem~\ref{thm:block}. The model's distribution over the values is \emph{not} a product. The decoder therefore uses the factorization only to draft, and certifies the draft.

\begin{theorem}[Certified block decoding]
\label{thm:cert}
Fix a policy $\varphi$ and greedy decoding. Draft each field independently, from the configuration at which the sequential decoder would begin it; assemble the draft $D$ and replay it through $\cA$ with $\varphi$, obtaining the configuration $q_i$ before every sampled position $i$; compute the logits at every position of $D$ in one pass; and accept position $i$ iff $D_i$ is the masked argmax at $q_i$. Then, in exact arithmetic, the longest accepted prefix of $D$, extended by the argmax at the first rejected position, is a prefix of the sequential greedy output. Repeating from that prefix produces exactly the sequential output, and each round adds at least one sampled token.
\end{theorem}

The theorem is an instance of speculative verification~\cite{leviathan2023,chen2023}, with drafts that the grammar's factorization makes independent. Two caveats are stated with it. In floating point, the certification pass computes logits over many positions at once, in fp16 and with chunked recurrent kernels, so a near-tie could flip; all 36 measured outputs were token-identical. And a round costs the longest draft plus one certification pass, so in the worst case, one new token per round, it is slower than sequential decoding. A cap on rounds followed by sequential decoding bounds this; our implementation stops at 24 rounds, a cap never approached (at most 3 rounds on the 2B and 4 on the 4B). Algorithms and the related \emph{finite-region} decision are in Appendix~\ref{app:mech}.

\begin{table}[t]
\centering
\small
\caption{\textbf{What each mechanism rests on.} Every mechanism is a construction of Section~\ref{sec:theory} applied to the step graph of a decoding step.}
\label{tab:usage}
\begin{tabular}{@{}L{3.3cm}L{3.6cm}L{6.2cm}@{}}
\toprule
Mechanism & Result & Guarantee \\
\midrule
one-pass token table & Thm.~\ref{thm:table}(i), Lemma~\ref{lem:universal} & exact table in $\mathrm{depth}(\TV)$ gathers, work within $2.4\times$ of the output size \\
walk kernel & Thm.~\ref{thm:table}(ii), Thm.~\ref{thm:reachul} & the table's row recomputed each step; memory = automaton + shared vocabulary \\
forced spans & Thm.~\ref{thm:pass}, Cor.~\ref{cor:length} & same outputs, minimum passes among non-speculative decoders \\
canonical policy & Thm.~\ref{thm:canon} + Assumption~\ref{ass:local} & committed tokens are the tokenizer's split under every continuation \\
budget mask & Thm.~\ref{thm:budget} & valid under every feasible limit; least restrictive such mask \\
pushdown constraint & Thm.~\ref{thm:pda}; the machine of Thm.~\ref{thm:logdcfl} & exact up to nesting depth 64; table indexed by top two stack symbols \\
certified blocks & Lemma~\ref{lem:cut}, Thm.~\ref{thm:cert} & greedy output in exact arithmetic (36/36 identical in fp16) \\
\bottomrule
\end{tabular}
\end{table}

\subsection*{Proofs}


\subsection{Theorem~\ref{thm:table}: the one-pass table and the walk}
\label{app:table}
\begin{algorithm}[H]
\caption{All-state token table: all $|Q|$ states side by side, one trie level per step. For each trie node at depth $d$, $\mathrm{parent}_d$ gives its parent and $\mathrm{byte}_d$ its last byte; $\mathrm{tok}_d$ lists the tokens ending at depth $d$.}
\label{alg:table}
\small
\begin{algorithmic}[1]
\State $S_0\gets[\,0,1,\dots,|Q|-1\,]$ \Comment{one row, the root; every state is a root}
\For{$d=1,2,\dots,\mathrm{depth}(\TV)$}
  \State rows $\gets$ nodes $w$ at depth $d$ whose parent row survived at depth $d-1$
  \State $S_d[w,\cdot]\gets\delta\bigl[S_{d-1}[\mathrm{parent}_d(w),\cdot],\ \mathrm{byte}_d(w)\bigr]$ for all rows \Comment{one gather; $\dead$ stays $\dead$}
  \State $T[t,\cdot]\gets S_d[\mathrm{node}(t),\cdot]$ for every $t\in\mathrm{tok}_d$
  \State drop the rows with $S_d[w,\cdot]\equiv\dead$
\EndFor
\end{algorithmic}
\end{algorithm}
\begin{proof}[Proof of Theorem~\ref{thm:table}]
(i) By induction on $d$, row $w$ of $S_d$ holds $(\ds(q,w))_{q\in Q}$: for $d=0$ it is the identity, and for $d>0$, $\ds(q,w)=\delta(\ds(q,\mathrm{parent}(w)),\mathrm{byte}(w))$ with $\dead$ absorbing. A row that is $\dead$ for every state has only $\dead$ descendants, so dropping it loses no token. Each level processes at most $|\TV^d|$ rows of $|Q|$ entries, so the work is at most $|\TV|\cdot|Q|$. The output has $|\V|\cdot|Q|$ entries; mapping every non-root trie node $w$ to a pair $(t,|w|)$ with $w$ a prefix of $t$ is injective, so $|\TV|-1\le\sum_t|t|$, the cost of walking every token separately. (ii) The walk of $t$ from $q$ reads the bytes of $t$ through $\delta$; it is the root-to-$t$ path of $P_q$ (Proposition~\ref{prop:mangrove}) and uses no other token's walk.
\end{proof}
On Qwen3.5's vocabulary and the 104-state \texttt{person} schema, the pass performs 61M lookups in 0.3\,s and touches the trie essentially once, because a schema with free-text fields keeps almost every trie node alive for some state (Fig.~\ref{fig:liverows}).

\begin{figure}[H]
\centering
\begin{tikzpicture}
\begin{axis}[width=0.72\textwidth, height=4.4cm, ybar, bar width=5pt, ymin=0, ymax=100000, scaled y ticks=false, xlabel={trie depth (bytes)}, ylabel={rows at this level},
  xtick={1,2,4,6,8,10,12,14,16,18,20}, ytick={0,25000,50000,75000,100000}, yticklabels={0,25k,50k,75k,100k},
  legend style={at={(0.98,0.98)}, anchor=north east}, legend columns=1]
\addplot[fill=deemph, draw=none, bar shift=-2.8pt] coordinates {(1,256)(2,4647)(3,24724)(4,50328)(5,82384)(6,94767)(7,68187)(8,61250)(9,55185)(10,33854)(11,27849)(12,22811)(13,12915)(14,9451)(15,8972)(16,5481)(17,5159)(18,4303)(19,3292)(20,2543)};
\addplot[fill=ours, draw=none, bar shift=2.8pt] coordinates {(1,210)(2,4579)(3,24149)(4,49492)(5,81631)(6,94213)(7,67783)(8,61020)(9,55046)(10,33760)(11,27779)(12,22764)(13,12884)(14,9420)(15,8951)(16,5462)(17,5140)(18,4283)(19,3278)(20,2530)};
\legend{trie nodes at this depth, rows still alive in the pass}
\end{axis}
\end{tikzpicture}
\caption{\textbf{The all-state pass, level by level} (\texttt{person} schema, 104 states). Gray: trie nodes of the 248k-token vocabulary by byte depth. Blue: rows Algorithm~\ref{alg:table} keeps alive at that level (a row dies only when it is dead for all states). Depths beyond 20 hold under 1\% of the nodes.}
\label{fig:liverows}
\end{figure}
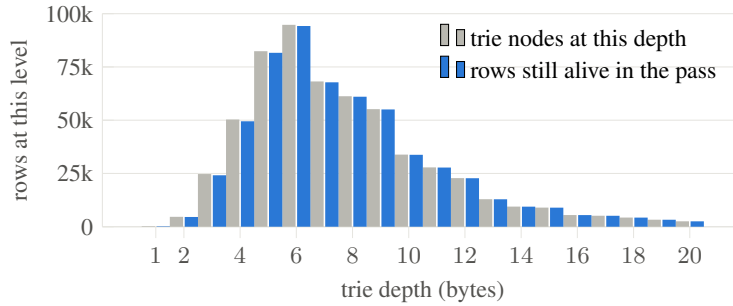

\subsection{Theorem~\ref{thm:pass}: pass optimality}
\begin{proof}
(i) Both decoders begin at the same sampling configuration with identical caches. Suppose this holds at the $k$-th sampling point. Both evaluate the same logits and, with the same randomness, draw the same $t$. $R_\varphi$ then feeds $t$ and $\varphi(c')$ in $1+|\varphi(c')|$ passes and $F_\varphi$ in one; by causality the resulting caches and last-position logits coincide, and the next sampling configuration $\ds(c',\varphi(c'))$ is the same because $\varphi$ does not depend on the model. (ii) is counting.

(iii) Fix a model $p$ and an output $y$ of positive probability, and let $i_1<\dots<i_N$ be the positions at which $R_\varphi$ samples on $y$. \emph{Step 1: $D$ evaluates $\ell(y_{<i_k})$ before it emits $y_{i_k}$.} Otherwise, with positive probability, $D$ emits $y_{i_k}$ after $y_{<i_k}$ without that evaluation. Let $p'$ agree with $p$ on every prefix except $y_{<i_k}$, where it changes the probability of $y_{i_k}$ among the at least two admissible tokens (under softmax sampling any change of the relative logits does this; under greedy decoding, make another admissible token the argmax). On that event $D$ behaves identically under $p$ and $p'$, while $R_\varphi$'s conditional distribution differs, contradicting exactness for both models. \emph{Step 2: the evaluations happen in distinct passes.} Let $\Pi_k$ be the pass that yields $\ell(y_{<i_k})$. $D$ emits $y_{i_k}$ only after $\Pi_k$ completes and, being non-speculative, feeds it only in a later pass; $\Pi_{k+1}$ computes logits at a prefix that contains $y_{i_k}$, so it is or follows a pass that feeds $y_{i_k}$. Hence $\Pi_{k+1}$ starts after $\Pi_k$ completes, and $\Pi_1,\dots,\Pi_N$ are distinct.
\end{proof}

\subsection{Theorem~\ref{thm:canon}: canonical commitment}
\begin{proof}
For a string $s$ let $e(s)$ be the largest block boundary of $\pi(s)$ that is at most $m(s)$; the blocks of $\pi(s)$ ending at or before $m(s)$ are exactly those that tile $[0,e(s))$. \emph{Claim: these blocks are blocks of $\pi(sa)$, and $e(sa)\ge e(s)$.} The first part is Assumption~\ref{ass:local}. For the second, $e(s)$ is a boundary of $\pi(sa)$; the last block of $\pi(sa)$ contains position $|s|$, so it is not one of the preserved blocks and starts at or after $e(s)$; the trailing whitespace run of $sa$ starts at $|s|+1$ if $a$ is not whitespace, and otherwise at $|s|$ or at the start of the trailing run of $s$, which is at least $m(s)\ge e(s)$. Hence $m(sa)\ge e(s)$ and $e(s)$ is a boundary at most $m(sa)$.
(i) follows by induction on $|x|$, since $e$ never decreases and the blocks tiling $[0,e(o))$ are preserved at every step; as $\tok$ encodes blocks independently, their canonical tokens are a prefix of $\tok(ox)$.
(ii) Let $x=bx'$ with $b\in\live(c)$ (every $x$ with $\ds(c,x)$ defined and $x\ne\varepsilon$ starts with such a byte). By hypothesis $\pi(o)$ is a prefix of $\pi(ob)$, so the last block of $\pi(ob)$ is $[|o|,|o|+1)$; since $o$ does not end in whitespace, $m(ob)=|o|$, every block of $\pi(o)$ ends at or before $m(ob)$, and (i) applied to $ob$ and $x'$ concludes.
\end{proof}
We identify a byte with the characters it can begin; the implementation tests every ASCII byte and one representative per class of non-ASCII lead bytes. For $o$ not ending in whitespace, ``$\pi(ob)$ has a boundary at $|o|$'' was equivalent to ``$\pi(o)$ is a prefix of $\pi(ob)$'' in all 1.44M tested pairs, and the implementation tests the former. The whitespace clause of Assumption~\ref{ass:local} is necessary: appending \texttt{\textbackslash n} to a string ending in \texttt{\textbackslash n\textvisiblespace} merges its last two blocks.

\subsection{Theorem~\ref{thm:budget}: the least restrictive valid mask}
\begin{proof}
(i) The definition of $\need$ is the Bellman equation of a shortest-path problem with positive integer weights $1+|\varphi(c')|$ on the token-step graph of the $|Q|$ sampling configurations; its least solution is the distance to $F$, reached within $|Q|$ rounds, each a gather over the transition table and a row-wise minimum.
(ii) Invariant: $\need(c)\le r$ at every sampling point. It holds at $c_1$. If $c\in F$ and $r=0$, the output ends validly. If $c\notin F$, some $t$ has $\cost(c,t)=\need(c)\le r$, so $M_r(c)\ne\emptyset$; for any admitted $t$ the forced span fits, because $1+|\varphi(c')|\le\cost(c,t)\le r$, and at the next sampling point $c''$, $r''=r-1-|\varphi(c')|\ge\cost(c,t)-1-|\varphi(c')|=\need(c'')$. If $r$ reaches $0$ at a sampling point, the invariant gives $\need(c)=0$, i.e.\ $c\in F$; an output that ends by \eos{} ends in $F$ by definition of $A(c)$.
(iii) If $\cost(c,t)\le r$, following minimizing tokens realizes a continuation of weight $\cost(c,t)$ that reaches $F$; conversely every continuation after $t$ that reaches $F$ has weight at least $\cost(c,t)$. If a mask admits $t$ with $\cost(c,t)>r$, a model that greedily prefers $t$ at $c$ leaves no continuation reaching $F$ within the budget, and the output is cut outside $F$.
\end{proof}

\subsection{Theorem~\ref{thm:pda} and budgets for pushdowns}
\label{app:pda}
\textbf{The machine.} Definitions on a cycle of the schema's reference graph (strongly connected components) become nonterminals $N_1,\dots,N_C$; every other reference is inlined. Each $N_j$ is a DFA over $\Sigma\cup\{\mathrm{call}_1,\dots,\mathrm{call}_C\}$; a call edge $q\xrightarrow{\mathrm{call}_j}r$ records its return state $r$, and $R$ is the set of return states. A configuration is $(q,\sigma)$ with $\sigma\in R^*$. On byte $b$ the machine (1)~moves if $\delta(q,b)$ is defined; else (2)~if $q$ has a call edge to $N_j$ with $b\in\mathrm{first}(N_j)$, pushes $r$ and moves to $\delta(\mathrm{start}_j,b)$; else (3)~if $q$ accepts and $\sigma=\sigma'r$, pops and retries $b$ at $r$ (a lazy return); and (4)~after a move, while the state accepts and has no successor, pops and continues at the popped state (an eager return; with an empty stack this is the end of the output). The compiler checks $\mathrm{first}(N_j)\cap\live(q)=\emptyset$ at every call site and $\live(q)\cap\mathrm{follow}(N_j)=\emptyset$ at every accepting state of a called $N_j$ that has successors (an LL(1) condition), refuses grammars that violate them, and requires return states not to accept, so that the end of output depends only on the state and an empty stack.

\begin{proof}[Proof of Theorem~\ref{thm:pda}]
Simulate the run of $t$ from $q$ with the initial stack unknown, tracking the symbols that $t$ itself pushes. Every move reads the state and, for a return, the top symbol. While the top symbol was pushed by $t$, it is known. A pop that reaches below these symbols is real. If no further move follows it, its symbol enters the result only as ``continue at the popped state''. If another move follows, the run branches on the popped symbol's value, through the return state reached and its moves. Each exposed real pop consumes one symbol of the original stack and is followed by at least one byte or pop of $t$, so the branching depth is at most $|t|$. The leaves record the final state, the number of real pops and the pushed string.
\end{proof}

\begin{proposition}[Sound budgets for pushdowns]
\label{prop:pdabudget}
Let $\mathrm{bn}(q)$ be the length of the shortest byte string that completes the current nonterminal from $q$: the least fixpoint of $\mathrm{bn}(q)=0$ for accepting $q$ and otherwise $\min\{1+\mathrm{bn}(\delta(q,b)),\ \mathrm{bn}(\mathrm{start}_j)+\mathrm{bn}(r)\}$ over bytes and call edges $q\xrightarrow{\mathrm{call}_j}r$. Then $\need^+(q,\sigma)=\mathrm{bn}(q)+\sum_{r\in\sigma}\mathrm{bn}(r)$ is the length of the shortest byte completion of $(q,\sigma)$. If every single byte is a token, admitting a token only if $1+\need^+(c')\le r$ for the configuration $c'$ after it guarantees validity in the sense of Theorem~\ref{thm:budget}(ii).
\end{proposition}
\begin{proof}
Every completion of $(q,\sigma)$ decomposes into a string taking $q$ to an accepting state of its nonterminal, the return to the top of $\sigma$, a completion of that return state's nonterminal, and so on; the LL(1) conditions make the pops occur exactly at these points. Each part has length at least $\mathrm{bn}$ of its state, and concatenating shortest parts is a valid completion, so $\need^+$ is the exact shortest byte completion. Every token and every forced span covers at least one byte, and forced bytes lie on every completion. A token admitted by $1+\need^+(c')\le r$ therefore leaves $r-1-f\ge\need^+(c')-f'=\need^+(c'')$, where $f$ forced tokens cover $f'\ge f$ bytes; the single-byte token that begins a shortest completion is always admitted; the invariant $\need^+\le r$ follows as in Theorem~\ref{thm:budget}(ii).
\end{proof}
The bound is sound but not least restrictive, because a token may cover several bytes. The device state vector carries the running sum $\sum_{r\in\sigma}\mathrm{bn}(r)$, so the mask reads one extra column.

\subsection{Lemma~\ref{lem:cut} and Theorem~\ref{thm:cert}: certified blocks}
\begin{proof}[Proof of Lemma~\ref{lem:cut}]
The run of $w\in L(\cA)$ is unique and visits $c_1,\dots,c_m$ in order, once each; cutting it at those visits gives the factorization, which is unique because the run is. Conversely, the runs on $x_0,\dots,x_m$ concatenate to an accepting run.
\end{proof}
\begin{proof}[Proof of Theorem~\ref{thm:cert}]
Let $Y$ be the sequential greedy output under $\varphi$. Between sampled positions both $Y$ and the replayed draft $D$ consist of the forced tokens $\varphi$ prescribes, which depend only on the configuration. Let $i^*$ be the first sampled position of $D$ that is not accepted. By induction over the sampled positions before $i^*$: at each, the replay's configuration equals the sequential decoder's, the prefixes agree, so the logits computed by the certification pass at that position equal the sequential decoder's logits in exact arithmetic (Theorem~\ref{thm:pass}(i)), and acceptance means $D$'s token is the sequential choice. Hence $D_{<i^*}=Y_{<i^*}$, and the masked argmax at $i^*$ is $Y_{i^*}$. Each round therefore extends the certified prefix by at least one sampled token, and repeating, or finishing sequentially after a cap on rounds, yields $Y$.
\end{proof}

\section{Mechanisms and engine as implemented}
\label{app:mech}

\subsection{Constraint compilation}
A JSON schema is turned into a regular expression by outlines-core; a regex constraint is used as given. Our parser builds a Thompson NFA over bytes (a character class over code points is expanded into UTF-8 byte sequences), determinizes it by subset construction with a state budget (200,000 by default; beyond it compilation stops with an error), trims states from which no accepting state is reachable, and minimizes by Moore refinement. Lookaround and backreferences are rejected. The vocabulary's byte strings are recovered from the tokenizer; special tokens have no bytes and are never allowed by a constraint; added tokens that are not special (Qwen's \texttt{<tool\_call>}, \texttt{</think>}) receive their surface bytes, so a constraint that expects that text can allow the single token. The trie is stored level by level (arrays $\mathrm{parent}_d$, $\mathrm{byte}_d$ and the tokens ending at each node).

\subsection{The walk kernel}
The automaton is kept on the device as $256\cdot|Q|$ int16 entries ($\dead=-1$), with the accepting vector and an end-of-sequence vector over the vocabulary; the vocabulary is packed once (every token's bytes concatenated, with offsets: 1.5\,MB plus 1\,MB of offsets for the 248k tokens, the shared 2.9\,MB with the rest). One Metal kernel (\texttt{mx.fast.metal\_kernel}, grid $|\V|\times B$, threadgroup 256) runs per step: thread $t$ of row $b$ reads state $q_b$, walks its token's bytes (\texttt{s = table[s*256 + byte]} until the end or $\dead$) and writes the state reached. The mask is $\text{row}\ne\dead$ or ($q$ accepting and the token is \eos), combined with the boundary cover ($\mathrm{len}(t)\ge\mathrm{cover}[q]$); the next state is $\text{row}[y]$. The host side (\texttt{accept}, \texttt{forced}, \texttt{finished}) runs the automaton over the token's bytes and uses the same planner as the table controller. States with at most 4,096 allowed tokens keep a sparse output head: their allowed ids are computed once on the CPU (about 4\,ms) and stored (at most 16\,KB per state). The batch engine builds one combined automaton table (block $i$ is constraint $i$'s table with next states offset to global ids; global state 0 is the free state) and launches one kernel over all rows. Healing is off with the walk kernel, as it is in practice with the table (its threshold exceeds 200 tokens on our schemas).

\subsection{The forced planner (the canonical policy)}
Given the previous sampled token $y$ and the configuration $q$ after it, the planner follows the unique live byte while the configuration is forced (the forced string $w$), and then: (1)~cuts $s=\mathrm{bytes}(y)\,w$ to its longest valid UTF-8 prefix; (2)~computes the pre-token block starts of $s$ and holds back the last block, unless all of $w$ survived, $s$ does not end in whitespace, and for every byte allowed after $w$ (every ASCII byte and one representative per non-ASCII class) the pre-tokenizer puts a boundary at $|s|$ (the stable-tail test of Theorem~\ref{thm:canon}(ii)); (3)~if the end of $y$ is not a block boundary of $s$, re-encodes $y$ together with $w$ and replaces $y$ before it is fed (\emph{healing}; no cache rollback is needed in the synchronous loop); (4)~tokenizes the committed segment with the model's tokenizer and checks that the tokens spell it; (5)~returns the tokens, whether $y$ is replaced, the end configuration and the \emph{cover} of the end configuration: if the remaining forced bytes form one pre-token block that is a single vocabulary token of length $\ell$, the next token must have at least $\ell$ bytes. The cover is a per-state table on the device.

\subsection{The pushdown constraint}
Each nonterminal's DFA has $256+C$ columns (a call of $N_i$ is the private code point $\mathrm{U}{+}200000+i$, mapped to column $256+i$); states are numbered globally. The token tables are the level pass of Algorithm~\ref{alg:table} for all states at once, with rows that carry a symbolic stack: pushes inside a token are recorded (at most three) and matched by pops inside it; a real pop splits the row into one row per candidate return state of the popped nonterminal (plus end of output for the root); at most two real pops are resolved and deeper ones masked. A token whose walk makes no real pop, or ends exactly at its single real pop, gets a context-free entry (next state, pushes, pops) in the $|\V|\times|Q|$ table; a token that continues past a real pop, or pops twice, lives in a context table of shape $n_{\mathrm{ctx}}\times n_{qc}\times(|R|+1)^2$ indexed by the top two stack symbols. The device state vector is $[q,\mathrm{ptr},\mathrm{stack}[0..63],S]$, with $S=\sum_{r\in\sigma}\mathrm{bn}(r)$ for budgets; masking and advancing are gathers on $(q,\mathrm{top}_1,\mathrm{top}_2)$. For the schema of a tree of named nodes the root rule is \texttt{\{"name":}$S$\texttt{,"children":[}$(\mathrm{call}_0(\texttt{,}\mathrm{call}_0)^*)^?$\texttt{]\}} with $S$ the JSON-string regex: 37 states, one call site and 6 context-dependent tokens at 4 states on Qwen's vocabulary.

\subsection{Certified block-parallel decoding}
\textbf{Plan.} The planner accepts an object schema whose properties are all required, flattening nested objects of the same shape into their leaves; arrays and free strings stay single blocks. Two dummy documents with different placeholder values are rendered, recording for every leaf the byte offsets after its key and at the next key; running the automaton over each prefix gives the state after the key name and the exit state, which must agree between the renderings (the automaton is minimal, so the language after a value does not depend on the value). A block starts where the sequential decoder starts sampling, i.e.\ after the committed part of the key's chain; its context is the bare key path (\texttt{\{"parent":\{"key}).
\textbf{Wave.} Every block is a row of the batch engine, started from the prompt's cache (array handles shared, recurrent state copied), with a controller that stops when the state leaves the block. Rows run in \emph{ragged} mode (a sampled token and its forced span in one step). First-wave budget per block: 24 steps, or $1.5\times$ the block's mean length $+2$ once the schema has been seen.
\textbf{Certificate.} The samples are replayed block by block through a fresh controller; the replay yields the token sequence $D$, the state before every sampled token and the block of each position, and detects three inconsistencies (a block reached at a state other than its branch's start, a block cut by its budget, a previous token that carried past the block's start). One pass over $D$ from a copy of the prompt's cache gives the logits at every position; the masked argmax (allowed row and cover at the recorded state) is compared with $D_i$. Every failing block continues in one wave from its first failing position with the correct token, its prefix re-prefilled from the prompt's cache; the loop repeats (at most 24 rounds).
\textbf{What was tried and rejected.} Earlier values as placeholders in the branch context (the model copied them), an ellipsis placeholder (copied), type names (copied); no initial plan at the block start (nothing certified); no first-wave cap (a runaway number such as \texttt{60.0000\ldots} cost the whole wave); certifying one mismatch per pass (one wave per failing block).
\textbf{Finite regions.} An enum, a boolean or a tool name is a finite region: a depth-first search over the automaton lists every completion up to a budget (raising an error on a cycle), and all completions are scored after the same prefix in one right-padded batched pass; the decision is the completion with the largest summed log-probability. It is exact for what it computes; Appendix~\ref{app:neg} reports why it is not a better classifier.

\subsection{Rows on demand (host-side variant)}
\texttt{LazyConstraint} keeps the automaton, the forced chains and the accepting information ($O(256|Q|)$) and a bounded dictionary of rows, each computed on first use by the level pass restricted to one root (about 4\,ms per state on the CPU, since a free-text state keeps nearly every trie node alive). Memory is proportional to the visited states: 17 rows (13\,MB) instead of the 824\,MB table for the typed tool grammar. It runs on the synchronous engine only, because the pipelined engine builds the next mask before the host knows the sampled token; the walk kernel supersedes it.

\subsection{The engine}
\textbf{Kernels.} 4-bit matrix-vector kernels for 1--16 rows dequantize weights straight into $8\times8$ simdgroup-matrix fragments (a fixed permutation inside each 128-wide chunk lets a lane find its 32 weights in one 128-bit load); a half-precision variant converts two nibbles with one shift, one mask and one fused multiply--add; a two-level autotuner picks, per shape and then per whole model step timed pipelined, among these families and MLX's kernels (first load on the 4B: about 19 minutes; later loads 3\,s). One fused Metal kernel per layer performs the whole Gated DeltaNet decode step (convolution, norms, gates, the delta-rule recurrence on the $128\times128$ state, the gated output norm) for up to 16 tokens. M1/M2 GPUs have no native bf16 arithmetic, so bf16 checkpoints run in fp16, which is also closer to float32 (Table~\ref{tab:fidelity}). From 128 rows the prompt pass uses MLX's dense matrix multiplication.
\textbf{Pipelining.} A step is built lazily (logits, mask, sampled token $y$, next state $q'=T[y,q]$), and step $n+1$ is enqueued with $y$ and $q'$ as unevaluated device values while the GPU runs step $n$; the host's graph construction (about 1.2\,ms) and command encoding (1.8\,ms) overlap the GPU's 7\,ms. When a forced span may follow $q'$, the next step is built on a functional cache and discarded if the host, on reading $y$, finds a span, which is then fed with $y$ in one pass. Outputs are identical to the synchronous loop.
\textbf{Batching.} Streams join and leave a running batch at any step; every row feeds exactly one token per step, its next forced token or the token it sampled (\emph{drip} mode), so each step is a uniform batch whose inputs are known without the host; new prompts are prefilled in one right-padded batch; a radix-trie prompt cache serves exact prefix hits (the recurrent state exists only at the end of a stored stream).
\textbf{Sampling.} Greedy, or temperature with top-$p$ and top-$k$ under the same rule as mlx-lm's sampler, so that sampled runs of the two engines follow the same distribution.

\subsection{The local tool: server, reply grammar and agents}
The engine is packaged as an OpenAI-compatible server (\texttt{/v1/chat/completions} with tools and \texttt{json\_schema} response formats), serving our engine and stock mlx-lm with llguidance from one process, and as a command that configures the OpenCode coding agent (a team of orchestrator, coder, researcher, tester and reviewer agents driven through one \texttt{opencode serve}). Every request is a row of the running batch. The server re-renders an assistant turn that carried tool calls verbatim, so the whole previous turn is served from the prompt cache (647 of 675 prompt tokens on a tool round trip).
A reply to a request that offers tools decodes under a regular language: free text that cannot start a call, then zero or more calls:
\begin{align*}
\mathrm{REPLY}&=\mathrm{FREE}\,\bigl(\mathrm{CALL}\,(\hookleftarrow\mathrm{CALL})^*\bigr)^?,\\
\mathrm{CALL}&=\texttt{<tool\_call>}\hookleftarrow\texttt{<function=}\mathrm{NAME}\texttt{>}\hookleftarrow\,\mathrm{PARAM}^*\\&\qquad\texttt{</function>}\hookleftarrow\texttt{</tool\_call>},\\
\mathrm{PARAM}&=\texttt{<parameter=}P\texttt{>}\hookleftarrow\mathrm{VAL}\hookleftarrow\texttt{</parameter>}\hookleftarrow,
\end{align*}
with $\mathrm{NAME}$ an alternation of the offered tool names, $\mathrm{VAL}$ any text not containing \texttt{</parameter>} and $\mathrm{FREE}$ any text containing neither \texttt{<tool\_call>} nor \texttt{<function=}. In the \emph{typed} grammar every tool's parameter names and value types are fixed (1,106 states on Qwen's vocabulary), which only the walk kernel can afford.
\begin{proposition}[Exact ``not containing'' expression]
\label{prop:notcontain}
Let $\ell_1,\dots,\ell_m$ be strings that start with \texttt{<} and contain no other \texttt{<}, and $\rho_i$ be $\ell_i$ without its first character. A string contains no $\ell_i$ iff no segment that follows a \texttt{<} (segments being the maximal \texttt{<}-free substrings) starts with some $\rho_i$. In the trie of the $\rho_i$, such a segment either ends at a node $p$ that is not a complete $\rho_i$ or leaves the trie at $p$ with a character that is not a child of $p$. Hence
$\mathrm{NotContaining}(\ell_1,\dots,\ell_m)=\texttt{[\textasciicircum<]*}\bigl(\texttt{<}(\bigcup_p p\,\texttt{[\textasciicircum<}c_1\cdots c_j\texttt{][\textasciicircum<]*}\cup\bigcup_p p)\bigr)^*$, where $p$ ranges over trie nodes that are not complete $\rho_i$ and have no complete $\rho_i$ as a prefix, and $c_1\cdots c_j$ are the children of $p$.
\end{proposition}
The construction is tested against brute force on all strings up to length 6 and on 2,000 random strings per literal set. Every tag of a call is a forced span (20--30\% of a call's tokens are fed, not sampled), and a reply cut by the token budget still closes its call. For small models the server cancels a reply that exceeds a number of complete calls, repeats a call, or repeats its last $n$ tokens three times ($n\in\{8,\dots,128\}$), and steers a turn that repeats the previous turn's call.

\section{Speed: all measurements}
\label{app:speed}
\textbf{Setup.} As in Section~\ref{sec:exp}. Qwen3.5-2B is 1.1\,GB resident and Qwen3.5-4B 2.26\,GB. \emph{Stock} means mlx-lm as shipped, in its own process, with its own kernels and bf16 arithmetic. Every number names the script that produced it; the result files ship with the paper.

\begin{figure}[H]
\centering
\begin{tikzpicture}
\begin{axis}[width=0.6\textwidth, height=4.4cm, xlabel={tokens in the pass, $S$}, ylabel={$\tau(S)/\tau(1)$}, xmin=0.5, xmax=17.5, ymin=0, ymax=6, xtick={1,4,8,12,16}, legend style={at={(0.02,0.98)}, anchor=north west}]
\addplot[deemph, thick, dashed, domain=1:6] {x};
\addplot[ours, thick, mark=*, mark size=1.3pt] coordinates {(1,1.00)(2,1.12)(3,1.21)(4,1.36)(5,1.58)(6,1.59)(7,1.65)(8,1.58)(9,1.99)(10,1.96)(11,2.01)(12,2.00)(13,1.99)(14,2.01)(15,2.05)(16,2.04)(17,2.53)};
\legend{if tokens cost linearly, {measured (2B, one stream)}}
\end{axis}
\end{tikzpicture}
\caption{\textbf{What a wider pass costs} (Qwen3.5-2B, one stream). A pass over $S$ tokens costs far less than $S$ passes, because it is bound by reading the weights. On the 4B a pass costs 13\,ms at $S=1$, 21\,ms at $S=4$ and 39\,ms at $S=16$.}
\label{fig:passcost}
\end{figure}
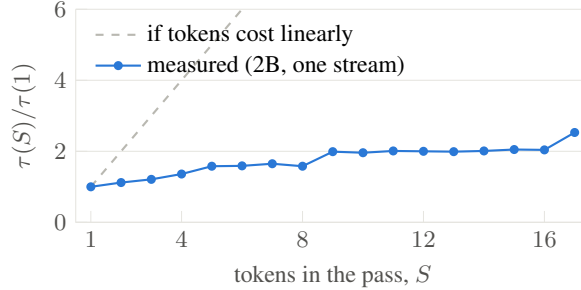

\subsection{One stream under a constraint}
\begin{table}[H]
\centering
\small
\caption{\textbf{Constrained decoding, one stream}: tokens per second (counting forced tokens), passes per output token, share of output tokens fed as forced spans, and the derived time per sampled token. \emph{ff}: llguidance's fast-forward. Six JSON schemas and five recursive schemas.}
\label{tab:constrained}
\footnotesize
\setlength{\tabcolsep}{3.5pt}
\begin{tabular}{@{}lrrrr rrrr@{}}
\toprule
 & \multicolumn{4}{c}{Qwen3.5-2B} & \multicolumn{4}{c}{Qwen3.5-4B} \\
\cmidrule(lr){2-5}\cmidrule(lr){6-9}
mode & tok/s & passes/tok & forced & ms/sampled & tok/s & passes/tok & forced & ms/sampled \\
\midrule
free decoding & 138 & 1.01 & 0\% & 7.2 & 63.2 & 1.01 & 0\% & 15.7 \\
JSON schema, \stcon{} & 180 & 0.69 & 31\% & 8.0 & 82.3 & 0.69 & 31\% & 17.6 \\
JSON schema, llguidance ff & 136 & 0.72 & 28\% & 10.2 & 68.6 & 0.73 & 27\% & 20.0 \\
recursive, \stcon{} & 170 & 0.72 & 28\% & 8.2 & 78.6 & 0.70 & 30\% & 18.2 \\
recursive, llguidance ff & 143 & 0.71 & 29\% & 9.9 & 67.4 & 0.72 & 28\% & 20.6 \\
recursive, llguidance & 94 & 1.00 & 0\% & 10.6 & & & & \\
\bottomrule
\end{tabular}
\end{table}

\begin{table}[H]
\centering
\scriptsize
\caption{\textbf{\stcon{} and llguidance in one loop, full} (max 256 tokens, medians of three interleaved rounds; 2B / 4B). For llguidance's fast-forward the better of the plain and the overlapped integration (CPU mask computed while the GPU runs) is shown; the overlap gained at most 2\%. Decode time is the sum over the tasks' outputs.}
\label{tab:three-full}
\begin{tabular}{@{}llrrrrrr@{}}
\toprule
path & engine & tok/s & passes/tok & forced & output tokens & decode time & ms/sampled \\
\midrule
regular & free decoding & 137 / 66 & 1.01 / 1.01 & 0 / 0\% & 570 / 597 & 4.15 / 9.02\,s & 7.2 / 15.0 \\
(6 schemas) & \stcon{} & 173 / 85 & 0.69 / 0.69 & 31 / 31\% & 327 / 329 & 1.89 / 3.89\,s & 8.3 / 17.1 \\
 & llguidance, mask only & 103 / 55 & 1.00 / 1.00 & 0 / 0\% & 343 / 333 & 3.34 / 6.08\,s & 9.7 / 18.3 \\
 & llguidance, fast-forward & 143 / 73 & 0.72 / 0.73 & 28 / 27\% & 343 / 333 & 2.41 / 4.59\,s & 9.8 / 18.8 \\
\midrule
recursive & free decoding & 135 / 67 & 1.01 / 1.00 & 0 / 0\% & 592 / 845 & 4.39 / 12.57\,s & 7.4 / 14.8 \\
(4 schemas) & \stcon{} (pushdown) & 161 / 84 & 0.72 / 0.70 & 28 / 30\% & 469 / 382 & 2.92 / 4.54\,s & 8.7 / 17.0 \\
 & llguidance, mask only & 103 / 56 & 1.00 / 1.00 & 0 / 0\% & 310 / 418 & 3.03 / 7.52\,s & 9.8 / 18.0 \\
 & llguidance, fast-forward & 141 / 73 & 0.71 / 0.73 & 29 / 28\% & 310 / 418 & 2.19 / 5.69\,s & 10.0 / 18.8 \\
\bottomrule
\end{tabular}
\end{table}

\textbf{End to end.} On 18 extraction jobs on the 2B (prefill plus decode), \stcon{} takes 8.0\,s against 12.3\,s for llguidance, and 6.8\,s when the prompt carries only the keys and enum options instead of the schema's syntax, which the forced chains supply (Fig.~\ref{fig:batchjson} gives the batched version); agreement with the unconstrained answer is 0.93 in every mode and all outputs are valid.

\subsection{Recursive schemas in detail}
\begin{table}[H]
\centering
\small
\caption{\textbf{Recursive schemas}: five tasks (a Solar-System tree, an expression tree, a threaded discussion, a directory tree, a taxonomy) $\times$ three phrasings, Qwen3.5-2B, greedy, 400-token limit. ms per token is over valid outputs only.}
\label{tab:recursive}
\begin{tabular}{@{}lcccc@{}}
\toprule
mode & valid & passes & forced tokens & ms / token \\
\midrule
\stcon{} pushdown, canonical policy & 11/15 & 1,452 & 921 & 5.9 \\
\stcon{} pushdown, commit-all & 12/15 & 1,304 & 875 & 5.8 \\
llguidance & 12/15 & 2,204 & 0 & 9.4 \\
llguidance fast-forward & 12/15 & 1,309 & 895 & 7.0 \\
\bottomrule
\end{tabular}
\end{table}
On the expression task every engine fails: the model nests without end; ours stops at the depth cap, llguidance at the token limit. Under the canonical policy one phrasing of the discussion task duplicates replies until the limit; the same policy yields the complete Solar-System tree (all five planets, 92 tokens) where llguidance stops after Earth (46 tokens). Single greedy samples of a 2B model are path-sensitive in both directions. Forced spans cross returns: after a node's last value the whole \texttt{]\},\{"name":"} continuation is one span.

\subsection{Typed tool calls through the server}
On eight concurrent agent requests, two rounds: on the 2B, the typed grammar with the walk kernel produces 46 of 46 calls with every argument present and of its declared type at 366--403 tokens/s (30\% of tokens forced, since parameter names are chains), against 43 of 52 typed calls with the untyped grammar and token tables at 372--410 tokens/s (20\% forced). On the 4B, 8 of 8 calls are typed either way (131 against 139 tokens/s): the typed grammar turns a habit into a guarantee at a cost of 6\%.

\subsection{Certified blocks}
\begin{table}[H]
\centering
\small
\caption{\textbf{Forward passes per output}: six schemas, three prompt variants each, greedy; sequential decoding against certified block-parallel decoding, which produces the same tokens in every case. On the 4B, block passes split into wave steps and certification passes.}
\label{tab:blocks}
\begin{tabular}{@{}l rrr rrr@{}}
\toprule
 & \multicolumn{3}{c}{Qwen3.5-2B} & \multicolumn{3}{c}{Qwen3.5-4B} \\
\cmidrule(lr){2-4}\cmidrule(lr){5-7}
schema & sequential & blocks & exact & sequential & blocks (waves + cert.) & exact \\
\midrule
person & 84 & \textbf{33} & 3/3 & 84 & \textbf{33} (30 + 3) & 3/3 \\
tool\_call & 35 & \textbf{17} & 3/3 & 35 & \textbf{20} (14 + 6) & 3/3 \\
product & 112 & \textbf{53} & 3/3 & 112 & \textbf{46} (40 + 6) & 3/3 \\
event & 169 & \textbf{96} & 3/3 & 167 & \textbf{109} (94 + 10) & 3/3 \\
classification & 126 & 112 & 3/3 & 142 & 159 (148 + 7) & 3/3 \\
contact\_card & 146 & \textbf{63} & 3/3 & 146 & \textbf{117} (108 + 6) & 3/3 \\
\midrule
all 18 outputs & 672 & \textbf{374} & 18/18 & 686 & \textbf{484} & 18/18 \\
rounds (max / median) & & 3 / 2 & & & 4 / 2 & \\
decode wall time & 5.9\,s & 7.4\,s & & 12.2\,s & 18.6\,s & \\
\bottomrule
\end{tabular}
\end{table}
First-round certification accepts 438 of 654 sampled tokens on the 2B and 365 of 668 on the 4B (the larger model commits to earlier values more). On the M2 Pro a wave step over 5--9 rows costs 14--20\,ms against 7\,ms for one row, and a certification pass over 34--50 tokens costs 33--55\,ms. Per job on the 2B (first prompt variant): person 11 against 27 passes (210 against 230\,ms), tool\_call 6 against 12 (125 against 124\,ms), product 19 against 38, event 32 against 56, classification 37 against 43, contact\_card 21 against 49.

\subsection{Throughput by number of streams (engine)}
\begin{table}[H]
\centering
\small
\caption{\textbf{Free decoding, total tokens/s by number of streams} (256 tokens per stream, no constraint). \emph{mlx-lm on ours}: mlx-lm's batch generator running on our patched model in the same process.}
\label{tab:batch}
\begin{tabular}{@{}r rrrr @{\hspace{1.2em}} rrr@{}}
\toprule
 & \multicolumn{4}{c}{Qwen3.5-2B} & \multicolumn{3}{c}{Qwen3.5-4B} \\
\cmidrule(lr){2-5}\cmidrule(lr){6-8}
streams & stock & mlx-lm on ours & \stcon{} & vs stock & stock & \stcon{} & vs stock \\
\midrule
1 & 127 & 140 & \textbf{142} & $1.12\times$ & 61 & \textbf{64} & $1.04\times$ \\
2 & 172 & 251 & \textbf{259} & $1.51\times$ & 82 & \textbf{116} & $1.41\times$ \\
4 & 186 & 341 & \textbf{339} & $1.82\times$ & 87 & \textbf{160} & $1.83\times$ \\
8 & 188 & 575 & \textbf{576} & $3.06\times$ & 87 & \textbf{271} & $3.13\times$ \\
16 & 342 & 751 & \textbf{758} & $2.22\times$ & 157 & \textbf{341} & $2.17\times$ \\
32 & 624 & 844 & \textbf{873} & $1.40\times$ & & & \\
64 & 615 & 837 & \textbf{952} & $1.55\times$ & & & \\
96 & 613 & & \textbf{962} & $1.57\times$ & & & \\
128 & 613 & 830 & \textbf{892} & $1.45\times$ & & & \\
\bottomrule
\end{tabular}
\end{table}
Stock mlx-lm's batch generator stalls between 4 and 8 streams; mlx-lm's own loop on our patched model reaches our numbers, so the gain is in the kernels and fp16, not the loop. With the same prompt on all eight streams the numbers do not move (273 against 273 tokens/s for ours on the 4B), so nothing is shared between streams. At 128 streams the recurrent states hold 2.3\,GB and the wider matrix multiplications cost more than the extra rows return. Single-stream decode on the 2B falls only from 138 to 118 tokens/s as the context grows from 128 to 8,192 tokens (Fig.~\ref{fig:context}), because only 6 of the 24 layers attend over the context; prefill runs at about 1,700 tokens/s.

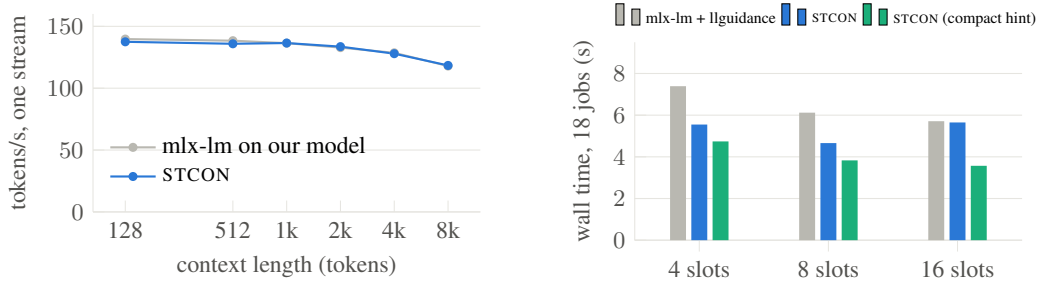
\begin{figure}[H]
\centering
\begin{tikzpicture}
\begin{axis}[width=0.48\textwidth, height=4.2cm, xlabel={context length (tokens)}, ylabel={tokens/s, one stream},
  xmode=log, log basis x=2, xtick={128,512,1024,2048,4096,8192}, xticklabels={128,512,1k,2k,4k,8k}, ymin=0, ymax=160,
  legend style={at={(0.02,0.08)}, anchor=south west}]
\addplot[deemph, thick, mark=*, mark size=1.3pt] coordinates {(128,139.7)(512,138.4)(1024,136.4)(2048,132.8)(4096,128.6)(8192,117.9)};
\addplot[ours, thick, mark=*, mark size=1.3pt] coordinates {(128,137.5)(512,135.9)(1024,136.5)(2048,133.6)(4096,127.9)(8192,118.4)};
\legend{mlx-lm on our model, \stcon{}}
\end{axis}
\end{tikzpicture}\hfill
\begin{tikzpicture}
\begin{axis}[width=0.48\textwidth, height=4.2cm, ybar, bar width=6pt, ymin=0, ymax=9.5, ylabel={wall time, 18 jobs (s)},
  symbolic x coords={4 slots, 8 slots, 16 slots}, xtick=data, enlarge x limits=0.25,
  legend style={at={(0.5,1.02)}, anchor=south}, legend columns=3, legend style={font=\tiny}]
\addplot[fill=deemph, draw=none] coordinates {(4 slots,7.39) (8 slots,6.12) (16 slots,5.71)};
\addplot[fill=ours, draw=none] coordinates {(4 slots,5.55) (8 slots,4.66) (16 slots,5.65)};
\addplot[fill=third, draw=none] coordinates {(4 slots,4.74) (8 slots,3.83) (16 slots,3.57)};
\legend{mlx-lm + llguidance, \stcon{}, \stcon{} (compact hint)}
\end{axis}
\end{tikzpicture}
\caption{\textbf{Left: decode speed against context length} (2B, one stream, 128 output tokens). \textbf{Right: batched constrained serving} (2B, 18 JSON extraction jobs, best of two runs, all outputs valid). With the schema in the prompt our drip engine is $1.33\times$ and $1.31\times$ faster than mlx-lm's batch generator with llguidance at 4 and 8 slots and level at 16, where a step is bound by arithmetic and forced spans buy nothing on a uniform batch. The compact hint puts only keys and enum options in the prompt (118 instead of 212 tokens) and lets forced spans supply the syntax: $1.5$--$1.6\times$ at every slot count.}
\label{fig:context}
\label{fig:batchjson}
\end{figure}

\subsection{Agents (engine, plus the typed grammar)}
\begin{table}[H]
\centering
\small
\caption{\textbf{Parallel agents} (built-in suite of small coding tasks, Qwen3.5-2B, greedy, every reply constrained to one well-formed tool call). Stock: mlx-lm's batch generator with llguidance for the tool-call grammar. Both use continuous batching and the same prompt cache (94--96\% hits).}
\label{tab:agents}
\begin{tabular}{@{}lcccc@{}}
\toprule
& \multicolumn{2}{c}{16 agents} & \multicolumn{2}{c}{8 agents} \\
& stock & \stcon{} & stock & \stcon{} \\
\midrule
wall time & 74.6\,s & \textbf{29.3\,s} ($2.55\times$) & 19.4\,s & 22.1\,s ($0.87\times$) \\
throughput & 125.5 tok/s & \textbf{271.1 tok/s} ($2.16\times$) & 108.6 tok/s & \textbf{170.6 tok/s} ($1.57\times$) \\
agents finished & 9/16 & 12/16 & 5/8 & 6/8 \\
checks passed & 11/16 & 10/16 & 5/8 & 4/8 \\
\bottomrule
\end{tabular}
\end{table}
At eight agents our wall time is longer although the throughput is $1.57\times$ higher, because our agents generated 80\% more tokens (3,778 against 2,104): the engines do not make the same decisions, and this small model finishes about half the tasks under either. Throughput is the like-for-like number. Through the OpenAI-compatible server (\texttt{stcon bench --n 8 --rounds 3}, 2B), ours serves 1,928--1,954 output tokens per round in 4.6--4.8\,s (median 410 tokens/s, 366--454 of the tokens forced) against 597--792 tokens in 4.5--7.3\,s for stock (134 tokens/s); the stock replies are shorter, so the fair reading is about three times the tokens in the same wall time, every reply parsing. Four OpenCode coder agents through one \texttt{opencode serve}: 19\,s wall, prompt cache 61\%, three of four tasks correct.

\subsection{LiveCodeBench (engine)}
\label{app:lcb}
The 4B on LiveCodeBench~\cite{lcb} (\texttt{code\_generation\_lite}, release v6, 1,055 problems with hidden tests): the 40 most recent easy and medium problems (March--April 2025; 20 functional, 20 stdin), the official prompt, the last code block extracted, a problem passed only when every public and hidden test passes. Every engine runs in its own process on the same prompts.
\begin{table}[H]
\centering
\small
\caption{\textbf{LiveCodeBench on Qwen3.5-4B}: pass@1, decode tokens, tokens/s (one stream) or in total (batch), wall time, answers cut at the cap. $^\dagger$Thinking on, 20 problems, 6 slots, temperature 0.6 / top-$p$ 0.95 / top-$k$ 20 (greedy loops in thinking mode); 5 of 20 finish within the cap on our engine and all 5 pass.}
\label{tab:lcb}
\begin{tabular}{@{}llrrrrr@{}}
\toprule
cap & engine & pass@1 & tokens & tokens/s & wall & cut \\
\midrule
1,024, greedy & \stcon{}, one stream & 8/40 & 28.2k & 63.4 & 481\,s & 23 \\
 & mlx-lm, one stream & 11/40 & 27.6k & 62.3 & 499\,s & 22 \\
 & \stcon{}, batch 8 & 9/40 & 27.4k & \textbf{160.4} & \textbf{171\,s} & 22 \\
 & mlx-lm, batch 8 & 10/40 & 28.3k & 74.5 & 379\,s & 23 \\
\midrule
4,096, greedy & \stcon{}, one stream & 9/40 & 96.0k & 59.1 & 1,667\,s & 22 \\
 & mlx-lm, one stream & 14/40 & 87.2k & 61.1 & 1,482\,s & 17 \\
 & \stcon{}, batch 8 & 12/40 & 93.3k & \textbf{196.1} & \textbf{476\,s} & 20 \\
 & mlx-lm, batch 8 & 13/40 & 93.0k & 78.7 & 1,181\,s & 19 \\
\midrule
16,384, thinking$^\dagger$ & \stcon{}, batch 6 & 6/20 & 275k & 119.8 & 2,296\,s & 15 \\
 & mlx-lm, batch 6 & 1/20 & 4.4k & --- & out of memory & --- \\
\bottomrule
\end{tabular}
\end{table}
\begin{table}[H]
\centering
\footnotesize
\caption{\textbf{Which token does float32 prefer at a divergence?}. For every problem where an engine's output parts from mlx-lm's single stream, the prompt plus the common prefix is run in float32 and the logit gap between the two competing tokens read.}
\label{tab:divergence}
\begin{tabular}{@{}llrrrrr@{}}
\toprule
cap & engine & divergent & median gap & max gap & float32: mlx-lm & float32: other \\
\midrule
4,096 & \stcon{}, one stream, fp16 & 35 & 0.056 & 0.52 & 3 & \textbf{32} \\
 & \stcon{}, batch 8, fp16 & 34 & 0.062 & 0.52 & 2 & \textbf{32} \\
 & mlx-lm, batch 8 & 36 & 0.145 & 1.83 & 22 & 14 \\
1,024 & \stcon{}, one stream, bf16 & 38 & 0.097 & 0.33 & 23 & 15 \\
\bottomrule
\end{tabular}
\end{table}
With thinking off, the 4B reasons in prose at length before its code; half the answers hit the 4,096-token cap, and the outputs of any two engines are identical on only 2--6 of 40 problems (mlx-lm's own single-stream and batched paths agree on 4 of 40). Pass rates therefore differ by a few problems in either direction between engines that compute the same model, and every divergence is a near-tie; at those ties float32 prefers our fp16 engine's token on 32 of 35. With thinking on and greedy decoding the 4B loops (16 of 20 problems run to a 16k cap, 11 repeating one sentence); with the model card's sampling settings 5 of 20 finish and all pass, and mlx-lm's batch generator runs out of memory at six 16k-token streams after one finished answer, where our engine completes the run inside the wired limit (Table~\ref{tab:regimes}).

\subsection{Arithmetic fidelity (engine)}
\begin{table}[H]
\centering
\small
\caption{\textbf{Compute dtype against float32} (six prompts, 96 teacher-forced positions each).}
\label{tab:fidelity}
\begin{tabular}{@{}llrrrrr@{}}
\toprule
model & dtype & top-1 agreement & mean KL & p99 KL & max KL & max $|\Delta\text{logit}|$ \\
\midrule
4B & bfloat16 (mlx-lm) & 0.9965 & $1.67\times10^{-4}$ & $2.0\times10^{-3}$ & $3.0\times10^{-3}$ & 5.60 \\
4B & float16 (\stcon{}) & \textbf{1.0000} & $\mathbf{3.1\times10^{-6}}$ & $4.6\times10^{-5}$ & $9.1\times10^{-5}$ & 0.31 \\
2B & bfloat16 & 0.9965 & $2.6\times10^{-4}$ & $2.8\times10^{-3}$ & $5.1\times10^{-3}$ & 0.71 \\
2B & float16 & \textbf{1.0000} & $\mathbf{3.6\times10^{-6}}$ & $3.6\times10^{-5}$ & $5.3\times10^{-5}$ & 0.09 \\
\bottomrule
\end{tabular}
\end{table}
On the 2B, greedy generations of the two engines diverge only at near-ties (competing tokens within 0.08 logits in float32); on twelve generic 4B prompts, 5 of 12 200-token generations are identical in either dtype, all gaps within $\pm0.27$, float32 siding with ours on 6 of 7 in fp16.

\section{Memory: all measurements}
\label{app:memory}
Memory is read from MLX's own counters (active and peak bytes on the GPU) inside each engine's process; the resident set size is meaningless for unified memory.

\begin{table}[H]
\centering
\scriptsize
\caption{\textbf{Constraint memory against grammar size}: token table against walk kernel, pipelined engine, greedy, no healing on either side; tokens/s median of three interleaved rounds. The table's memory is $3|Q||\V|$ bytes; the walk's is the automaton plus the 2.9\,MB packed vocabulary. Outputs are identical wherever the table exists. 4B rows for the typed grammar and 96 fields are from a fresh process.}
\label{tab:lazy}
\begin{tabular}{@{}lrrrrrrrrr@{}}
\toprule
 & & & & & & \multicolumn{2}{c}{tok/s, 2B} & \multicolumn{2}{c}{tok/s, 4B} \\
\cmidrule(lr){7-8}\cmidrule(lr){9-10}
grammar & states & table & table build & walk & kernel & table & walk & table & walk \\
\midrule
person & 104 & 77\,MB & 0.7\,s & 2.9\,MB & 12--14\,$\mu$s & 156 & 153 & 75 & 75 \\
tool\_call & 164 & 122\,MB & 1.0\,s & 2.9\,MB & 7--9\,$\mu$s & 263 & 256 & 124 & 121 \\
product & 192 & 143\,MB & 1.3\,s & 2.9\,MB & 8--9\,$\mu$s & 191 & 189 & 90 & 89 \\
event & 178 & 133\,MB & 1.3\,s & 2.9\,MB & 7--9\,$\mu$s & 168 & 161 & 80 & 78 \\
classification & 158 & 118\,MB & 1.0\,s & 2.9\,MB & 7--9\,$\mu$s & 159 & 169 & 82 & 83 \\
contact\_card & 246 & 183\,MB & 1.7\,s & 3.0\,MB & 7--8\,$\mu$s & 183 & 180 & 88 & 87 \\
48 fields & 986 & 735\,MB & 9\,s & 3.3\,MB & 9--14\,$\mu$s & 274 & 267 & 131 & 132 \\
typed tool calls & 1,106 & 824\,MB & 62--127\,s & 3.4\,MB & 11--18\,$\mu$s & 206 & 189 & 89 & 87 \\
96 fields & 1,970 & 1.5\,GB & not built & 3.85\,MB & 8--14\,$\mu$s & --- & 260 & --- & 121 \\
\bottomrule
\end{tabular}
\end{table}
What remains of the compile with the walk kernel is the automaton: 0.02--0.5\,s for the schemas, 25--28\,s for the tool grammar, whose subset construction, not its table, is the cost. On the 4B, a process that had already built the tables of the earlier grammars next to the model no longer had room for the typed grammar's 824\,MB table: it took 332\,s to build and decoded at 6 tokens/s under swap pressure, while the walk in the same process ran at 25 tokens/s; the rows in Table~\ref{tab:lazy} are from a fresh process.

\begin{table}[H]
\centering
\small
\caption{\textbf{Rows on demand} (host-side variant, Qwen3.5-2B, synchronous engine, no healing): rows computed by one decode and their memory against the full table; tokens/s with the full table, with warm rows and with cold rows.}
\label{tab:lazyhost}
\begin{tabular}{@{}lrrrrr@{}}
\toprule
grammar & states & full table & rows & row memory & tok/s full / warm / cold \\
\midrule
person & 104 & 77\,MB & 11 & 8\,MB & 129 / 124 / 95 \\
tool\_call & 164 & 122\,MB & 9 & 7\,MB & 224 / 210 / 159 \\
product & 192 & 143\,MB & 21 & 16\,MB & 154 / 158 / 133 \\
event & 178 & 133\,MB & 17 & 13\,MB & 138 / 136 / 106 \\
classification & 158 & 118\,MB & 13 & 10\,MB & 135 / 139 / 123 \\
contact\_card & 246 & 183\,MB & 16 & 12\,MB & 147 / 150 / 114 \\
48 fields & 986 & 735\,MB & 85 & 63\,MB & 215 / 222 / 160 \\
typed tool calls & 1,106 & 824\,MB & 17 & 13\,MB & 148 / 164 / 94 \\
96 fields & 1,970 & 1.5\,GB (not built) & 85 & 63\,MB & --- / 222 / 159 \\
\bottomrule
\end{tabular}
\end{table}

\begin{table}[H]
\centering
\scriptsize
\caption{\textbf{Constraint memory and compile time, token table, walk kernel and llguidance}: Qwen3.5 vocabulary, CPU-side, each library in its own process, every grammar driven through a valid document computing the mask at every step.}
\label{tab:three-memory}
\begin{tabular}{@{}lrrr@{}}
\toprule
 & \stcon{} token table & \stcon{} walk kernel & llguidance \\
\midrule
person schema (104 states) & 77\,MB & 0.05\,MB & $<$0.5\,MB \\
96 fields (1,970 states) & 1.5\,GB & 1.0\,MB & $<$0.5\,MB \\
typed tool calls (1,106 states) & 824\,MB & \textbf{0.57\,MB} & does not compile \\
recursive schema (solar) & 38\,MB (pushdown tables) & --- & $<$0.5\,MB \\
16-grammar server, grammars & 7.2\,GB (dies at the 15th) & 4.9\,MB & 4.7\,MB \\
shared, once per process & --- & 2.9\,MB & 153\,MB \\
compile, typed tool grammar & 25\,s DFA + 62--127\,s table & 25\,s DFA & --- \\
compile, one schema & 0.7--1.7\,s & 0.02--0.4\,s & 1--2\,ms \\
mask, per step & GPU gather & 7--18\,$\mu$s (GPU) & 60--270\,$\mu$s (CPU) \\
\bottomrule
\end{tabular}
\end{table}

\begin{table}[H]
\centering
\footnotesize
\caption{\textbf{The server in the wild, both models}: grammars served before failure, GPU memory above the model after the last grammar served, longest first-request wait.}
\label{tab:wild}
\begin{tabular}{@{}llrrr@{}}
\toprule
model, prompt cache & constraints & served & above the model & longest first request \\
\midrule
2B, 50\,MB & token tables & 15/16, then out of memory & 8.2\,GB & 26\,s \\
2B, 50\,MB & walk kernel & \textbf{16/16} & \textbf{0.6\,GB} & 2.8\,s \\
2B, 1.5\,GB & token tables & 14/16, then out of memory & 7.6\,GB & 20\,s \\
2B, 1.5\,GB & walk kernel & \textbf{16/16} & \textbf{2.2\,GB} & 3.1\,s \\
4B, 50\,MB & token tables & 14/16, then out of memory & 6.9\,GB & 26\,s \\
4B, 50\,MB & walk kernel & \textbf{16/16} & \textbf{1.1\,GB} & 6.6\,s \\
\bottomrule
\end{tabular}
\end{table}
The sixteen grammars are the six schemas, then schemas of 8, 12, 16, 24, 32, 40, 48, 64, 80 and 96 fields; four requests per grammar, three concurrent. With the walk kernel the automata and the packed vocabulary are about 50\,MB of the 0.6\,GB above the model on the 2B; the rest is MLX's buffer cache. Throughput with the walk kernel stays within the run-to-run spread (114--209 against 107--218 tokens/s on the 2B).

\subsection{Decoding regimes, the prefill transient and a memory budget (engine)}
\begin{table}[H]
\centering
\small
\caption{\textbf{Peak GPU memory above the model by decoding regime} (Qwen3.5-4B, $B$ streams $\times$ $N$ tokens on LiveCodeBench prompts, no constraint).}
\label{tab:regimes}
\begin{tabular}{@{}l rr rr@{}}
\toprule
 & \multicolumn{2}{c}{\stcon{}} & \multicolumn{2}{c}{stock mlx-lm} \\
\cmidrule(lr){2-3}\cmidrule(lr){4-5}
streams $\times$ tokens & above the model & tok/s & above the model & tok/s \\
\midrule
$1\times4096$ & 1.0\,GB & 65 & 0.8\,GB & 60 \\
$8\times64$ (prefill only) & 6.4\,GB & & 3.8\,GB & \\
$8\times1024$ & 6.4\,GB & 175 & 3.7\,GB & 82 \\
$8\times4096$ & 6.4\,GB & 212 & 5.2\,GB & 84 \\
$6\times8192$ & 4.9\,GB & 93 & 5.9\,GB & 64 \\
$6\times16384$ & \textbf{4.9\,GB} & 128 & out of memory & --- \\
\bottomrule
\end{tabular}
\end{table}
\begin{table}[H]
\centering
\small
\caption{\textbf{The prefill transient} (Qwen3.5-4B, our engine, nothing decoded): peak memory above the model while prefilling the longest prompt (1,360 tokens) alone, and $B$ prompts in one right-padded batch, by prefill chunk.}
\label{tab:prefill}
\begin{tabular}{@{}lrrrrr@{}}
\toprule
 & 2048 & 1024 & 512 & 256 & 128 \\
\midrule
one stream, 1,360 tokens & 1.41\,GB & 1.46\,GB & 1.23\,GB & 1.09\,GB & 1.19\,GB \\
batch of 2 & & 2.12\,GB & & 1.36\,GB & \\
batch of 4 & & 3.50\,GB & & 1.59\,GB & \\
batch of 8 & & 6.48\,GB & & 2.43\,GB & \\
\bottomrule
\end{tabular}
\end{table}
At short contexts the peak is the prompt pass, not the caches: our batch engine prefills every admitted prompt in one right-padded batch with 1,024-token chunks, about 0.8\,MB per token of (rows $\times$ chunk), while stock prefills four prompts at a time in 512-token steps but computes logits at every prompt position (about 1\,MB per token). Capping a prefill pass at about 2,000 tokens brings our 8-slot peak to about 2.4\,GB; skipping the output head in the single-stream prompt pass saves about 1.2\,GB. At long contexts our caches never exceed the prefill transient (the peak at $6\times8192$ equals the peak at $6\times16384$, although the caches hold 3.3\,GB at 16k tokens), while stock grows its KV cache by copying, so its growth costs about twice the raw cache and exceeds the wired limit at six 16k-token streams.
\textbf{A budget for $n$ agents.} With $W$ the weights, $C$ the prompt-cache budget, $P$ the prefill transient, $n$ streams of $T$ tokens and $G$ the constraints, the peak is about $M\approx W+C+P+n(s+\tau T)+G$; on the 4B, $W=2.26$\,GB, $s=50$\,MB, $\tau=32$\,KB per token, $P\approx0.8\,\mathrm{MB}\times(\text{rows}\times\text{chunk})$ for our batch prefill, and $G$ is 3\,MB per grammar with the walk kernel against $3|Q||\V|$ bytes, never freed, with tables. Eight coding agents with 1,000-token prompts, 4,000-token answers and four schemas need $2.26+1.5+6.4+8\times(0.05+0.13)+0.01\approx11.6$\,GB on our engine today, about $7.6$\,GB with the 2,000-token prefill cap; on stock mlx-lm with tables for the four schemas, about $10.8$\,GB, at the edge of the 10.7\,GB limit.

\section{Budget masking: measurements}
\label{app:budgetexp}
\begin{table}[H]
\centering
\small
\caption{\textbf{Validity under a token budget, regular path} (Qwen3.5-2B): budgets $B=2,4,\dots$ up to the full greedy answer; $\need_0$ is the budget table's value at the start (the shortest valid output under the forced-span policy). Feasible budgets are those with $B\ge\need_0$.}
\label{tab:budgetsteps}
\begin{tabular}{@{}lrrrrrr@{}}
\toprule
schema & full answer & $\need_0$ & first valid & feasible & valid, mask & valid, no mask \\
\midrule
person & 34 & 16 & 16 & 10 & 10 & 1 \\
tool\_call & 29 & 25 & 26 & 2 & 2 & 0 \\
product & 57 & 37 & 38 & 10 & 10 & 0 \\
event & 75 & 18 & 18 & 29 & 29 & 0 \\
classification & 58 & 22 & 22 & 19 & 19 & 1 \\
contact\_card & 74 & 35 & 36 & 20 & 20 & 1 \\
\midrule
all & & & & 90 & \textbf{90} & 3 \\
\bottomrule
\end{tabular}
\end{table}
None of the 72 infeasible budgets ($B<\need_0$) yields a valid output with or without the mask, as Theorem~\ref{thm:budget} predicts; the first valid budget is $\need_0$ or the next even number. On the pushdown path (four completable recursive tasks, steps of six from eight tokens), the sound byte-level bound of Proposition~\ref{prop:pdabudget} gives 54 of 61 valid outputs against 2 of 61, every budget valid from the first one that yields a valid output (12--20 tokens) upward.

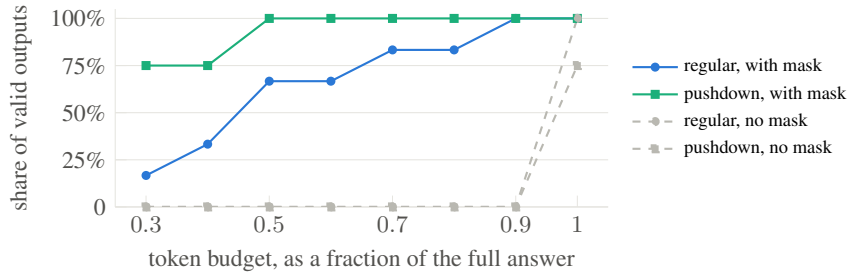
\begin{figure}[H]
\centering
\begin{tikzpicture}
\begin{axis}[width=0.58\textwidth, height=4.2cm, xlabel={token budget, as a fraction of the full answer}, ylabel={share of valid outputs},
  xmin=0.25, xmax=1.05, ymin=0, ymax=1.05, xtick={0.3,0.5,0.7,0.9,1.0}, ytick={0,0.25,0.5,0.75,1}, yticklabels={0,25\%,50\%,75\%,100\%},
  legend style={at={(1.03,0.5)}, anchor=west, font=\scriptsize}, legend cell align=left]
\addplot[ours, thick, mark=*, mark size=1.3pt] coordinates {(0.3,0.167)(0.4,0.333)(0.5,0.667)(0.6,0.667)(0.7,0.833)(0.8,0.833)(0.9,1)(1.0,1)};
\addplot[third, thick, mark=square*, mark size=1.3pt] coordinates {(0.3,0.75)(0.4,0.75)(0.5,1)(0.6,1)(0.7,1)(0.8,1)(0.9,1)(1.0,1)};
\addplot[deemph, thick, dashed, mark=*, mark size=1.3pt] coordinates {(0.3,0)(0.4,0)(0.5,0)(0.6,0)(0.7,0)(0.8,0)(0.9,0)(1.0,1)};
\addplot[deemph, thick, dashed, mark=square*, mark size=1.3pt] coordinates {(0.3,0)(0.4,0)(0.5,0)(0.6,0)(0.7,0)(0.8,0)(0.9,0)(1.0,0.75)};
\legend{{regular, with mask}, {pushdown, with mask}, {regular, no mask}, {pushdown, no mask}}
\end{axis}
\end{tikzpicture}
\caption{\textbf{Validity against the token budget} (Qwen3.5-2B; budget as a fraction of each task's full answer; six regular and four recursive tasks). Without the mask every cut output is invalid. With it, outputs shorten gracefully; where the masked regular curve is below 100\%, the budget is below the shortest valid answer for some schemas.}
\label{fig:budget}
\end{figure}

\begin{figure}[H]
\centering
\footnotesize
\begin{tabular}{@{}rL{11.6cm}@{}}
\toprule
budget & output (Qwen3.5-2B, the \texttt{event} schema) \\
\midrule
320 (75 used) & \texttt{\{"title":\allowbreak "The Applied AI Summit",\allowbreak "date":\allowbreak "2026-11-12",\allowbreak "location":\allowbreak \{"venue":\allowbreak "Congress Center",\allowbreak "city":\allowbreak "Berlin",\allowbreak "country":\allowbreak "Germany"\},\allowbreak "attendees":\allowbreak [\{"name":\allowbreak "Dr. Aiko Tanaka",\allowbreak "role":\allowbreak "speaker"\},\{"name":\allowbreak "Lars Petersen",\allowbreak "role":\allowbreak "speaker"\},\{"name":\allowbreak "Nina Weber",\allowbreak "role":\allowbreak "organizer"\}]\}} \\
48 & \texttt{\{"title":\allowbreak "The Applied AI Summit",\allowbreak "date":\allowbreak "2026-11-12",\allowbreak "location":\allowbreak \{"venue":\allowbreak "Congress Center",\allowbreak "city":\allowbreak "Berlin",\allowbreak "country":\allowbreak "Germany"\},\allowbreak "attendees":\allowbreak [\{"name":\allowbreak "Dr.",\allowbreak "role":\allowbreak "speaker"\}]\}} \\
32 & \texttt{\{"title":\allowbreak "The Applied AI Summit",\allowbreak "date":\allowbreak "2026-11-",\allowbreak "location":\allowbreak \{"venue":\allowbreak "",\allowbreak "city":\allowbreak "",\allowbreak "country":\allowbreak ""\},\allowbreak "attendees":\allowbreak []\}} \\
24 & \texttt{\{"title":\allowbreak "The Applied AI Summit",\allowbreak "date":\allowbreak "",\allowbreak "location":\allowbreak \{"venue":\allowbreak "",\allowbreak "city":\allowbreak "",\allowbreak "country":\allowbreak ""\},\allowbreak "attendees":\allowbreak []\}} \\
\midrule
40, no mask & \texttt{\{"title":\allowbreak "The Applied AI Summit",\allowbreak "date":\allowbreak "2026-11-12",\allowbreak "location":\allowbreak \{"venue":\allowbreak "Congress Center",\allowbreak "city":\allowbreak "Berlin",\allowbreak "country":\allowbreak "Germany"\},\allowbreak "attendees":\allowbreak [\{"name":\allowbreak "} \quad (cut mid-value, invalid) \\
\bottomrule
\end{tabular}
\caption{\textbf{The budget mask on one output.} As the budget tightens the decoder closes the array early, then shortens strings, then empties them, because it admits only tokens after which a valid ending still fits.}
\label{fig:budgetex}
\end{figure}

\section{Quality: details}
\label{app:quality}
\textbf{Metrics.} \emph{Agreement} is the share of leaf values of the model's unconstrained greedy answer (its parsed JSON) that the constrained output reproduces exactly, pooled over the six schemas (the per-workload mean is 0.94 / 0.87 for \stcon{}). It measures fidelity to what the model would have said, not correctness, and it penalizes harmless variants. \emph{Gold correctness} compares 45 objective fields of five schemas with values read off the source texts: person (name, age, email, status, three tags and their number: 8 fields), tool\_call (function, location, unit, include\_forecast, days: 5), product (name, with or without ``The''; price, currency, in\_stock, three categories and their number, three dimensions: 11), event (title, with or without ``The''; date, venue, city, country, three attendees with names and roles and their number: 12), contact\_card (all nine fields). The classification task has no single correct answer and is excluded. Strings are compared after trimming and lowercasing, numbers numerically. Intervals are 95\% Wilson intervals with fields treated as independent, which they are not entirely.

\begin{table}[H]
\centering
\small
\caption{\textbf{Gold fields reproduced, per schema} (outputs of Table~\ref{tab:quality}).}
\label{tab:goldper}
\begin{tabular}{@{}l rrrrr r@{}}
\toprule
engine (2B / 4B) & person & tool\_call & product & event & contact\_card & all \\
\midrule
free, \stcon{}, llguidance (all modes) & 8 / 8 & 5 / 5 & 11 / 11 & 12 / 12 & 9 / 9 & 45 / 45 \\
\bottomrule
\end{tabular}
\end{table}

\textbf{Segmentation policies.} A separate run on the 2B (six schemas, 320-token limit, per-workload mean agreement) compares our three policies on the same regex: plain token-level masking (the model chooses how forced bytes are split) 0.93 at 343 passes, commit-all 0.93 at 246, and the canonical policy 0.95 at 224; llguidance in the same run reaches 0.95 with 341 passes (mask only) and 245 (fast-forward). The policy visibly matters on one task: under plain token-level masking the model writes ``Just-Luc Moreau'' on the contact card where its unconstrained answer and both forcing policies write ``Jean-Luc''. Since only the segmentation policy differs between these runs, the divergence is due to the policy.

\section{What did not help}
\label{app:neg}
\begin{table}[H]
\centering
\small
\caption{\textbf{Measured ceilings and negative results} (Qwen3.5-2B). Every candidate was first bounded by replacing the component with a free stand-in and timing the real step.}
\label{tab:negative}
\begin{tabular}{@{}L{5.4cm}L{7.6cm}@{}}
\toprule
idea & measurement \\
\midrule
defer the recurrent-state writes (a log of updates) & removing the state store saves 2.3\% of a step at batch 1, 6.7\% at 16--32; not built \\
a certified low-rank vocabulary head & the head's spectrum is flat: rank 1,024 of 2,048 rules out under 2\% of rows; not built \\
fuse the norms, residual adds and SwiGLU & 2--10\% of a pipelined step; each costs 1--2\,$\mu$s \\
score few-choice regions as a tree in one pass & forced spans already absorb them: 0--4\% fewer passes, no wall-time change \\
finite regions (sequence-level decision of a label among 16) & 2.0 passes against 2.9, but the intended label on 7 of 24 tickets against 15 for token-level greedy: the most probable complete string is the short generic label \\
decode sibling subtrees of a recursive array in parallel & slower on every task (Table~\ref{tab:siblings}); branches invent what their siblings decide \\
speculative drafts in the batch engine & a 2-token step at 8 rows costs $1.48\times$ a 1-token step: batched steps are compute-bound \\
prompt-lookup drafts verified against the constraint & exact; $1.02\times$ overall, $1.1$--$1.27\times$ where the model copies, up to 11\% slower where drafts are rejected; off by default \\
block-parallel fields without a certificate & $1.76\times$ end to end on 18 requests at agreement 0.94, but approximate: a value does not see earlier values \\
\bottomrule
\end{tabular}
\end{table}
\begin{table}[H]
\centering
\small
\caption{\textbf{Sibling-parallel decoding of recursive arrays} (\texttt{siblings.py}, Qwen3.5-2B, wall time = prefill + decode).}
\label{tab:siblings}
\begin{tabular}{@{}lll@{}}
\toprule
task & sequential pushdown & sibling-parallel \\
\midrule
solar & 0.84\,s, 109 tok/s, 11 nodes & 2.02\,s, 66 tok/s, 13 nodes (Venus given Earth's subtree) \\
thread & 2.67\,s, 143 tok/s & 2.67\,s, 71 tok/s (a reply duplicated) \\
fs & 0.85\,s, 122 tok/s & 4.03\,s, 82 tok/s \\
cats & 0.55\,s, 76 tok/s & 2.00\,s, 55 tok/s \\
org & 0.32\,s, 79 tok/s & 2.09\,s, 89 tok/s \\
menu & 0.39\,s, 104 tok/s & 35\,s, 1,168 invented nodes \\
\bottomrule
\end{tabular}
\end{table}
The block product holds for the \emph{language} of sibling subtrees; it does not make the model's decisions independent. That is why the certified block decoder verifies its drafts, and why the parallel decoding of siblings without verification fails.

\section{Exactness and tests}
\label{app:exact}
\textbf{Exactness.} In exact arithmetic none of the mechanisms changes what greedy decoding produces, except the canonical policy, which changes the segmentation of forced text on purpose (Theorem~\ref{thm:pass} compares decoders under the same policy; Appendix~\ref{app:quality}). Measured: certified block decoding is token-identical to sequential decoding on all 36 outputs of both models; the walk kernel's rows equal the table's columns for every state, and its outputs are identical on every grammar, single-stream and batched; pipelined and synchronous loops give identical outputs; the pushdown tables equal byte-by-byte stepping on every reachable configuration of a small vocabulary with stacks to depth 4 (69k, 35k and 13k configuration--token pairs for three grammars) and brute force over all 248k tokens at sampled prefixes with the real vocabulary; budgets survive a simulated stopping rule with adversarial token choices. Comparisons with other engines are subject to floating-point near-ties (Table~\ref{tab:divergence}).

\textbf{Tests.} 234 automated tests, 44 of which also run against the 4B: the capped traversal and path counts against brute force on random graphs; the regex parser, UTF-8 expansion, subset construction, trimming and minimization against Python's \texttt{re}; forced chains, the canonical split, stable-tail commit, boundary cover and healing against the tokenizer; pipelined against synchronous decoding; every kernel family and the fused recurrent step against stock mlx-lm in both dtypes; the budget stopping rule; the pushdown (exhaustive reachable configurations, real-vocabulary traces, determinism and depth limit); the walk kernel against the table, single and batched; certified blocks against sequential decoding; the not-containing expression against brute force; the server (tool calls, prompt-cache round trip, JSON response formats); sampling with top-$k$ and top-$p$.

\textbf{Protocol.} One machine, one GPU job at a time, configurations interleaved inside one process with medians reported, baselines run as controllers in our loop with our kernels when the comparison is about the control side, and in their own process when it is about the engine.

\end{document}